\documentclass[twoside,11pt]{article}
\usepackage{jmlr2e}

\newcommand{\mSet}[1]{\{#1\}}
\newcommand{\mAbs}[1]{\vert #1 \vert}

\newcommand{\timesRq}[0]{\times_{\mathbb{R}_q}}

\usepackage{lastpage}

\ShortHeadings{Systematic Neural Network Representations of Algorithms}{Kratsios, Zvigelsky, and Hart}
\firstpageno{1}

\usepackage{graphicx} 

\usepackage{marginnote}
\usepackage[table]{xcolor}
\usepackage{float}
\usepackage{subcaption}
\usepackage{amssymb}
    \usepackage{stmaryrd}
\usepackage{amstext}
\usepackage{amsmath}
\usepackage{amscd}
\usepackage{amsfonts}
\usepackage{enumitem}
\usepackage{graphicx}
\usepackage{latexsym}
\usepackage{mathrsfs}
\usepackage{mathtools}
\usepackage{cases}
\usepackage{verbatim}
\usepackage{bm}
\usepackage{tikz}
\usepackage[top=3cm, bottom=3cm, left=2.5cm, right=2.5cm]{geometry}
\usepackage{natbib}
\usepackage{mathabx}

\definecolor{oxfordblue}{rgb}{0.0, 0.13, 0.28}
\definecolor{linkblue}{rgb}{0.0, 0.20, 0.40}
\usepackage{hyperref} 
\hypersetup{
    pdftitle={Quantifying The Limits of AI Reasoning Systematic Neural Network Representations of Algorithms},
    pdfauthor={Anastasis Kratsios, Dennis Zvigelsky, Bradd Hart},
    pdfsubject={AI Reasoning, Neural networks, Circuit emulation, Universal approximation},
    pdfkeywords={Artificial Intelligence, AI, Machine Learning, Deep Learning, 
        Neural Networks, ReLU Networks, Transformer Models, 
        Circuit Emulation, Universal Approximation Theorem, 
        Symbolic Computation, Algorithmic Reasoning, 
        Computational Complexity, Logic Reasoning, Predicate Logic, 
        Tropical Circuits, Dynamic Programming, Higher-Order Logic, 
        Mathematical Reasoning, Theoretical Machine Learning, 
        Representation Learning, Function Approximation, 
        Generalization Bounds, Optimization, Gradient Descent
    },
    colorlinks = true,
    linkcolor = linkblue,
    anchorcolor = linkblue,
    citecolor = linkblue,
    filecolor = linkblue,
    urlcolor = linkblue
}

\usepackage{comment}

\usepackage{cleveref}
\newcounter{termcounter}
\renewcommand{\thetermcounter}{\Roman{termcounter}}
\crefname{term}{term}{terms}
\creflabelformat{term}{#2\textup{(#1)}#3}

\makeatletter
\def\term{\@ifnextchar[\term@optarg\term@noarg}
\def\term@optarg[#1]#2{%
  \textup{#1}%
  \def\@currentlabel{#1}%
  \def\cref@currentlabel{[][2147483647][]#1}%
  \cref@label[term]{#2}}
\def\term@noarg#1{%
  \refstepcounter{termcounter}%
  \textup{(\thetermcounter)}%
  \cref@label[term]{#1}}
\makeatother

\definecolor{darkcyan}{rgb}{0.0, 0.55, 0.55}
\definecolor{MidnightBlue}{RGB}{25,25,112}
\definecolor{MidnightBlueComplementingGreen}{RGB}{25,112,25}
\definecolor{MidnightBlueComplementingPurple}{RGB}{112,25,112}
\definecolor{MidnightBlueComplementingRed}{RGB}{112,25,69}
\definecolor{WowColor}{rgb}{.75,0,.75}
\definecolor{MildlyAlarming}{rgb}{0.85,0.25,0.1}
\definecolor{SubtleColor}{rgb}{0,0,.50}
\definecolor{antiquefuchsia}{rgb}{0.57, 0.36, 0.51}
\definecolor{fashionfuchsia}{rgb}{0.96, 0.0, 0.63}
\definecolor{jade}{rgb}{0.0, 0.66, 0.42}
\definecolor{caribbeangreen}{rgb}{0.0, 0.8, 0.6}
\definecolor{aquamarine}{rgb}{0.5, 0.8, 0.85}
\definecolor{darkmidnightblue}{rgb}{0.0, 0.2, 0.4}
\definecolor{attentioncolor}{RGB}{152,90,81}
\definecolor{burgred}{RGB}{40,3,22}
\definecolor{AKGreen}{RGB}{17,123,92}
\definecolor{egyptianblue}{rgb}{0.06, 0.2, 0.65}
\definecolor{Turquoise}{RGB}{64,224,208}

\definecolor{darkjade}{RGB}{0,122,84}
\definecolor{Window1}{RGB}{92,150,31}%
    \definecolor{Window1dark}{RGB}{41,67,13}%
\definecolor{Window2}{RGB}{255,168,28}
    \definecolor{Window2dark}{RGB}{114,75,12}
\definecolor{Window3}{RGB}{255,96,33}
    \definecolor{Window3dark}{RGB}{97,36,12}
\definecolor{InputColor}{RGB}{20,255,177}
    \definecolor{InputColorlight}{RGB}{222,237,229}

\usepackage[colorinlistoftodos]{todonotes}

\NewDocumentCommand{\M}{o}{
    \IfValueT{#1}{
            \mathbb{R}_q^{#1}
        }
    \IfValueF{#1}{
            \mathbb{R}_q
        }
                    }
                    
\NewDocumentCommand{\Z}{o}{
    \IfValueT{#1}{
            \mathbb{Z}^{#1}
        }
    \IfValueF{#1}{
            \mathbb{Z}
        }
                    }
                    
\NewDocumentCommand{\F}{o}{
    \IfValueT{#1}{
            \mathbb{G}_{#1}
        }
    \IfValueF{#1}{
            \mathbb{G}
        }
                    }

\NewDocumentCommand{\R}{o}{
    \IfValueT{#1}{
            \mathbb{R}^{#1}
        }
    \IfValueF{#1}{
            \mathbb{R}
        }
                    }
                    
\NewDocumentCommand{\N}{o}{
    \IfValueT{#1}{
            \mathbb{N}^{#1}
        }
    \IfValueF{#1}{
            \mathbb{N}
        }
                    }

\newcommand{ \eqdef }{
    \ensuremath{\stackrel{\mbox{\upshape\tiny def.}}{=}}
}

\usepackage{marginnote}
\definecolor{MidnightBlue}{RGB}{25,25,112}
\definecolor{MidnightBlueComplementingGreen}{RGB}{25,112,25}
\definecolor{MidnightBlueComplementingPurple}{RGB}{112,25,112}
\definecolor{MidnightBlueComplementingRed}{RGB}{112,25,69}
\definecolor{coolblack}{rgb}{0.0, 0.18, 0.39}

\definecolor{deepjunglegreen}{rgb}{0.0, 0.29, 0.29}
\definecolor{tropicalrainforest}{rgb}{0.0, 0.46, 0.37}
\definecolor{applegreen}{rgb}{0.55, 0.71, 0.0}
\definecolor{WowColor}{rgb}{.75,0,.75}
\definecolor{MildlyAlarming}{rgb}{0.85,0.25,0.1}
\definecolor{SubtleColor}{rgb}{0,0,.50}
\definecolor{SubtleColor2}{rgb}{0.6,0.21,.50}

\definecolor{lasallegreen}{rgb}{0.03, 0.47, 0.19}
\newcounter{margincounter}

\NewDocumentCommand{\AK}{moo}{
    \IfValueF{#2}{
        \IfValueF{#3}{%
                        {{%
                            \textcolor{deepjunglegreen}{%
                                \textbf{Annie:}%
                                \textit{{#1}}%
                            }%
                        }}%
        }%
    }
    \IfValueT{#2}{
        \IfValueF{#3}{
                        {{\scriptsize
                            \textcolor{deepjunglegreen}{
                                \hfill\\
                                    \textbf{Annie:}
                                    \textit{{#1}}
                                \hfill\\
                            }
                        }}
        }
        \IfValueT{#3}{
                        \marginnote{{\scriptsize
                            \textcolor{deepjunglegreen}{ 
                            \textbf{Annie:}
                            \textit{{#1}}
                            }
                        }}
                }
        }
                    }

\NewDocumentCommand{\sam}{moo}{
    \IfValueF{#2}{
        \IfValueF{#3}{
                        {{
                            \textcolor{blue}{
                                \textbf{Sam:}
                                \textit{{#1}}
                            }
                        }}
        }
    }
    \IfValueT{#2}{
        \IfValueF{#3}{
                        {{
                            \textcolor{blue}{
                                \hfill\\
                                    \textbf{Sam:}
                                    \textit{{#1}}
                                \hfill\\
                            }
                        }}
        }
        \IfValueT{#3}{
                        \marginnote{{\scriptsize
                            \textcolor{blue}{ 
                            \textbf{SL:}
                            \textit{{#1}}
                            }
                        }}
                }
        }
                    }

\newcommand{\relu}{\mathrm{ReLU}}
\newcommand{\ReLU}{\relu}

\newcommand{\TRUE}{\texttt{TRUE}}
\newcommand{\FALSE}{\texttt{FALSE}}

\newcommand{\EMBED}{\texttt{EMBED}}
\newcommand{\EQUAL}{\texttt{EQUAL}}
\newcommand{\NAND}{\texttt{NAND}}
\newcommand{\IMPLY}{\texttt{IMPLY}}
\newcommand{\NOT}{\texttt{NOT}}
\newcommand{\AND}{\texttt{AND}}
\newcommand{\OR}{\texttt{OR}}
\newcommand{\XOR}{\texttt{XOR}}
\newcommand{\LEFT}{\texttt{LSHIFT}}
\newcommand{\RIGHT}{\texttt{RSHIFT}}
\newcommand{\ADD}{\texttt{ADD}}
\newcommand{\MULT}{\texttt{MULT}}
\newcommand{\COMP}{\texttt{COMP}}
\newcommand{\IDENTITY}{\texttt{IDENTITY}}

\newcommand{\INT}{\texttt{INT}} 
\newcommand{\TWOSINT}{\overline{\texttt{INT}}} 
\newcommand{\TINT}{\TWOSINT}
\newcommand{\UINT}{\texttt{UINT}} 

\NewDocumentCommand{\In}{m}{\operatorname{In}{(#1)}}
\NewDocumentCommand{\Out}{m}{\operatorname{Out}{(#1)}}
\NewDocumentCommand{\pa}{m}{\operatorname{pa}_{#1}}
\NewDocumentCommand{\ch}{m}{\operatorname{ch}_{#1}}
\NewDocumentCommand{\comp}{o}{\operatorname{Comp}
    {\IfValueT{#1}{({#1}})}
}

\NewDocumentCommand{\Rep}{o}{
{\operatorname{Cpt}}
    {\IfValueT{#1}{
        ({#1})
    }}
}

\NewDocumentCommand{\fff}{o}{{
    \mathcal{G}
    {\IfValueT{#1}{_{{#1}}}}
    }}

\date{August 21, 2025}

\newtheorem{lemma}{Lemma}
\newtheorem{theorem}{Theorem}
\newtheorem{corollary}{Corollary}
\newtheorem{proposition}{Proposition}
\newtheorem{definition}{Definition}
\newtheorem{example}{Example}

\newtheorem{remark}{Remark}

\usepackage{booktabs}
\usepackage{array}
\usepackage{longtable}
\usepackage{colortbl}
\usepackage{adjustbox}
\usepackage{etoolbox} 
\definecolor{rowgray}{gray}{0.96}
\definecolor{headergray}{gray}{0.90}
\definecolor{sectionblue}{rgb}{0.85,0.92,1}
\newcolumntype{C}[1]{>{\centering\arraybackslash}p{#1}}

\usepackage{minitoc}

\usepackage{todonotes}

\usepackage[most]{tcolorbox}
\usepackage{etoolbox} 
\definecolor{faintgray}{RGB}{245,245,245}     
\definecolor{faintborder}{RGB}{230,230,230}   
\definecolor{lightblack}{gray}{0.4}           
\newcounter{question}
\newtcolorbox[auto counter, use counter=question]{question}[1][]{
  enhanced,
  colback=faintgray,
  colframe=faintborder,
  boxrule=0.2pt,
  arc=2mm,
  title=\textcolor{lightblack}{\textbf{Question~\thequestion}},
  fonttitle=\bfseries,
  before upper={\centering\itshape},
  after title={\vspace{0.5ex}},
  boxsep=4pt,
  left=6pt,
  right=6pt,
  top=4pt,
  bottom=4pt,
  #1
}

\usepackage{algorithm}
\usepackage{algpseudocode}
\usepackage{amsmath}  

\begin{document}
\doparttoc 
\faketableofcontents 

\part{} 

\title{Quantifying The Limits of AI Reasoning:\\ 
Systematic Neural Network Representations of Algorithms}
\author{\name Anastasis Kratsios
\email kratsioa@mcmaster.ca \\
       \addr Department of Mathematics and Statistics\\
       McMaster University and Vector Institute\\
       Ontario, Canada
\AND 
    \name Dennis Y. Zvigelsky 
    \email yankovsd@mcmaster.ca \\
       \addr Department of Mathematics and Statistics\\
       McMaster University and Vector Institute\\
       Hamilton ON, Canada
\AND
    \name Bradd Hart
    \email hartb@mcmaster.ca\\
    \addr Department of Mathematics and Statistics\\
    McMaster University\\
    Hamilton, ON, Canada 
}

\editor{}

\maketitle

\begin{abstract}
A main open question in contemporary AI research is quantifying the forms of reasoning neural networks can perform when perfectly trained.  This paper answers this by interpreting reasoning tasks as circuit emulation, where the gates define the type of reasoning; e.g.\ Boolean gates for predicate logic, tropical circuits for dynamic programming, arithmetic and analytic gates for symbolic mathematical representation, and hybrids thereof for deeper reasoning; e.g.\ higher-order logic.

We present a systematic meta-algorithm that converts essentially any circuit into a feedforward neural network (NN) with ReLU activations by iteratively replacing each gate with a canonical ReLU MLP emulator. 
We show that, on any digital computer, our construction emulates the circuit exactly—no approximation, no rounding, modular overflow included—demonstrating that no reasoning task lies beyond the reach of neural networks.
The number of neurons in the resulting network (parametric complexity) scales with the circuit’s complexity, and the network's computational graph (structure) mirrors that of the emulated circuit.  This formalizes the folklore that NNs networks trade algorithmic run-time (circuit runtime) for space complexity (number of neurons). 

We derive a range of applications of our main result, from emulating shortest-path algorithms on graphs with cubic-size NNs, to simulating stopped Turing machines with roughly quadratically-large NNs, and even the emulation of randomized Boolean circuits. Lastly, we demonstrate that our result is \textit{strictly more powerful} than a classical \textit{universal approximation} theorem: any universal function approximator can be encoded as a circuit and directly emulated by a NN.
\end{abstract}



\maketitle

\section{Introduction}
\label{s:Intro}

Since their mainstream adoption reasoning-based LLM models have redefined the state-of-the-art in nearly every computational discipline, from medicine and biology~\cite{luo2022biogpt,zhang2023biomedgpt}, chemistry~\cite{bran2023chemcrow}, finance~\cite{yang2023fingpt,yang2023instruct,fan2024finqapt,MoEF}, physics~\cite{xu2025ugphysics,qiu2025phybench,zhang2025physreason}, epidemiology~\cite{chen2023genspectrum,li2024ae}, and many others.  
The success of deep learning models is typically understood as a synergy between the \textit{expressive potential}, the \textit{statistical} properties, and their amenability to (gradient-descent-type) optimization algorithms.  However, the relationship between these three factors and AI-based reasoning remains somewhat opaque.  We focus on understanding AI reasoning through the lens of the \textit{expressive potential} of AIs, by which we mean what such systems are capable of given unlimited amounts of perfect data and idealized training procedures.

The AI theory community's interpretation of ``expressivity potential'' of a deep learning model is heavily influenced by the approximation-theoretic tradition, shaped by the historical impact of the \textit{universal approximation theorem} of~\cite{hornik1989multilayer,cybenko1989approximation,funahashi1989approximate}.  This tradition of approximation theorems framed our understanding of neural network \textit{expressivity} as the ability to express (effectively) any function asymptotically, ~e.g.~\cite{yarotsky2018optimal,petersen2018optimal,zech2019deep,elbrachter2021deep,kratsios2022do,zaman2022trans,gribonval2022approximation,zech2023deep,voigtlaender2023universal,siegel2024sharp,van2024noncompact,neufeld2024universal,zhang2024deep,adcock2025near}, and various well-behaved functions using relatively few parameters, e.g.~\cite{barron1993universal,suzuki2018adaptivity,MR4808368,lu2021deep,guhring2021approximation,MR4557620,yang2024near,abdeljawad2024weighted}.  This lens ultimately culminated in extremal function approximation results in two bifurcating directions either: 1) identifying irregular activation functions that let fixed-size neural networks approximate essentially any function with arbitrary precision~\cite{pmlr-v139-yarotsky21a,zhang2022deep} or 2) constructing networks balancing regularity and approximation power~\cite{hong2024bridging,riegler2024generating}.

Though the approximation-theoretic viewpoint of AI expressive potential has been extremely natural in applications of deep learning in the computational sciences~\cite{hutzenthaler2020overcoming,cai2021physics,marcati2023exponential,gonon2023random}, natural~\cite{velivckovic2023everything}, economic~\cite{buehler2019deep,MR4144883} sciences, signal processing~\cite{gama2018convolutional,ma2021unified,hovart2023deep}, etc.; the traditional approximation-theoretic lens fails to quantify, let alone answer, one of the central \textit{open questions} in the contemporary AI landscape.  Namely:
\begin{question}
\label{eq:Main_Question}
What problems can feedforward neural networks reason through?
\end{question}
The reasoning capabilities of deep learning models has swiftly come into sharp focus in the LLM community with most modern LLMs routinely being used, or designed~\cite{li2019modeling,li2020graph,lu2023dynamic}, for prompt-based reasoning expressivity is progressively being re-framed in a more classical AI sense as the ability of an AI to \textit{reason} through mathematical problem solving~\cite{didolkar2024metacognitive,ahn2024large} deductive logic~\cite{riegel2020logical,zhang2024extracting}, inductive reasoning~\cite{cropper2022inductive}, counterfactual thought~\cite{DBLP:journals/corr/abs-2307-02477}, multi-modal~\cite{li2023blip}, language-based deductive reasoning~\cite{hsu2021hubert,chen2022wavlm}, code-generation~\cite{cohn2010inducing,raffel2020exploring}, and a growing wealth of modalities and use-cases.  

The disconnect between approximation theory and AI-reasoning communities raises a hope for answering~\eqref{eq:Main_Question} by re-examining our approximation-theoretic tools, and the quantification of ``structure'' of a function; guided by the lens of the AI-reasoning community.  The main purpose of this paper is to take a stride forward towards answering~\eqref{eq:Main_Question}, by probing the relationship between \textit{algorithmic reasoning}, well-structured functions, and the approximation theory of deep neural networks.  Our perspective diverges from the smoothness-based viewpoint influenced by classical constructive approximation~\cite{wojtaszczyk1997mathematical,MR1393437,MR1990555} rooted in its Besov-theoretic origins~\cite{besov1959OG,TriebelDomains_2008}, by adopting a new \textit{algorithmic complexity}-centric perspective.

This gap is rooted in the fact that approximation theory oversimplifies functions by viewing them simply as input–output maps; overlooking the fact that functions are algorithmically definable relative to a language, in the sense of mathematical logic~\cite{MarkerModelTheory2002}, where computation is defined by the compositional execution of elementary computations.  
From the logical lens, the complexity of a function is quantified by the number of these elementary computations, called gates, required to compute it.  These gates belong to a pre-specified dictionary of algebraic, logical, and analytic operations, which can be understood as \textit{elementary reasoning steps} and the directed acyclic graph tracing the composition of the elementary computations executed by these gates formalizes \textit{chain-of-thought}/sequential reasoning conceptualized in the recent LLM literature~\cite{wei2022chain,nye2022show}. 
Consequently, we understand a neural network architecture \textit{being able} to efficiently approximate a function only if it can efficiently encode the underlying algorithm computing it; that is, only if the network can correctly emulate each step (or gate) in the sequential reasoning, or chain-of-thought, in the algorithm computing/defining the function being approximated.

\begin{figure}[H]
    \centering
    \begin{subfigure}[t]{0.5\textwidth}
        \centering
        \includegraphics[height=1.2in]{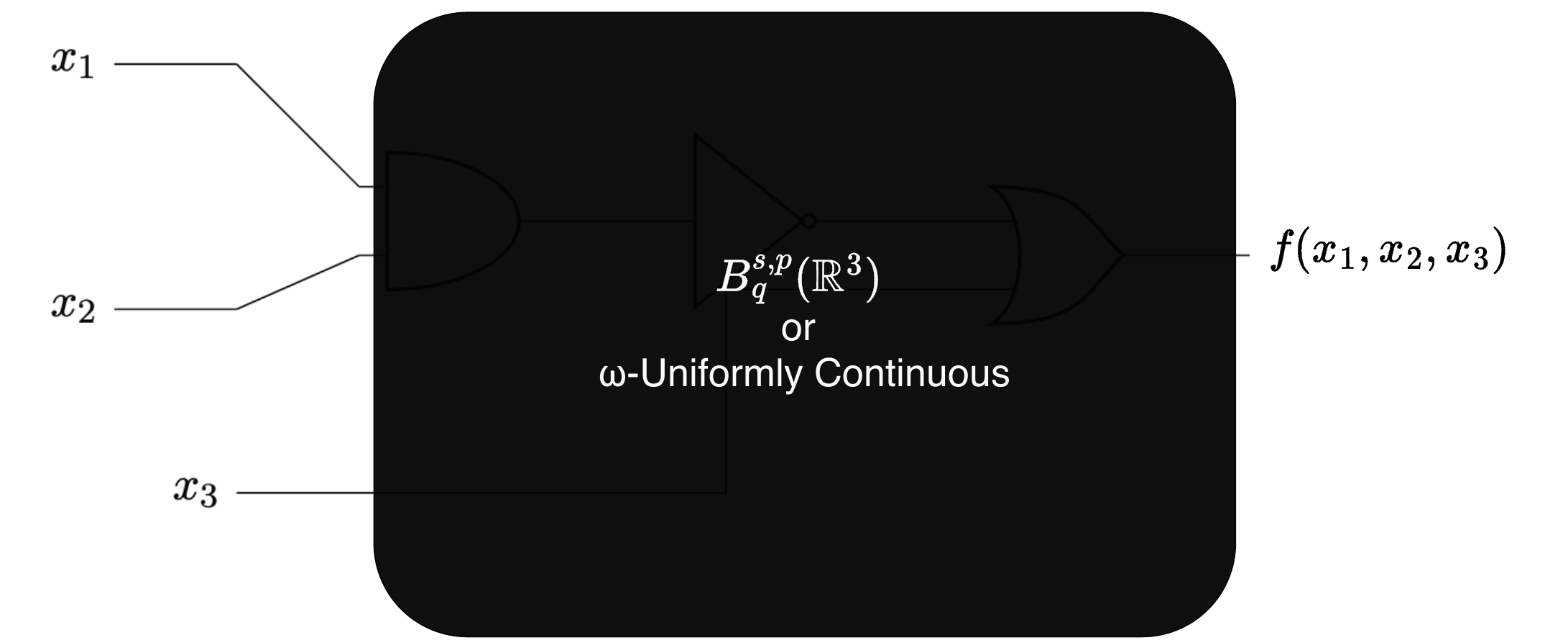}
        \caption{Classical approximation theory}
    \end{subfigure}%
    ~ 
    \begin{subfigure}[t]{0.5\textwidth}
        \centering
        \includegraphics[height=1.2in]{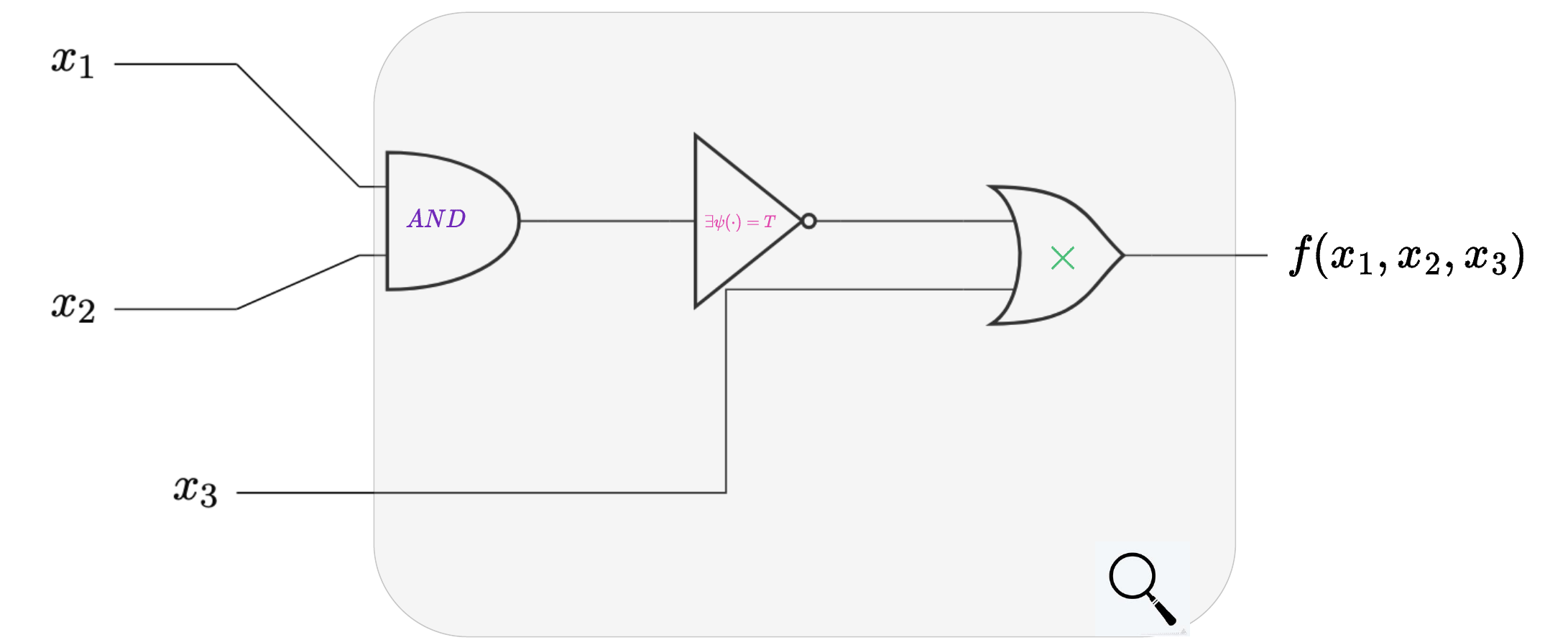}
        \caption{Ours: approximation theory + circuit complexity}
    \end{subfigure}
    \caption{Comparison: \textit{Classical approximation theory vs.\ our complexity-theoretic refinement.}
    \hfill\\
    \textbf{Left:} Classical approximation perspective views functions $f$ as black-boxes which, when queries with an input $x\in [0,1]^d$, yield some output $f(x)\in \mathbb{R}^D$.  All that is assumed of the black-box is a priori knowledge of its regularity; e.g.\ smoothness or uniform continuity, and the best one can hope for is interpolation or approximation (since no internal structure is known of $f$).
    \hfill\\
    \textbf{Right:}
    We \textit{augment} the classical approximation-theoretic perspective with a \textit{circuit complexity} lens by incorporating information on the computation burden required to realize $f$.  This creates transparency, reducing black-box settings to a \textit{white-box} access of the algorithm computing $f$.
    The elementary reasoning steps are each gate in the circuit computing $f$, the chain-of-thought, is the end-to-end sequence of computations, and reasoning emulation is not only interpolation of $f$ but also exact \textit{emulation} of each step in the computational flow.}
\end{figure}

To best convey our frame of reference, let us briefly examine the common proof strategy employed when constructing every approximation guarantee that the authors are aware of, from generalist universal approximation theorems to specialist efficient approximation results.  Every approximation theorem of which the authors are aware operates in two broad stages.  First, they construct an optimal function that computes a solution to the given problem of interest; then, a neural network is constructed that can (approximately) implement that function.  Examples range from worst-case (universal) function approximators~\cite{yarotsky2018optimal} which encode wavelets, optimal/robust/regular interpolators~\cite{vardi2021optimal,egosi2025logarithmicwidthsufficesrobust,hong2024bridging} e.g.\ which encode tent functions on the Kuhn triangulation, effective sparse deep learning procedures which emulate compressed sensing algorithms~\cite{adcock2024learningsmoothfunctionshigh,franco2024practicalexistencetheoremreduced}, neural operator-based PDE solvers achieving fast convergence rates~\cite{doi:10.1137/21M1465718,furuya2024simultaneouslysolvingfbsdesneural,furuya2024quantitativeapproximationneuraloperators,kratsios2025generative} which emulate fixed-point iterations or compute rapidly converging series, etc.
In each case, the constructed neural networks solve problems that are algorithmically computable and the complexity of the neural network required to solve a given problem reflects the complexity of the best possible algorithm for solving that problem.  

By understanding reasoning in this way, this suggests the following quantitative reformulation of Question~\eqref{eq:Main_Question} which also proposes an interpretation of ``expressivity of neural networks'' which more closely expresses the popular understanding of artificial intelligence-based reasoning
\begin{question}
\label{eq:Main_Question__V2}
If a function is computable, how large must a neural network be to compute it?
\end{question}

The idea underpinning our solution to~\eqref{eq:Main_Question__V2} is summarized graphically in Figure~\ref{fig:TLDR}.  
\begin{figure}[ht!]
    \centering
    \includegraphics[width=0.65\linewidth]{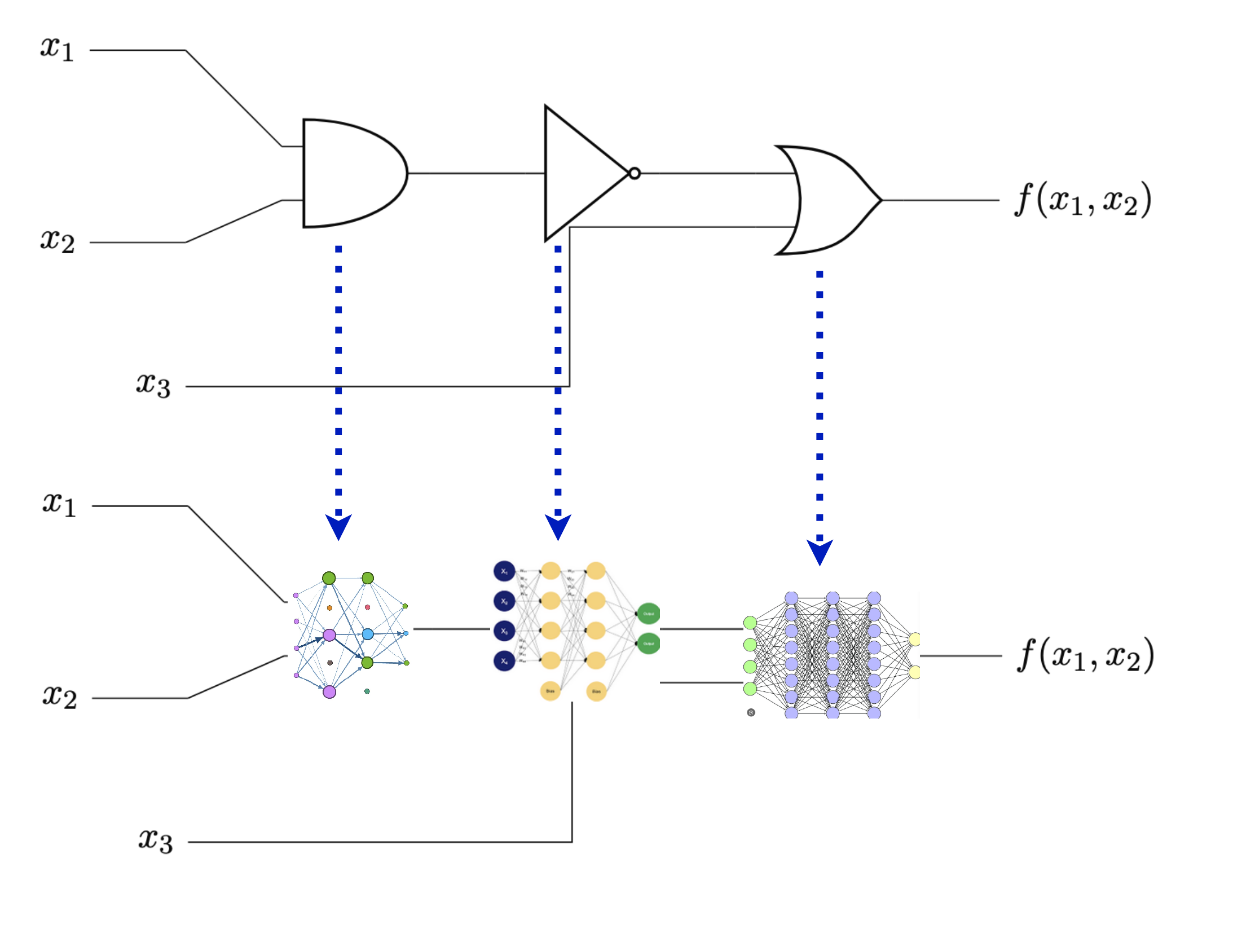}
    \caption{\textbf{Summary of Main Result:} Given any (approximate) representation of a computable function as a circuit, we replace each computational node with an elementary ReLU neural network—selected from a prespecified dictionary—that emulates the corresponding computation.}
    \label{fig:TLDR}
\end{figure}
Broadly speaking, our main result (Theorem~\ref{thrm:Main}) shows that any computable function can be approximated by a circuit—a directed acyclic graph where each internal node performs an elementary computation. Thus, by constructing a dictionary of ``elementary neural networks'' that implement each permissible computation, any computable function can be realized by replacing each node in the circuit with its neural counterpart.  
Moreover, the resulting network mirrors the circuit it computes since:
\begin{description}[leftmargin=!,labelwidth=\widthof{\textbf{Step 1:}}]
    \item[\textbf{Time-to-Space Complexity:}] The NN has $\approx$ the same number of \textit{neurons} (\textit{space complexity}) as the number of gates in the circuit (\textit{time complexity}),
    \item[\textbf{Graph Structure Preservation:}] The NN's computational graph has roughly the same \textit{shape} as the circuit it computes.
\end{description}

Our main result, covering cases where an explicit algorithm cannot be elicited in \textit{Step 1} above, by our \emph{universal} representation theorem (Theorem~\ref{thrm:WorstCaseUniversalGate}), which guarantees that any gate, and hence any circuit, can be \emph{exactly} implemented by a large generic neural network on a digital computer.

\paragraph{How to \textit{Use} Our Main Result: Tailored NN Approximation Guarantees}
Our main result (Theorem~\ref{thrm:Main}) can be viewed as a unifying \emph{meta-theorem} that subsumes existing approximation-theoretic results and serves as a general blueprint for deriving new ones. In particular, it yields the following pipeline for producing \emph{custom} approximation theorems, tailored to a specific setting, when working on a digital computer:
\begin{description}[leftmargin=!,labelwidth=\widthof{\textbf{Step 1:}}]
    \item[\textbf{Step 1:}] Specify a circuit efficiently computing the target function.
    \item[\textbf{Step 2:}] Represent the circuit as a neural network via Table~\ref{tab_thrm:Main}, thereby invoking Theorem~\ref{thrm:Main}.
\end{description}
We demonstrate this across a range of applications, showing how our main result directly yields representation guarantees for randomized Boolean circuits~\ref{cor:BoolFunctions}, dynamic programming problems on graphs~\ref{cor:ASP_MLPs}, and recursive computation via stopped Turing machine emulation~\ref{cor:Turing}.

Moreover, our main theorem provides a basic form of neural network \emph{interpretability}: the constructed model computes a traceable sequence of operations—a “chain of thought”—that implements a concrete circuit performing a well-defined task.

\paragraph{Theorem~\ref{thrm:Main} is More Powerful than Universal Approximation}
For scenarios where one does not wish to work under the digital computing assumption, e.g., allowing real-valued inputs and outputs as in classical approximation theory, Lemma~\ref{lem:Basic_Decomposition} can be applied directly to obtain an approximate analogue of the corresponding result. Thus, our main theorem \emph{implies} the \emph{universality} of neural networks for real-valued inputs and outputs (Corollary~\ref{cor:Continuous_Computable}).

\begin{remark}[Why Feedforward Neural Networks with ReLU Activation Functions?]
\label{rem:Why_ReLU}
We need only consider feedforward neural networks, not necessarily MLPs, with the ReLU activation function of~\cite{fukushima1969visual}. This is because there are procedures for converting such into other standard deep learning models, e.g., convolutional networks~\cite{petersen2020equivalence}, multi-head transformers~\cite{kratsios2025context}, and spiking neural networks~\cite{singh2023expressivity}. Similarly, networks using any standard activation function are also known to be approximately representable as ReLU networks with similar size and structure; see e.g.~\cite{zhang2024deep}.
\end{remark}

\subsection{Related Literature}
\label{s:Introduction__ss:RelatedLiterautre}

\paragraph{Neural Network and Algorithmic Complexity Theory}

The computational implications of neural networks have a long and active history. Early work investigated the Turing completeness of recursive deep learning models, particularly recurrent neural networks~\cite{siegelmann2012neural,perez2018on,chung2021turing,bournez2025universal}, and their ability to emulate push-down automata~\cite{zeng1993discrete}. More recently, the computational power of modern architectures such as autoregressive transformers has been studied in~\cite{schuurmans2024autoregressive}.

Beyond recursive models, the computational capabilities of non-recursive architectures have also been explored through the lens of Boolean circuit complexity. For instance, it has been shown that there \textit{exist} (possibly large) transformers capable of encoding finite circuits composed of ${\operatorname{AND}, \operatorname{OR}, \operatorname{NOT}}$ gates, corresponding to the $AC^0$ complexity class~\cite{li2024chain}. The class $TC^0$, which additionally allows for majority gates, was studied in~\cite{CircuitRep,chiang2025transformers}.

There is also limited work exploring the interplay between circuits and feedforward networks~\cite{parberry1994circuit,vollmer1999introduction,karpinski1997polynomial}; however, this line of research focuses on learning-theoretic rather than complexity-theoretic questions. Perspectives from algebraic circuit complexity have also been considered, notably in~\cite{wang2024compositional}.  
We also mention the growing literature showing the computational intractability of perfectly optimized neural networks from the complexity theoretic vantage point; e.g.~\cite{froese2022computational,boob2022complexity,
bresson2020neural,bournez2023delta,wurm2023complexity,
bertschinger2023training,brand2023new,froese2023training,ganiantraining2025}, their limits in terms of formal verification~\cite{wurm2024robustness}, and representing numerical solvers (which are algorithms) neural networks; e.g. certain ODE solvers are residual-type neural networks~\cite{larsson2017fractalnet,zhang2017polynet} for ODEs, certain integrators~\cite{bresson2020neural}as neural networks, and certain fixed-point algorithms in convex optimization~\cite{FixedPointIteration}.  
We also mention~\cite{hashemi2025tropical} which considers emulation of tropical circuits by transformer-like models.

\paragraph{AI Reasoning in Mathematical Logic}
More recently, the logical perspective has proven fruitful in recasting ``chain-of-thought'' reasoning as questions in circuit complexity. While a direct connection to AI reasoning remains nascent, early deep learning literature briefly explored links between definable formulae, circuit complexity theory, and VC-dimension~\cite{blumer1989learnability,karpinski1997polynomial}. These ideas were revisited in a more informal fashion in~\cite{montana2009vapnik}, which considered sample complexity for formula classes, and further developed in~\cite{yarotsky2018optimal}, which studied the function approximation capabilities of neural networks.  
The connections between specific graph-based algorithm computation neural reasoning has also begun to see significant recent investigation, e.g.\ in~\cite{frydenlund2025language,sanford2024understanding,sanford2024transformers}.  There have been some preliminary links between computational structure and learnability in~\cite{mhaskar2016deep,MR4134774,MR4319248} before any connection to computability as well as some work~\cite{luca2025positional} with links to learning thoery.

\paragraph{Implications of Machine-Precision and Quantization On Deep Learning}
Our focus on inputs and outputs with a fixed maximal bit-complexity aligns with the growing literature on the benefits of deep learning on digital computers \cite{merrill2023parallelism,liu2023transformers,kratsios2024tighter} as well as the limitations thereof~\cite{boche2023limitations,boche2025inverse}.  This also avoids purely mathematical pathologies such as the Kolmogorov-Arnol'd representation~\cite{Kolmogorov1961,arnold2009functions}, which as emphasized in the proof of~\cite{KAHANE1975229}, fundamentally hinges on properties of irrational numbers.  We allow our networks access to one order more precision than the data they are processing, which reflects the recent uncovering of the power of quantization in amplifying an AI's predictive performance~\cite{hubara2018quantized,jegelka2022theory,jacob2018quantization}.

\paragraph{The Expressive Potential Graph Neural Networks via Graph Isomorphisms Tests}
Possibly the most widely accepted exception to understanding of neural network expressive potential arises in the geometric deep learning community when formulating the expressive potential of graph neural networks (GNNs).  For that community, a GNN architecture is understood as being expressive if they can distinguish various families of graphs essentially by computing the Weisfeiler-Leman graph isomorphism tests introduced in~\cite{leman1968reduction}; see~\cite{xu2018how,beddar2024weisfeiler}, among many others~\cite{sato2020survey}.  There are, nevertheless, a small number of works which frame GNN expressive power in the classical function approximation lens~\cite{d2024approximation} as well as a few earlier works focusing on their CNN predecessors~\cite{petersen2020equivalence,yarotsky2022universal} (which are effectively GNNs on lattices instead of on general graphs).

\section{Preliminaries}
\label{s:Prelim}
This section contains all the necessary background and aggregates all notation in order to treat our main results.  

\subsection{Digital Computing} 
\label{s:Prelim__ss:DigitalComputing}
Calculations on digital computers are usually performed in binary floating-point arithmetic. Here, arithmetic operations on real numbers are actually constrained to a discrete grid of the form
\begin{equation}
\label{eq:Rp_discretized_reals}
    \mathbb{R}_q
\eqdef 
    \biggl\{ 
        (-1)^{\beta_{q+1}}
        \,
        \sum_{j=-q}^{q} \frac{\beta_j}{ 2^{j}}  :\,
        (\beta_j)_{j=-q}^{q+1} \in \{0,1\}^{2q+2}  
    \biggr\}
.
\end{equation}
In this (radix-2) representation, the parameter $q \in \N$ is referred to as the ``precision'' and the integers $q$ is the maximum exponent.  For convenience, we define the maximal representable value in our number system $M\eqdef 
2^{q+1}-2^{-q}$.  Note that, there at-most $4^{d(q+1)}$ points in $\mathbb{R}_q^d$ for any $d,q\in \mathbb{N}_+$.

The set $\mathbb{R}_q$ is not closed under the standard addition and multiplication of $\mathbb{R}$. This necessitates \textit{overflow handling} methods, i.e., definitions of the sum or product of two elements in $\mathbb{R}_q$ when the standard operations produce values outside that set. Throughout, we adopt a \textit{modular}%
\footnote{Alternative approaches include ``saturating'' addition/multiplication, e.g.~in the Rust language~\cite{klabnik2023rust}%
. These alternatives need not yield standard algebraic structures such as rings.} convention~ 
which is standard in digital hardware \cite{patterson2020computer}, cryptography \cite{menezes1996handbook,shoup2009computational}, hashing \cite{cormen2009introduction,rabin1981fingerprinting}, random number generation \cite{knuth1997seminumerical,lecuyer2012rng}, checksums in network protocols \cite{rfc1071,stallings2020crypto}, and low-level programming languages; e.g.\ C/C++.
In modular arithmetic
\begin{equation}
\begin{aligned}
\label{eq:addition}
    x+_{\mathbb{R}_q}y 
&\eqdef
    (-1)^{\alpha_{q+1}}
    \Sigma^q_{j = -q} 
    \,\big(
        \alpha_j 2^j
    \big)
    +
    (-1)^{\beta_{q+1}}
    \Sigma^q_{j = -q} 
    \,\big(
        \beta_j 2^j
    \big)
    &\pmod{2^{q+1}}
\\
    x
        \times_{\mathbb{R}_q}
    y 
&\eqdef 
    (-1)^{\alpha_{q+1}+\beta_{q+1}}
    \,
        \Sigma^q_{j = -q} 
        \,
            \big(
                \alpha_j 2^j
            \big)
            \times
            \Sigma^q_{j = -q} 
            \,\big(
                \beta_j 2^j
            \big)
        &\pmod{2^{q+1}}
\end{aligned}
\end{equation}
for any $x = \Sigma^{q}_{j = -q} \alpha_j 2^{j}$ and $y = \Sigma_{j = -q}^q \beta_j 2^j \in \mathbb{R}_q$.
The vectorized versions are defined componentwise; akin to their real-valued (infinite-precision) counterparts.

We will show that, these modular operations can be \textit{exactly computed} by ReLU MLPs, in contrast to their non-modular counterparts, which admit only approximate representations; via the standard approximation~\cite[Proposition 3]{yarotsky2017error} via the sawtooth function~\cite{telgarsky2015representation}.  
We also note that modular addition less straightforward to implemented than standard addition \textit{without overflow} by MLPs; which only requires width $2$ and depth $1$.

\subsection{One Extra Degree of Precision for Computation}
\label{s:Prelim__ss:DigitalComputing___sss:ExtraDegree}
Although floating-point numbers are stored with fixed precision (e.g., $64$-bit \texttt{double}), many CPUs perform computations with extra internal precision; for instance, the x87 FPU uses $80$-bit extended registers~\cite{goldberg1991floating,intel-sdm}, while modern standards require fused multiply-add operations and guard digits to ensure correct rounding~\cite{ieee754-2019,hauser-softfloat}. Thus, intermediate computations are often more accurate than both the initial inputs and the final rounded outputs stored in memory~\cite{fog-optimizing}. This extra precision in digital hardware is captured by our neural networks, whose weights and biases lie in $\mathbb{R}_{q+1}$ and therefore operate with one order of precision higher than their inputs and outputs in $\mathbb{R}_q$.

This fundamentally differs from the ``infinite precision'' paradigm in classical approximation-theoretic analyses of deep learning, e.g.~\cite{hornik1989multilayer,yarotsky2017error}, where inputs, outputs, weights, and biases are all assumed to have the same infinite precision. These results ignore that intermediate computations often use higher precision than inputs and outputs, and thus fail to distinguish \textit{computational} from \textit{storage precision}; a distinction which forms a key contrast between finite and infinite precision regimes.  

The key distinction lies in order of operations: compute first in finite precision then take the limit, or take the limit first then compute. Influenced by classical approximation theory, classical results follow the latter paradigm while this paper adopts the former and recovers standard uniform approximation results (Corollary~\ref{cor:Continuous_Computable}) as a direct corollary of our worst-case guarantee (Theorem~\ref{thrm:WorstCaseUniversalGate}).

\subsubsection{Infinite Precision (Real-Valued) Functions on Digital Computers}
Key to our analysis is the idea of being indistinguishable up to machine precision, as formalized by the following equivalence relation.
\begin{definition}[Indistinguishably]
Let $f,g:\mathbb{R}^d\to \mathbb{R}$.  We say that $f$ and $g$ are indistinguishable at a precision level $p$ for the base $b$, or simply write $f\sim g$, if: for all $x\in \mathbb{R}^d_q$ 
\[
    f(x)= g(x)
.
\]
\end{definition}
\begin{remark}
Let $\mu$ be the uniform probability measure on $\mathbb{R}^d_q$.  Then, for measurable $f,g:\mathbb{R}^d\to \mathbb{R}$, $f\sim g$ if and only if $f$ and $g$ belong to the same equivalence class in $L^{\infty}_{\mu}(\mathbb{R}^d)$.
\end{remark}

\paragraph{Rounding Down Onto the Grid}
When working with real-valued functions in Section~\ref{s:Algos__ss:Approximable_Algoritmhs}, both inputs and outputs must be ``placed on the grid'' $\mathbb{R}_q$. 
For each $q,d\in \mathbb{N}_+$, we fix a choice of a \textit{rounding scheme }on $\mathbb{R}^d$ to precision $2^{q}$; by which we mean any fixed choice of a \textit{metric projection} $\pi^d:\mathbb{R}^d\to \mathbb{R}^d_{q}$ satisfying
\begin{equation}
\label{eq:met_proj}
        \|\pi_q^d(x)\|_{\infty}
    =
        \min_{z\in \mathbb{R}^d_{q}}\,
            \|z-x\|_{\infty}
\end{equation}
for all $x\in \mathbb{R}^d$; the choice of the $\ell^{\infty}$ metric is inconsequential but convenient\footnote{This is similar to some modern pathwise approaches to rough analysis, where one fixes a choice of dynamic refinements on which rough integrals are defined, e.g.~\cite{Pourba2023quadratic,allan2024cadlag}.}.  
\begin{example}
\label{ex:rounding}
If $D=1$ we may consider the rounding scheme $\pi_{q}^1: \R \to \M$ for each $x\in \mathbb{R}$ by
\begin{align}
\label{eq:rounding}
    \pi_{q}^1(x)
\eqdef 
 \mathrm{sign}(x) \, \sup \{a\in \M:\, 0 \le a \le |x|\}
.
\end{align}
\end{example}
This rounding scheme allows us to work with mathematical'' functions $f:\mathbb{R}^d\to \mathbb{R}^D$—defined beyond our number system $\mathbb{R}_q$—by projecting'' them down to a rounded function $\bar{f}:\mathbb{R}_q^d\to \mathbb{R}_q^D$. Concretely, $\bar{f}$ is obtained by restricting the domain of $f$ to the grid $\mathbb{R}_q^d$ and then projecting its outputs to $\mathbb{R}_q^D$, whenever they do not lie thereon, using our chosen rounding operation satisfying~\eqref{eq:met_proj}. Once we fix $\pi_q^D$, the map $\bar{f}$ is defined for each $x\in \mathbb{R}^d$ as
\begin{equation}
\label{eq:rounded_function}
\bar{f}\eqdef \pi_q^D\circ f|_{\mathbb{R}_q^d}
.
\end{equation}
If $f(\mathbb{R}_q^d)\subseteq \mathbb{R}_q^D$ then $\bar{f}=f$; meaning that (extensions of) functions $f$ define between the digital computing grids $\mathbb{R}_q^d$ and $\mathbb{R}_q^D$ are preserved under the rounding operation; $f\mapsto \bar{f}$.
We highlight that, while a general $f$ may not be emulatable by a ReLU network, while it may be possible for its rounded version $\bar{f}$.   Theorem~\ref{thrm:WorstCaseUniversalGate} confirm this, and Theorem~\ref{thrm:Main} shows that the network may be small.

\subsection{Circuits}
\label{s:Prelims__ss:Algos}
We now formalize  $\mathbb{G}$-circuits.  We begin by formalizing the \textit{computational graph} underlying any circuit, and any feedforward neural network.

\subsubsection{Directed Acyclic Graphs (DAGs)}
\label{s:Prelims__ss:Algos___sss:DAGs}
In what follows, we consider connected \textit{directed acyclic graphs} (DAGs).  A DAG is a pair of $D=(V,E)$ of a (possibly infinite) non-empty set $V\subseteq \mathbb{N}$ and directed edges $E\subset \{(v_1,v_2)\in V^2:\, v_1\neq v_2\}$ with the property that $E$ has no cycles; i.e.\ there is no finite sequence $(v_n)_{n=1}^N\subseteq V$ satisfying $(v_n,v_{n+1})\in E$ for each $n=1,\dots,N-1$ and such that $v_1=v_N$; such finite sequences are called \textit{paths}.
A parent of a vertex $v\in V$ is a vertex $w\in D$ for which $(w,v)\in E$, in this case, we say that $v$ is the \textit{child} of $w$.  
The set of parents of a vertex $v$ is denoted by $\pa{v}\eqdef \{w\in E:\, (w,v)\in E\}$ and the set of children of $v$ is denoted by $\ch{v}\eqdef \{w\in E:\, (v,w)\in E\}$.

Following standard circuit complexity theory, we will use DAGs to encode the computational graphs of algorithms; as is, for instance, implemented by any contemporary deep learning software and many standard algorithms.

A vertex $v\in V$ is called an \textit{input node} if it has no parents; these represent an input (or a component of a vector of inputs) of an algorithm whose computational pipeline/graph is encoded by $D$.  A vertex $g \in D$ is called an \textit{output node} if it has no children; output nodes represent points an algorithm produces an output (or a component of a vector of outputs).  
The set of input nodes and output nodes are thus, respectively
\[
    \In{D}\eqdef \{v\in V:\, \pa{v}=\emptyset \}
    \mbox{ and }
    \Out{D}\eqdef \{v\in V:\, \ch{v}=\emptyset \}
\]

All other vertices in $D$, i.e.\ those which are neither input nor output nodes, are called \textit{computation nodes}; the set of all computational nodes is
\[
    \comp[D]\eqdef \{v\in D:\, \pa{v}\neq \emptyset \mbox{ and } \ch{v}\neq \emptyset\}.
\]
Computational nodes are precisely those where an algorithm implements some computation (to be formalized shortly).  

Let $d,D \in \mathbb{N}_+$.  We use $\operatorname{DAG}$ to denote the set of connected DAGs with exactly $d$ input nodes, exactly $D$ output nodes, and with finitely many computational nodes $\#\comp[D]<\infty$.  These will be used to encode circuits computing functions from $\mathbb{R}^d$ to $\mathbb{R}^D$.

\subsubsection{Circuits}
\label{s:Prelims__ss:Algos___sss:Circuits}
We study neural networks as non-uniform models of computation by showing that these networks can approximately implement infinite (non-standard) circuits and a broad range of finite (standard) circuits.  Our analysis includes all finite Boolean~\cite{MR2895965}, arithmetic~\cite{MR3308677,MR4238568}, probabilistic Boolean, and even certain Pfaffian circuits~\cite{karpinski1997polynomial,morton2015generalized}.

Let $\mathbb{G}$ be a subset of $\bigcup_{n=1}^{\infty}\,[\mathbb{R}^n:\mathbb{R}]$ containing the identity on $\mathbb{R}$, denoted by $1_{\mathbb{R}}$.  The elements of $\mathbb{G}$ are call $\mathbb{G}$-\textit{gates}, or simply gates when $\mathbb{G}$ is contextually apparent; where $[\mathbb{R}^n:\mathbb{R}]$ denotes the set of all functions from $\mathbb{R}^n$ to $\mathbb{R}$.
An \textit{elementary $\mathbb{G}$-circuit} $\mathcal{A}$ is a triple $\mathcal{A}\eqdef (V,E,\mathcal{G})$ of an elementary DAG $D=(V,E)$ in $\operatorname{DAG}^{d,D}$, with $V\subseteq \mathbb{N}$, and a set $\mathcal{G}\eqdef \{g_v\}_{v\in V}$ 
such that for each input and output node $v\in \operatorname{In}(D)\cup \operatorname{Out}(D)$ and
\[
g_v(x)=x
\]
for each $x\in \mathbb{R}$, and satisfying the \textit{compatibility condition} for each computational node $v\in \operatorname{Comp}(D)$
\begin{equation}
\tag{Comp}
\label{eq:compatability}
        \operatorname{Im}
        \big(
            g_{v_1}
            ,
            \dots
            ,
            g_{v_n}
        \big)
    \subseteq 
        \operatorname{dom}(g_v)
\end{equation}
where $\{v_1<\dots<v_n\} = \pa{v}$ are the set of parent nodes feeding into $v$.
An elementary $\mathbb{G}$-circuit $\mathcal{A}$ \textit{computes} a function $\Rep(\mathcal{A}):\mathbb{R}^d\to \mathbb{R}^D$ defined recursively for each $x\in \mathbb{R}^d$ by:
\hfill\\
Let $\In{D}=(v_i)_{i=1}^d$, with $v_1<\dots<v_d$, define each $x^{(v_i)}\in \mathbb{R}$ by
\[
    (x^{(v_1)},\dots,x^{(v_d)})
    \eqdef
    (x_1,\dots,x_d)
    \mbox{ where } v_1<\dots<v_d
    \footnote{The ordering makes sense since there are distinct integers by definition of $V\subseteq \mathbb{N}$.} 
\]
for each $w\in \comp[D]$ order $\pa{w}=(w_1,\dots,w_{D_w})$ by $w_1<\dots<w_{D_w}$\footnotemark[1] and define $x^{(w)}\in \mathbb{R}$ by
\[
    x^{(w)}
        \eqdef 
    g_w\big(
        x^{(w_1)},\dots,x^{(w_{D_w})}
    \big)
\]
again using the natural ordering on $\Out{D}=(u_i)_{i=1}^D$ given by $u_1<\dots<u_D$, we output
\[
        \Rep[\mathcal{A}](x)
    \eqdef
        (x^{(u_i)})_{i=1}^D
.
\]
In this way, we distinguish an algorithm and the function it computes, as there may be more algorithmic representations of a single function.  This is analogous to the distinction between the parametric MLPs and the function they realize made in~\cite{petersen2018optimal}.

\subsubsection{\texorpdfstring{$\mathbb{G}$-}{}Circuit Surgery}
\label{s:Algos__ss:Surgery}
The formalization of our main result, as illustrated in Figure~\ref{fig:TLDR}, relies on the systematic replacement of the {computations}
in a $\mathbb{G}$-circuit with standard ReLU MLPs (possibly leveraging a skip connection).  This systematic replacement is formalized by the \textit{surgery operation} defined between two ``compatible'' $\mathbb{G}$-circuits as at a specific node.
\begin{definition}[Surgery at a Node]
\label{defn:SurgeryI__NodeLevel}
Let $\mathcal{A}=(V^{\mathcal{A}},E^{\mathcal{A}},G^{\mathcal{A}})$ and $\mathcal{B}=(V^{\mathcal{B}},E^{\mathcal{B}},G^{\mathcal{B}})$ be $\mathbb{G}$-circuits with $\operatorname{Out}(B)=\{v_{out}^{\mathcal{B}}\}$.
Let $v\in 
\operatorname{Comp}(\mathcal{A})
$ and (if it exists) a surjective ``rewiring map'' $f:\operatorname{pa}_v\to \operatorname{In}(\mathcal{B})$.  The $\mathcal{B}$-surgery of $\mathcal{A}$ at $v$ via $f$ is a $\mathbb{G}$-circuit, denoted by $
(f,\mathcal{B})[\mathcal{A}]
=(\tilde{V},\tilde{E},\tilde{G})$, and defined by
\begin{itemize}
    \item[(i)] $
        \tilde{V}
    \eqdef 
        V^{A}\bigcup V^{B}\setminus\{v\}
    $
    \item[(ii)] $
        \tilde{E}
    \eqdef 
            \big(
                E^{\mathcal{A}}\setminus 
                \underbrace{
                    (
                        \operatorname{pa}_v
                        \times 
                            \operatorname{ch}_v
                    )
                }_{\text{Detach $v$}}
            \big)
        \bigcup 
            \underbrace{
                E^{\mathcal{B}}
            }_{\text{Add $\mathcal{B}$}}
        \bigcup
            \underbrace{
                \{
                    (u,f(u)) :\, u\in \operatorname{pa}_v
                \}
                \bigcup
                \{
                    (v_{out}^{\mathcal{B}},u 
                    :
                        \, u\in \operatorname{ch}_v
                \}
            }_{\text{Attach $\mathcal{B}$'s Inputs and Outputs to $v$'s Children}}
    .
    $
    \item[(iii)] $
    \tilde{G}
    \eqdef 
    G^{\mathcal{B}}\bigcup G^{\mathcal{A}}\setminus \{g_v\}$.
\end{itemize}
We call the surgery \textit{flawless} if $f$ is a bijection.
\end{definition}
Having formalized defining a surgery at a given node, we can now formalize replacing several nodes in a $\mathbb{G}$-circuit by several ``replacement'' $\mathbb{G}$-circuits.
\begin{definition}[Surgery]
\label{defn:SurgeryII_Global}
Let $\mathcal{A} = (V^{\mathcal{A}},E^{\mathcal{A}},G^{\mathcal{A}})$ be a $\mathbb{G}$-circuit, $\mathcal{B}_{\cdot}\eqdef 
(\mathcal{B}_{w}\eqdef (V^w,E^w,G^w))_{v\in W}
$ be an ordered family of $\mathbb{G}$-circuits with $W=(w_t)_{t=1}^T$ an ordered subset of $\operatorname{Comp}(\mathcal{A})$ and for each $w\in W$ let $f_w:\operatorname{pa}_w\to \operatorname{In}(\mathcal{B}_w)$ be a surjective ``rewiring map''.  Then, the $\mathcal{B}_{\cdot}$-surgery of $\mathcal{A}$ by $f_{\cdot}\eqdef (f_w)_{w\in W}$ is the $\mathbb{G}$-circuit, denoted by $(f_{\cdot},\mathcal{B}_{\cdot})[\mathcal{A}]$, and defined recursively as follows
\begin{align*}
        \mathcal{A}_0 
    & \eqdef 
        \mathcal{A}
\\
        \mathcal{A}_{t+1} 
    & \eqdef 
        (f_{w_t},\mathcal{B}_{w_t})[\mathcal{A}_t]
            \mbox{ for }
        t=1,\dots,T-1
\\
        (f_{\cdot},\mathcal{B}_{\cdot})[\mathcal{A}]
    & \eqdef 
        \mathcal{A}_T
.
\end{align*}
If, for each $t=1,\dots,T$, $(f_{w_t},\mathcal{B}_{w_t})$ is flawless, then we say that $(f_{\cdot},\mathcal{B}_{\cdot})[\mathcal{A}]$ is \textit{flawless}.  
If, additionally, $T=\#\operatorname{Comp}(\mathcal{A})$ then $(f_{\cdot},\mathcal{B}_{\cdot})[\mathcal{A}]$ is \textit{completely flawless}.
\end{definition}

\subsection{Notation}
\label{s:Prelim__ss:SpecialMatrices}
We round off this section by collect all the notation, not already introduced, required to formulate our main results and proofs.  
Formal definitions of every \textit{gate} used to define our circuits can be found in the paper's Appendix Section~\ref{s:Ann_Gates}. 
Any additional notation, not already introduced thus far, and required for the formulation of our main results is not introduced.
\paragraph{Special Matrices}
\label{s:Prelim__ss:SpecialMatricesII}
In what follows, we will make use of the following ``special'' vectors and matrices.
For every $B\in \mathbb{N}_+$, we use $\mathbf{1}_B\in \{0,1\}^B$ to denote the vector with all entries all equal to $1$.  We distinguish the 
$B \times 2B$ block matrix $A_{2,B}\eqdef [I_{B\times B}, I_{B \times B}]$ since: for every $a,b\in \mathbb{R}^d$ we have 
$
        a + b = A_{2,B}(a,b)^{\top}
.
$
Similarly, we distinguish the $B \times 2B$ block matrix $A_{2,B}^-\eqdef [I_{B\times B}, -I_{B \times B}]$ since: for every $a,b\in \mathbb{R}^d$ we have 
$
\label{eq:alternating_matrix}
        a - b = A_{2,B}^-\,(a,b)^{\top}
.
$
Let $\Pi_{2,B} \eqdef \begin{pmatrix}
0_{B \times B} & I_{B \times B}\\
I_{B \times B} & 0_{B \times B}
\end{pmatrix}$ denote the $2B\times 2B$ permutation matrix. For each $(a,b)\in \mathbb{R}^{2B}$, we have 
$
\label{eq:permutation_matrix}
    \Pi_{2,B}(a,b)=(b,a)
.
$
\paragraph{Misc}
If $X$ and $Y$ are non-empty sets, we denote the set of all functions from $X$ to $Y$ by $[X:Y]$.

We now present our two main results.  
\section{Main Results}
\label{s:Main_Results}
Our main result, Theorem~\ref{thrm:Main} below, considers the more reasonable cases where the target function is itself computable by a $\mathbb{G}$-circuit with any gate listed in Table~\ref{tab_thrm:Main} (and defined formally in Appendix~\ref{a:MLP_Constructions}) while implementing those gates in the correct underlying computation graph (DAG) structure directly.  Doing so not only emulates the underlying chain-of-thought/reasoning implicit in the target function itself, but it ensures that the approximator has (spatial) complexity scaling with the (time) complexity of the algorithm (circuit) used to compute $f$ itself.  Our result thus formalizes and validates the folklore that \textit{``neural networks trade \textit{time} complexity for \textit{space} complexity''.}

\begin{theorem}[Universal Reasoning - Via Circuit Emulation]
\label{thrm:Main}
Let $\mathcal{A}=(E,V,\mathcal{G})$ be a $\mathbb{G}$-circuit, whose gates are listed in Table~\ref{tab_thrm:Main}.  Then there is a complete flawless surgery $(f_{\cdot},\mathcal{B}_{\cdot})\eqdef (f_{v},\mathcal{B}_v)_{v\in \operatorname{Comp}(\mathcal{A})}$ of $\mathcal{A}$ where, for each $v\in \operatorname{Comp}(\mathcal{A})$, $\mathcal{B}_v$ is a ReLU MLP whose depth and width is bounded described in Table~\ref{tab_thrm:Main} according to its gate, 
and $(f_{\cdot},\mathcal{B}_{\cdot})$ is a ReLU$^{+}$-FFNN.
\end{theorem}

\begin{table}[htp!]
\begin{adjustbox}{width=\textwidth, height=\textheight, keepaspectratio,center}
\small
\begin{tabular}{@{\,}%
  >{\raggedright\arraybackslash}l  
  >{\centering\arraybackslash}l    
  >{\centering\arraybackslash}l    
  >{\centering\arraybackslash}l    
  >{\raggedright\arraybackslash}l  
@{}}
\toprule
\textbf{Gate} & \textbf{Depth} & \textbf{Width} & \textbf{No.\ Param.} & \textbf{Reference} \\
\midrule

\multicolumn{5}{c}{\textbf{\hyperref[s:BitManipulators]{Bit Manipulators}}} \\[3pt]

Bit Decoder (see~\eqref{eq:Bit_Decoder})
  & $1$
  & $2q+2$
  & $2q+10$
  & Prop.~\ref{prop:bit_decoder} \\

Bit Encoder (see~\eqref{eq:Bit_Encoder})
  & $3\lceil\log_2(2M+1)\rceil+9$
  & $4(q+1)(32M+17)$
  & $404M+256Mq+155q+171$
  & Prop.~\ref{prop:bit_encoder} \\

Left Shift \LEFT
  & $1$
  & $B$
  & $2B-1$
  & Lem.~\ref{lem:left} \\

Right Shift \RIGHT
  & $1$
  & $B$
  & $2B-1$
  & Lem.~\ref{lem:right} \\

\midrule
\multicolumn{5}{c}{\textbf{\hyperref[s:ArithmeticBitWise_Gates]{Bitwise Arithmetic Gates}}} \\[3pt]
Addition $+^n$
  & $2Bn$
  & $B\lfloor\log_2(n)\rfloor$
  & ---
  & Lem.~\ref{lem:add}\\

Multiplication $\times^2$
  & $B^2$
  & $B^2$
  & ---
  & Lem.~\ref{lem:mult} \\

Identity $I_n$
  & $1$
  & $2n$
  & $4n$
  & \footnotesize \cite[Lemma 5.1]{petersen2024mathematical} \normalsize \\

\midrule
\midrule
\multicolumn{5}{c}{\textbf{\hyperref[s:ModularArithmeticGrid]{Modular Arithmetic Gates}}} \\[3pt]
Addition $+_{\mathbb{R}_q}$
  & $\mathcal{O}(q)$
  & $\mathcal{O}(q)$
  & ---
  & Lem.~\ref{lem:Modular_Add} \\

Multiplication $\times_{\mathbb{R}_q}$
  & $\mathcal{O}(q^2)$
  & $\mathcal{O}(q^2)$
  & ---
  & Lem.~\ref{lem:Modular_Mult} \\

\midrule
\midrule
\multicolumn{5}{c}{\textbf{\hyperref[s:Ann_Gates]{Analytic Gates}}} \\[3pt]
Constant $x\mapsto c$
  & $1$
  & $2n$
  & $4n$
  & Lem.~\ref{lem:constant} \\

Indicator Half-Line $I_{[a,\infty)}$
  & $2$
  & $1$
  & $5$
  & Lem.~\ref{lem:Indictor__PositiveLine} \\

Indicator Compact Interval $I_{[a,b]}$ ($a<b$)
  & {$2$}
  & {$2$}
  & {$9$}
  & Lem.~\ref{lem:Closed_Interval} \\

Indicator of Point $I_{\cdot=a}$
  & $4$
  & $4n$
  & $18n+5$
  & Lem.~\ref{lem:neural_spike} \\

Ceiling $\lceil\cdot\rceil$
  & $\lceil\log_2(2^{1+q}+1)\rceil+4$
  & $2^{1+q}+1$
  & $2^5\,(2^{1+q}+1)$
  & Lem.~\ref{lem:ceiling} \\

\midrule
\midrule
\multicolumn{5}{c}{\textbf{\hyperref[s:Operations__ss:Logic]{Predicate Logic on $B$ Bits}}} \\[3pt]

Equality $\equiv$
  & {$3$}
  & {$2B$}
  & {$9B$}
  & Lem.~\ref{lem:EQUAL} \\

NAND
  & $2$
  & $B$
  & $6B$
  & Lem.~\ref{lem:NAND} \\

NOT
  & $1$
  & $B$
  & $3B$
  & Lem.~\ref{lem:not} \\

AND
  & $1$
  & {$B$}
  & $4B$
  & Lem.~\ref{lem:and} \\

OR
  & $1$
  & {$B$}
  & {$6B$}
  & Lem.~\ref{lem:or} \\

XOR
  & {$2$}
  & {$2B$}
  & {$8B$}
  & Lem.~\ref{lem:xor} \\

Implication
  & {$1$}
  & $B$
  & $4B$
  & Lem.~\ref{lem:imply} \\

\midrule
\multicolumn{5}{c}{\textbf{Propositional Logic on $B$ Bits}} \\[3pt]
$(\forall x\in\{0,1\}^{B})\,\phi(x,\cdot)=1$
  & $5 + B + D$
  & $B + \bigl(W + 3\cdot2^{B}\bigr)$
  & $13B + \bigl(1 + S + 2^{2B+5}\bigr)$
  & Lem.~\ref{lem:1rstOrderLogic_Verification} \\

\midrule
\midrule
\multicolumn{5}{c}{\textbf{Tropical Operations}} \\[3pt]
Maximum of $n$ Inputs
  & $\lceil\log_2(n)\rceil$
  & $3n$
  & $16n$
  & Lem.~\ref{lem:Minimum} \\

Median of $2n+1$ Inputs
  & $11n+3$
  & $6n+3$
  & ---
  & Lem.~\ref{lem:median_estimators_have_small_errors} \\

Majority Vote Between $n$ Inputs
  & $11\lfloor n/2\rfloor+4$
  & $6\lfloor n/2\rfloor+3$
  & ---
  & Lem.~\ref{lem:Maj} \\


\bottomrule
\end{tabular}
\end{adjustbox}
\caption{Reference for $\operatorname{ReLU}$ MLP emulator of each gate in $\mathbb{G}$ in Theorem~\ref{thrm:Main}, and summary of their depth, width, and the number of non-zero parameters.} 
\label{tab_thrm:Main}
\end{table}

At this stage, Theorem~\ref{thrm:Main}, covers most functions one may encounter in practice.  However, it is natural to wonder if there are any circuits which are not implementable by a ReLU neural network.  Our following result shows that no such circuit exists as every function mapping from the finite precision spaces $\mathbb{R}_q^d$ to $\mathbb{R}_q^D$ is \textit{exactly} representable by a ReLU neural network.  This result is thus not an approximation theorem, which has a positive approximation error, but more akin to a digital computing analogue of the Kolmogorov-Arnold superposition theorem (see e.g.~\cite{Kolmogorov1956,arnol1959representation}).  Namely, we show that every function on a digital computer can be \textit{exactly} emulated by a ReLU MLP of double its precision up to machine precision; i.e.\ up to the $\sim$ relation.  However, as with most universal approximation guarantees, the number of parameters required for the emulation of the worst-behaved function is wildly infeasible. 

\begin{theorem}[Universal Circuit Design By ReLU Networks]
\label{thrm:WorstCaseUniversalGate}
Fix $d,q\in \mathbb{N}_+$ and define $N_1\eqdef (2^{2q+2}-1)^d\le 4^{d(q+1)}$.
There is a ``universal encoder'' ReLU MLP $\Phi_{\operatorname{Enc}}:\mathbb{R}^d\to \mathbb{R}^{N_1}$ 
of depth $4$, width $2N_1 \in \mathcal{O}(4^{dq})
$, with less than $2N_1 \in \mathcal{O}(4^{dq})
$ non-zero parameters, such that: for each 
$f:\mathbb{R}^d_q\to \mathbb{R}_q$ there exists a $\beta_f \in \mathbb{R}^{N_1}$ with entries in $\{f(x):\,x\in \mathbb{R}_q^d\} \subseteq \mathbb{R}_q$ satisfying: for all $x\in \mathbb{R}_q^d$
\[
        \beta_f^{\top}\Phi_{\operatorname{Enc}}(x)
    =
       f(x)
\]
and $\Phi_f\eqdef \beta_f^{\top}\Phi_{\operatorname{Enc}}$ is a ReLU MLP of depth $4$, width $\mathcal{O}(4^{dq})$, with $\mathcal{O}( 4^{dq})$ non-zero parameters.
\end{theorem}
A sketch of a pseudo-code for implementing a (very similar) circuit to the one in Theorem~\ref{thrm:WorstCaseUniversalGate} is sketched in~\eqref{a:UniversalCircuit}.

A closer look at the construction in Theorem~\ref{thrm:WorstCaseUniversalGate}, mirrors the constructions of most optimal universal approximation theorems for non-smooth functions; e.g.~\cite{yarotsky2018optimal,Shen} without regularity constraints as in~\cite{hong2024bridging,riegler2024generating}, shows that the encoder $\Phi_{\operatorname{Enc}}$ places a spike at each ``grid'' points $p$ in $\mathbb{R}_q^d$, i.e.\ an indicator function of $I_{x=p}$ and then scales the value of that indicator according to the target function value $f(p)$ at that point.  As illustrated in Figure~\ref{fig:WorstCaseApproximator}, this {simply constructs}
the ``spiky surface'' $\sum_{q\in \mathbb{R}_q^d}\, f(p)I_{x=p}$.  This function can be computed by an algorithm, see Appendix~\ref{a:UniversalCircuit}, whose run time is lower-bounded by the number of spikes $f(p)I_{x=p}$ needed to be constructed, which is $\Omega(p^d)$ and thus suffers from the curse of dimensionality.   

\begin{figure}[htp!]
    \centering
    \includegraphics[width=1\linewidth]{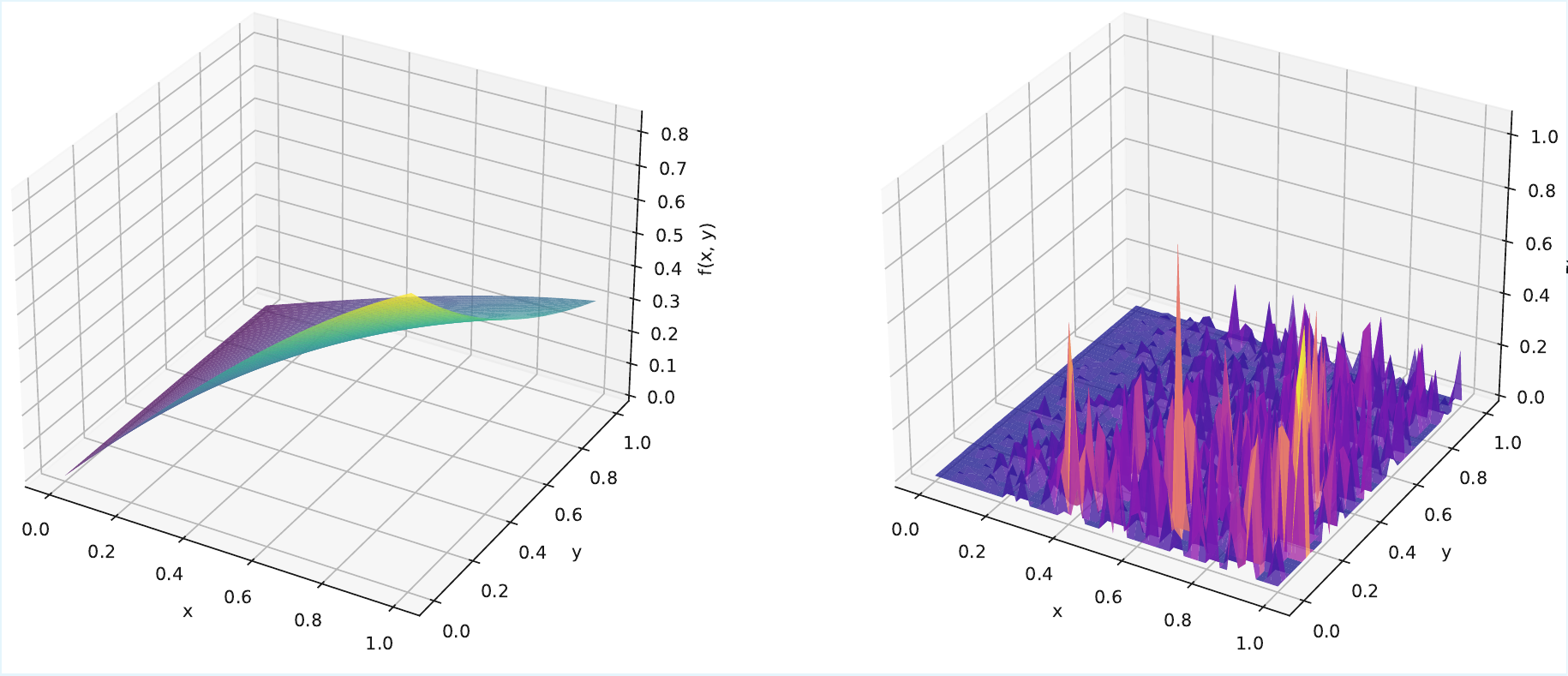}
    \caption{Illustration of the Worst-Case Approximator constructed in Theorem~\ref{thrm:WorstCaseUniversalGate}.  \hfill\\
    Ground truth (left) approximator (right).}
    \label{fig:WorstCaseApproximator}
\end{figure}

Theorem~\ref{thrm:WorstCaseUniversalGate} uses a ReLU feedforward neural network to compute a circuit which separately assigns a value to every possible input pointing $\mathbb{R}_q^d$.

We now examine several implications of our main result before delving into its proof.

\section{Applications}
\label{s:Applications}

We showcase some of the implications of Theorem~\ref{thrm:Main} before providing the rather lengthy proof of our main result.

\subsection{Computability Implications and Circuit Complexity}
\label{s:Applications__ss:Comput}

We now focus on the computability theoric implication of our main result, from implications on Turing-type computability to randomized and deterministic boolean circuit complexity guarantees.

\subsubsection{Boolean-Computable Functions}
\label{s:Applications__ss:Comput__sss:GeneralBooleanCircuits}

Let $f:\{0,1\}^B\to \{0,1\}$ be a Boolean function.  A quick computation shows that there are exactly $2^{2^B}$ such functions; thus, the number grows \textit{double exponentially} in the number of input bits, mirroring Theorem~\ref{thrm:WorstCaseUniversalGate}.  Our first application of our theory is on the smallness of the representation of any such function by ReLU MLP.  Namely, we bring down the worst-case complexity for computing a Boolean function by a ReLU FFNN from $\mathcal{O}(2^{2^B})$ to $\mathcal{O}(2^B)$!
\begin{corollary}[Generic Boolean Functions are Simple To implement by ReLU MLPs]
\label{cor:BoolFunctions}
Let $B\in \mathbb{N}_+$.  Then there is an absolute constant $c>0$ such that: for every Boolean function $f:\{0,1\}^B\to \{0,1\}$ 
there exists a ReLU FFNN with at-most
\[
    5\left(1+\frac{c+\log(B)+\log(\log(B))}{B}\right) 
    2^B
    \in \mathcal{O}(2^B).
\]
non-zero parameters.  
\end{corollary}
\begin{proof}[{Proof of Corollary~\ref{cor:BoolFunctions}}]
Since $f$ is Boolean, then the main result of~\cite{lozhkin1996tighter}, as presented in~\citep[Theorem 1.15 and page 26]{JunkaBooleanBook}, implies that there exists an absolute constant $c>1$, and a Boolean circuit $\mathcal{A}$ with gates in $\tilde{\mathbb{G}}\eqdef \{\operatorname{AND}_2,\operatorname{OR}_2,\operatorname{NOT}\}$ with 
\begin{equation}
\label{eq:compute_units_generic_boolean}
    K\eqdef \left(1+\frac{c+\log(B)+\log(\log(B))}{B}\right) \frac{2^B}{B}
\end{equation}
computation units such that 
\begin{equation}
\label{eq:rep}
    \Rep(\mathcal{A})=f
\end{equation}
i.e.\ $\mathcal{A}$ computes $f$.  Now, Lemmata~\ref{lem:or},~\ref{lem:and}, and~\ref{lem:not} imply that each gate in $\tilde{\mathbb{G}}$ can be implemented by a ReLU MLP using no-more than width $B$, depth $1$, and $5B$ non-zero parameters implementing any gate in $\tilde{\mathbb{G}}$.  Thus,~\eqref{eq:rep} and~\eqref{eq:compute_units_generic_boolean} implies that $f$ can be computed by a ReLU MLP circuit with at-most $K5B$ non-zero parameters.
\end{proof}
We emphasize that Corollary~\ref{cor:BoolFunctions} is nearly optimal nearly matches the lower-complexity bound of~\cite{muller2009complexity} for a circuit on $\{\operatorname{AND}_2,\operatorname{OR}_2,\operatorname{NOT}\}$ implementing a generic Boolean function on $B$ bits; which requires at-least $\Theta(2^B)$ computation nodes.

\subsubsection{{Computable Randomized Logical Algorithms\texorpdfstring{$(RTC^0)$}{}}}
\label{s:Applications__ss:Comput__sss:Rand_Log_Algos}
Let $\mathbb{G}$ be a (non-empty) set of Boolean gates and let $\mathcal{A}$ be a $\mathbb{G}$-gate on $n+m$ inputs, where $n,m\in \mathbb{N}_+$.  Let $X_{\cdot}\eqdef (X_{k})_{k=n+1}^{n+m}$ be independent Bernoulli random variables taking value $1$ with probability $p>\frac{1}{2}$, henceforth defined on a common probability space $(\Omega,\mathcal{B},\mathbb{P})$.  We say that the associated $p$-\textit{randomized} circuit $(\mathcal{A},X_{\cdot})$ computes a Boolean function $f:\{0,1\}^n\to \{0,1\}$ with probability $p$ if: for each $x\in \{0,1\}^n$
\begin{equation}
\label{eq:comput_random}
    \mathbb{P}\big(
        \Rep(x,X_{n+1},\dots,X_{n+m}) = f(x)
    \big)\ge p
.
\end{equation}
The class of all such algorithms is typically denoted by $RTC^0$.
\begin{corollary}[Derandomization of Randomized Logical Circuits]
\label{cor:RandCirc}
Let $B,m,k,K,\Delta,w\in \mathbb{N}_+$ with $k\le K$.
Let $f:\{0,1\}^n\to \{0,1\}$ be computable by $p$-randomized circuit with Boolean gates in $\{\operatorname{AND}_2,\operatorname{OR}_2,\operatorname{NOT}\}$ and of size $K$, depth $\Delta$, and width $2$ with each gate having fan-in at-most $k$.  Then, there is a ReLU feedforward neural network $f:\mathbb{R}^B\to \mathbb{R}$ of width $
60BK+6
$ and depth $33BK+3+4\Delta$ computing $f$ (deterministically).
\end{corollary}

\subsubsection{Finite-Time Turing Machine Simulation}
\label{s:Applications__ss:Comput__sss:Turing}

\textit{Transductors} are \textit{Turing machines} well-suited to tasks requiring continuous streaming of inputs and outputs, with its internal logic dictated by the usual Turing machine operations.  Formally, a transductor $\mathcal{M}$ is a specialized Turing machine, defined as a $7$-tuple: $
\mathcal{M} = (Q, \Sigma, \Gamma, \delta, q_0, q_{accept}, q_{reject})
$ where: $Q$ is the set of states, $\Sigma$ is the input alphabet, in this  case $\Sigma=\{0,1\}^B$, $\Gamma$ is the tape alphabet,  
$
\delta: Q \times \Sigma \times \Gamma \to Q \times \Gamma \times \{L, R\} \times \Sigma
$
 is the transition function, $q_0$ is the initial state, $q_{accept}, q_{reject} \in Q$ are the accepting and rejecting states.

At every computational step, $\mathcal{M}$ either reads a symbol to the input tape or writes a symbol to the output tape.  At each step $t\in \mathbb{N}_+$, $\mathcal{M}$ transitions by reading the current input  $i_t \in \Sigma$ and writing a corresponding output symbol $o_t \in \Gamma$, and then changing its state according to $\delta$.  The machine $\mathcal{M}$ is \textit{oblivious}, in the sense that for each step $t$, the transductor (deterministically) must read an input symbol and write an output symbol.  Thus, using the transduction function $\delta$, for each state $q \in Q$ and input symbol $i_t \in \Sigma$, the machine reads $i_t$, processes it to determine the next state, writes an output symbol, and moves the tape head.

Formally, for each $t\in \mathbb{N}_+$, the machine: 1) $\mathcal{M}$ begins state $q_t$, 2) it reads the symbol $i_t$ from the input tape, 3) it then writes the symbol $o_t$ to the output tape (as determined by $\delta$), and 4) it transitions to the next state $q_{t+1}$.  The transductor $\mathcal{M}$ only halts when $q_t=q_{accept}$ or $q_t=q_{reject}$ (similar to a standard Turing machine).  Thus, when $\Sigma=\{0,1\}^B$, for some $B\in \mathbb{N}_+$, then: for each $x\eqdef q_0\in \Sigma$ and every $t\in \mathbb{N}_+$ we write
\[
    \mathcal{M}(x)_t\eqdef q_t
.
\]
We thus have the following guarantee that a ReLU neural network can 

\begin{corollary}[Finite-Time Turing Machine Simulation with \textit{Explicit Chain-of-Thought}]
\label{cor:Turing}
Let $\mathcal{M}$ be a transductor with input alphabet $\Sigma=\{0,1\}^B$ and output alphabet $\{0,1\}$.  Then, for every reasoning time $T\in \mathbb{N}_+$, there is a ReLU neural network $\hat{f}:\mathbb{R}^{B+1}\to \mathbb{R}$ of depth $\mathcal{O}(T)$ and with $\mathcal{O}(T\log(T))$, with $\mathcal{O}$ suppressing the dependence on $B$, computation nodes computing $T$ on $\mathbb{R}^B_q$
\[
    \hat{f}(x,t) = \mathcal{M}(x)_t
\]
for all $x\in \{0,1\}^B$ and each computation during the CoT $t=0,\dots,T$.
\hfill\\
Moreover, $\hat{f}$ also has depth $\mathcal{O}(T)$ and $\mathcal{O}(T^2\log(T))$ computation units.
\end{corollary}

\subsection{Dynamic Programming}
\label{s:Applications__ss:DP}

We now apply our main result to demonstrate that ReLU FFNNs can solve various pure dynamic programming (DP) problems. We focus on graph-based issues due to their central role in geometric deep learning. Specifically, we examine neural network computations for the traveling salesperson problem (TSP) and the calculation of minimal graph distances between all pairs of nodes. The former is particularly notable for extremal graphs such as expanders~\cite{bollobas1998randomBookTime}, illustrated in Figure~\ref{fig:Expander}, while the latter is most relevant for graph approximations of manifolds, a classical technique in manifold learning~\cite{mcinnes2018umap,tenenbaum2000global,balasubramanian2002isomap}, see Figure~\ref{fig:ManGraph}.

\begin{figure*}[htp!]
    \centering
    \begin{subfigure}[t]{0.45\textwidth}
        \centering
        \includegraphics[width=.95\textwidth]{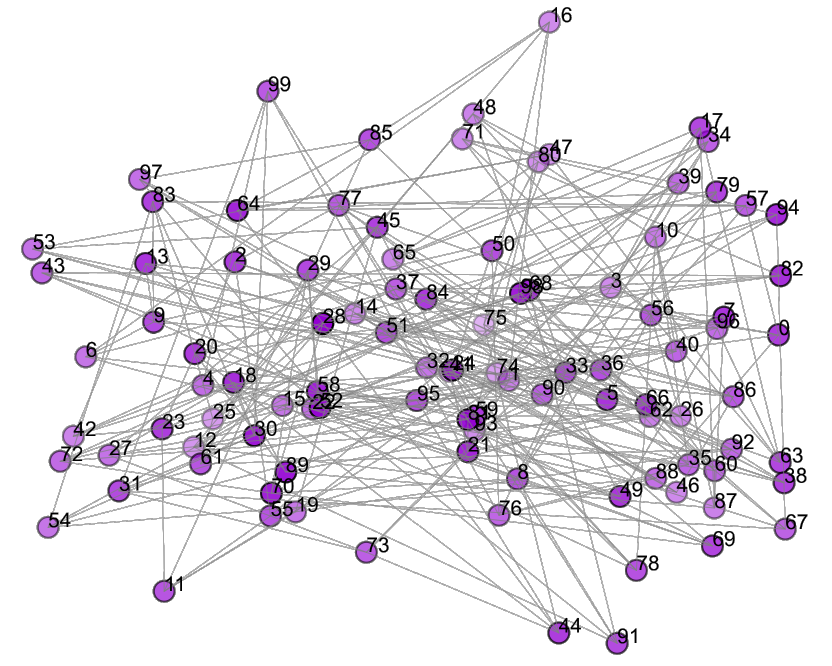}
        \caption{All shortest paths problem - computing the shortest weighted distances between all pairs of nodes on a graph.}
        \label{fig:Expander}
    \end{subfigure}
    ~
    \begin{subfigure}[t]{0.45\textwidth}
        \centering
        \includegraphics[width=.95\textwidth]{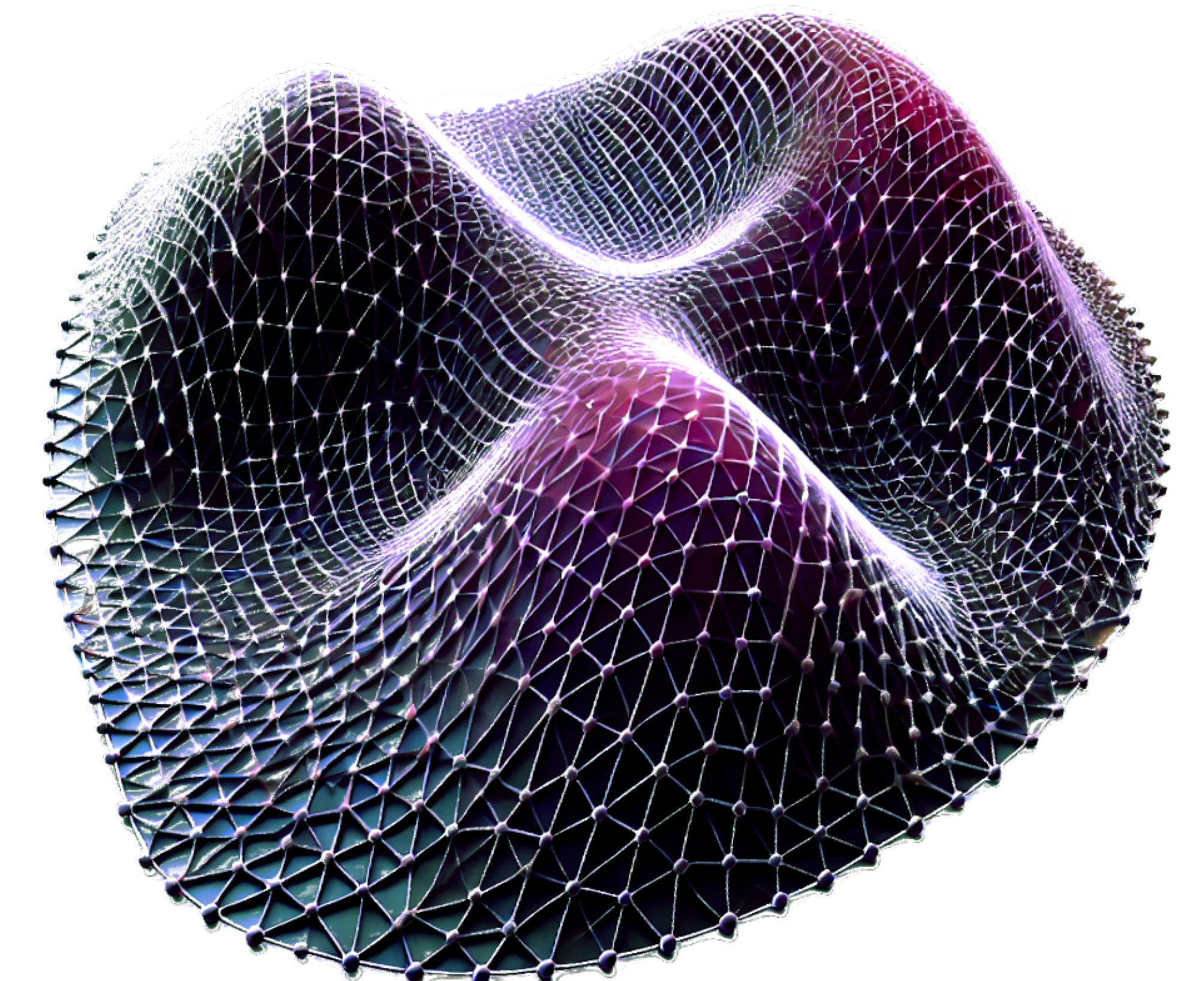}
        \caption{Graph Geodesic Distance - Important on Weights Graph Approximations to Riemannian Manifold.}
        \label{fig:ManGraph}
    \end{subfigure}%
    \caption{Extremal examples of Graphs in Geometric Deep Learning.}
\end{figure*}

Our main result establishes that any positive result in pure dynamic programming (DP)—see, e.g.,~\cite{JunkaTropicalBook}—automatically yields an upper bound on the complexity required by a ReLU feedforward network to solve the corresponding pure DP problem. We demonstrate this by showing that the set of all shortest path distances on any given graph can be computed by a single fixed ReLU network, independent of the specific graph, though dependent on the number of vertices, with only cubic complexity. That is, the network may be large, but it remains computationally feasible.
\hfill\\
Without loss of generality, we restrict attention to the complete graph, as any other graph can be obtained by assigning sufficiently large weights to the edges to be effectively removed.
Let $k\in \mathbb{N}_+$ and denote the complete graph on $V_k\eqdef \{1,\dots,k\}$ by $K_k=(V_k,E_k)$ where $E_k\eqdef \{\{i,j\}:\,i,j\in V_k\mbox{ and }i<j\}$ which we recall has has $k^{\star}\eqdef k(k-1)/2$ distinct edges.  Every possibly assignment of edge weights $w=(w_e)_{e\in E_k}\in (0,\infty)^{k^{\star}}$ induces a \textit{shortest path distance} $d_w$ on $K_k$ defined for any pair of vertices $i,j\in V_k$ by
\[
        d_w(i,j)
    \eqdef 
        \min_{\gamma}
        \,
            \sum_{e \in \gamma} w_e
\]
where the minimum is taken over all paths $\gamma$ from $i$ to $j$; i.e.\ over all sequences of edges $\gamma=(e_1,\dots,e_k)$ where $i_1=i$, $j_k=j$, and for $l=1,\dots,k-1$ we have $e_l=(i_l,j_l)$ and $j_l=i_{l+1}$.  The \textit{all-pairs shortest paths (ASP)} problem to construct a circuit which given any $w\in (0,\infty)^{k^{\star}}$ returns the flattened distance matrix
\[
        D_w
    \eqdef 
        (d_w(i,j))_{{i,j}\in E_k}
.
\]
By~\cite[Example 1.8]{JunkaTropicalBook} there exists a \textit{tropical} $\mathbb{G}\eqdef \{\operatorname{min},+_2\}$ circuit $\mathcal{A}_{ASP}$ with $\mathcal{O}(k^3)$ solving the ASP problem.  That is, for every set of positive edge-weights $w\in \mathbb{R}^{k^{\star}}_q$ we have
\begin{equation}
\label{eq:Sol_ASP_Problem}
        \operatorname{Comp}(\mathcal{A}_{ASP})(w) 
    = 
        D_w
.
\end{equation}
Moreover, the complexity of $\mathcal{A}_{ASP}$ is optimal in the class of such tropical $\mathbb{G}\eqdef \{\operatorname{min},+_2\}$ circuits due to the marching lower-bound in~\citep[Corollary 2.3]{JunkaTropicalBook}.  
Applying Theorem~\ref{thrm:Main} we deduce that
\begin{corollary}[ReLU Networks Solve the ASP problem with Cubic Complexity]
\label{cor:ASP_MLPs}
For every $k\in \mathbb{N}_+$, 
there is a ReLU FFNN $\Phi_{ASP}:\mathbb{R}^{k^{\star}}_q \to \mathbb{R}$ with $\mathcal{O}(k^3)$ non-zero parameters such that
\[
    \Phi_{ASP}(w) = D_w
\]
for each $w \in \mathbb{R}^{k^{\star}}_q$.
\end{corollary}

\subsubsection{{(Optimal) Universal Approximation via \texorpdfstring{$\mathbb{G}$}{}-Computability}}
\label{s:Algos__ss:Approximable_Algoritmhs}

We say that a function $f\in [\mathbb{R}^d:\mathbb{R}^D]$ is $\mathbb{G}$-\textit{computable} if there exists an elementary $\mathbb{G}$-circuit $\mathcal{A}$ such that $f$ is indistinguishable from $\Rep[\mathcal{A}]$ on $\mathbb{R}^d_{q,b}$; i.e.\
\[
\Rep[\mathcal{A}] \sim f
.
\]
Many \textit{mathematically} interesting or even simple functions are not computable in the above sense, e.g.\ the radical function $x\mapsto \sqrt{x}$ on $\mathbb{R}$.  We expand our circuit-based notion of computability to allow for some asymptotic structure often formalized by recursion.  Instead, we draw inspiration from \textit{constructive approximation} theory, see~\cite{CABook2} or its modern deep learning counterpart~\cite{elbrachter2021deep,gribonval2022approximation}, which describes functions as being approximable if they can be approximately expressed at a given worst-case rate.

Computable functions on $\mathbb{R}^d$ and computable reals are typically thought of as those objects for which there exists a sequence of (some) model of computability, e.g.\ circuits or  Turing machines, each of which produces approximate representations of that number to a strictly higher number of decimals of accuracy; see e.g.~\cite{Weihrauch_ComputableAnalysisBook2000}.  Similarly, 
we understand a function as being \textit{computable} in our $\mathbb{G}$-circuit our framework if it can be asymptotically approximated to arbitrary precision by a fixed sequence of machines, e.g.\ circuits or Turing machines, for all computable inputs.    Though these ideas capture the idea of algorithmic complexity, they need not reflect the rate at which functions in a given class can be approximated; or rather, \textit{approximately computed}.

Recall that, for each $q,d\in \mathbb{N}_+$, we fix a choice of a rounding scheme on $\mathbb{R}^d$ to precision $2^{q}$ in~\eqref{eq:met_proj}.
%
Let $\omega:[0,\infty)\to [0,\infty)$ be a modulus of continuity, i.e.\ it is monotone 
increasing, right-continuous at $0$, with $\omega(0)=0$, and let $\rho:[0,\infty)\to [0,\infty)$ be a rate function, by which we mean a (not necessarily strictly) monotonically increasing continuous function.
\begin{definition}[Computable Function]
\label{defn:ComputabilitySpace}
A function $f:\mathbb{R}^d\to \mathbb{R}^D$ is $\mathbb{G}$-computable if: exists a sequence of $\mathbb{G}$-circuits $(\mathcal{A}_q)_{q=1}^{\infty}$ with $\mathcal{A}_q$ of \textit{size} $\mathcal{O}(\rho(p))$ satisfying
\begin{equation}
        \max_{x\in [0,1]^d}\,
            \big\|
                \Rep(\mathcal{A}_q)\circ \pi_q^d(x)
                -
                f(x)
            \big\|
    \lesssim
        \omega(2^{-p})
.
\end{equation}
The set of all such functions $f:\mathbb{R}^d\to \mathbb{R}^D$ is denoted by $\mathbb{G}^{\omega}_{\rho}(\mathbb{R}^d,\mathbb{R}^D)$.
\end{definition}
This definition is motivated by the elementary inequality that describes the computability of a function in two parts.  1) The \textit{discretization error} from considering a problem at a fixed resolution $p$ and 2) the \textit{computation error} arising from the difficulty of algorithmically computing the discretized function via a $\mathbb{G}$-circuit.
\begin{lemma}[Discretization-Computation Decomposition]
\label{lem:Basic_Decomposition}
Let $f:\mathbb{R}^d\to \mathbb{R}^D$ be \textit{any function}, let $q\in \mathbb{N}_+$, and define
\begin{equation}
\label{eq:compute}
\bar{f}_q\eqdef \pi^D_q\circ f\circ \pi_q^d.
\end{equation}
For any $\mathbb{G}$-circuit $\mathcal{A}$ computing $\Rep(\mathcal{A}):\mathbb{R}^d_{q}\to \mathbb{R}^D_{q}$, and any $x\in [0,1]^d$
\begin{align*}
    \big\|
        f(x) - \Rep(\mathcal{A})\circ \pi_q^d(x)
    \big\|
\le 
    \underbrace{
        2^{-p+1}
        +
        \|
            f(x)
            - 
            f\circ \pi_q^d(x)
        \|_{\infty}
    }_{\text{Disc Err.:} \term{t:DiscError}}
    +
    \underbrace{
        \|
            \bar{f}_q(x)
            - 
            \Rep(\mathcal{A})\circ \pi_q^d(x)
        \|_{\infty}
    }_{\text{Compute Err.:} \term{t:ComputeError}}
\end{align*}
\end{lemma}

As with most standard notions of computability and approximation spaces, we find that any continuous function is computable.  More interestingly, we find that the only real quantity which we want to control is the computation error in~\eqref{t:ComputeError}.
\begin{corollary}[{Reasoning Implies Universal Approximation}]
\label{cor:Continuous_Computable}
If $f:[0,1]^d\to \mathbb{R}^D$ uniformly continuous with modulus of continuity $\omega$ 
then, $f\in \mathbb{G}_{2^{(d+D)p}}^{\omega}(\mathbb{R}^d,\mathbb{R}^D)$.
\hfill\\
In other words, for every $q\in \mathbb{N}_+$ there is a ReLU NN $\mathcal{A}_q$ with $\Rep(\mathcal{A}_q):\mathbb{R}^d_{q}\to \mathbb{R}^D_{q}$ of size $\mathcal{O}(2^{(d+D)p})$ satisfying
\[
    \sup_{x\in [0,1]^d}\,
    \big\|
        f(x) - \Rep(\mathcal{A})\circ \pi_q^d(x)
    \big\|
\le 
    \underbrace{
        2^{-p+1}
        +
        \omega(2^{-p+1})
    }_{\text{Disc Err.}}
.
\]
In particular,~\eqref{t:ComputeError} is zero.
\end{corollary}
Unlike classical uniform approximation theory for MLPs, we see that dyadic regression trees are computable at a finite resolution.

\section{Conclusion}
\label{s:Conclusion}

We present a systematic meta-algorithm that converts essentially any circuit into a feedforward neural network (NN) with ReLU activations by iteratively replacing each gate with a canonical ReLU MLP emulator (Theorem~\ref{thrm:Main}). This construction exactly emulates the original circuit on modern digital computers, with the network’s size scaling with the circuit’s complexity and its computational graph mirroring the structure of the emulated computation. In this way, neural networks trade algorithmic runtime for network size, formalizing a folklore principle of deep learning.

Our results imply universal approximation (Corollary~\ref{cor:Continuous_Computable}) while going further: they yield new guarantees ranging from randomized logic emulation to exact Turing-machine simulation, showing that no reasoning task lies beyond the reach of neural networks. We further obtained a superposition-type result (Theorem~\ref{thrm:WorstCaseUniversalGate}) demonstrating that any circuit computing a function from $\mathbb{R}_q^d$ to $\mathbb{R}_q^D$ can be implemented by a (possibly very large) neural network.

Lastly, we demonstrate that our result is \textit{strictly more powerful} than a classical \textit{universal approximation} theorem: any universal function approximator can be encoded as a circuit and directly emulated by a NN. 
By connecting circuit complexity, universal approximation, and black-box function access, our work opens a concrete pathway toward a theory of AI reasoning of neural networks.

\paragraph{Future Work}
In the near future, we would like to extend our results to handle non-Euclidean inputs and outputs, as in the geometric deep learning literature~\cite{bronstein2021geometric,kratsios2022universal} and infinite-dimensional inputs and outputs, as in the operator learning literature~\cite{chen1995universal,lu2021learning,li2023fourier}.  

One may interpret best approximations of Besov functions, studied in classical approximation theory~\cite{CABook2,wojtaszczyk1997mathematical} and compressive sensing~\cite{adcock2022sparse}, as ``converging'' sequences of circuits whose elementary operations are scaling, translation, and summation of objects such as wavelets~\cite{Daubechies}, polynomial splines~\cite{devore1993besov}, or sparse polynomials~\cite{adcock2018compressed}; respectively. An interesting future research direction could develop a general method for translating these classical constructions into the language of circuit complexity, thus directly translating them into the framework introduced in this paper. 

\section*{Acknowledgements}
A.\ Kratsios acknowledges financial support from an NSERC Discovery Grant No.\ RGPIN-2023-04482 and No.\ DGECR-2023-00230.  They also acknowledge that resources used in preparing this research were provided, in part, by the Province of Ontario, the Government of Canada through CIFAR, and companies sponsoring the Vector Institute\footnote{\href{https://vectorinstitute.ai/partnerships/current-partners/}{https://vectorinstitute.ai/partnerships/current-partners/}}.

A.\ Kratsios would like to thank \href{https://scholar.google.com/citations?hl=en&user=v-Jv3LoAAAAJ&view_op=list_works&sortby=pubdate}{Samuel Lanthaler} for the early work together on this manuscript and for the very enjoyable discussions.
The authors would also like to thank \href{https://scholar.google.com/citations?user=2aOpOJIAAAAJ&hl=en}{Giovanni Ballarin} for the helpful references.  They would also like to the members of the ``AI Reasoning reading group'', lead by \href{https://scholar.google.ca/citations?user=Xbef2UMAAAAJ&hl=en}{FuTe Wong}, at the Vector Institute for their valuable feedback and insightful discussions.

\appendix

\addcontentsline{toc}{section}{Appendix} 
\part{Appendix} 

Our paper's appendix is organized as follows.

\parttoc

\section{{\texorpdfstring{Proof of Theorem~\ref{thrm:Main}}{Proof of Main Result}}}
\label{s:Proof_Main_Result}

Having formalized the notion of a surgery, to prove Theorem~\ref{thrm:Main} it is enough to construct a local surgery for each gate in $\mathbb{G}$; as defined precisely in Section~\ref{s:ListOfGates}.

\subsection{{\texorpdfstring{$\mathbb{G}$-}{Some }Gates}}
\label{s:ListOfGates}
We now provide a fully-explicit description of the gates making up $\mathbb{G}$, referred to in Theorem~\ref{thrm:Main}.
\paragraph{Algebraic Operations}
The algebraic operations considered here are enough to express the standard $\mathbb{R}$-algebra structure on $\mathbb{R}^d_q$, to overflow handling by $\pi$.  These consist of the classes
$
\mathbb{G}_{\operatorname{alg}} \eqdef 
\{
+^n, -^n,\times^n, \cdot^{-1:n}, I_n:\, n\in \mathbb{N}_+
\}
$ where, for each $n\in \mathbb{N}_+$, we define
\begin{enumerate}
    \item \textbf{Addition:} $+^n:\mathbb{R}^n\times \mathbb{R}^n\to \mathbb{R}^n_q$, sends $(x,y)\mapsto \big(\pi(x_i+y_i)\big)_{i=1}^n$,
    \item \textbf{Additive Inverse:} $-^n:\mathbb{R}^n\to \mathbb{R}^n_q$, sends $x\mapsto \big(\pi(-x_i)\big)_{i=1}^n$,
    \item \textbf{Multiplication:} $\times^n:\mathbb{R}^n\times \mathbb{R}^n\to \mathbb{R}^n_q$, sends $(x,y)\mapsto \big(\pi(x_iy_i)\big)_{i=1}^n$,
    \item \textbf{Identity:} $\text{id}^n:\mathbb{R}^n\to \mathbb{R}^n_q$, sends $x\mapsto \big(\pi(x_i)\big)_{i=1}^n$.
\end{enumerate}
\paragraph{Analytic Operations}
In what follows, we use $\llbracket$ (resp.\ $\rrbracket$) to respectively define either $($ or $[$ (resp.\ $)$ or $]$).
We consider the elementary analytic operations
$
\mathbb{G}_{\operatorname{ann}} \eqdef 
\{
\exp^n, \log^n: \,n\in \mathbb{N}_+
\}
\cup 
\{
    I_{
        \llbracket
        a,c
        \rrbracket
    }
    :\, a,c\in \mathbb{R},\, a\le c
\}
\cup\{ \bar{z}^m:\, z\in \mathbb{R}^n_{b,q}:\,m,n\in \mathbb{N}_+\}
$ where for each $n\in \mathbb{N}_+$ and each pair of real numbers $a\le c$
\begin{enumerate}
    \item \textbf{Constants:} For each $z\in \mathbb{R}^{m}_q$, let $\bar{z}^m:\mathbb{R}^m\to \mathbb{R}^n_q$ sends $x\mapsto z$,
    \item \textbf{Indicator Functions:} for $a,c\in \mathbb{R}^m_q$ with $a_i\le c_i$ for $i=1,\dots,m$, $I_{
        \llbracket
        a,c
        \rrbracket
    }:\mathbb{R}^n\to \{0,1\}$ sends $x\in \mathbb{R}^n$ to $1$ if $x\in
        \llbracket
        a,c
        \rrbracket$ and $0$ otherwise
\end{enumerate}
\paragraph{Logical Operations}
We consider the following class of logical operations; here we follow the convention that 
\TRUE = $\mathbf{1}$ and \FALSE = $\mathbf{0}$.
We begin with the predicate ($0^{th}$ order logical operations)
$\mathbb{G}_{lgc}\eqdef \{
\vee^n, \wedge^n, \neg^m, \ge^n, \le^n, \equiv^n:\, n,m\in \mathbb{N}_+,\, n\ge 2
\}$
where for each $n,m\in \mathbb{N}_+, $ with $n\ge 2$, we have 
\begin{enumerate}
    \item \textbf{Conjunction:}
    $\wedge^n: \{0,1\}^n\times \{0,1\}^n \to \{0,1\}^n$ sends 
    $
    (x,y)\mapsto ({\min}\{a_i,b_i\})_{i=1}^n
    $,
    \item \textbf{Disjunction:} 
    $\vee^n: \{0,1\}^n\times \{0,1\}^n \to \{0,1\}^n$ sends 
    $
    (x,y)\mapsto ({\max}\{a_i,b_i\})_{i=1}^n
    $,
    \item \textbf{Negation:}
    $\neg^m: \{0,1\}^m \to \{0,1\}^m$ sends 
    $
    x\mapsto (1 - x_i \, \bmod(2))_{i=1}^n
    $
    \item \textbf{Ordering:}
    $\le^n: \{0,1\}^n\times \{0,1\}^n \to \{0,1\}^n$ sends 
    $
        (x,y)\mapsto (I_{x_i\le y_i})_{i=1}^n
    $,
    \item \textbf{Ordering (v2):}
    $\ge^n: \{0,1\}^n\times \{0,1\}^n \to \{0,1\}^n$ sends 
    $
        (x,y)\mapsto (I_{x_i\ge y_i})_{i=1}^n
    $,
    \item \textbf{Equality}
    $\equiv^n: \{0,1\}^n\times \{0,1\}^n \to \{0,1\}^n$ sends 
    $
        (x,y)\mapsto (I_{x_i = y_i})_{i=1}^n
    $.
\end{enumerate}

\paragraph{Tropical Operations}
Tropical gates are typical in several dynamic programming implementations; see e.g.~\cite{JunkaTropicalBook}.  We list some key tropical gates which are often considered; and are considered herein.
\begin{enumerate}
    \item \textbf{Minimum:} $\min:\mathbb{R}^n\to \mathbb{R}$ send $x\mapsto \min_{1 \leq j \leq n} x_j$,
    \item \textbf{Maximum:} $\max:\mathbb{R}^n\to \mathbb{R}$ send $x\mapsto \max_{1 \leq j \leq n} x_j$,
    \item \textbf{Median:} $\operatorname{median}:\mathbb{R}^n\to \mathbb{R}$ sends any $x\in \mathbb{R}^n$ to
    \begin{equation}
    \label{eq:Median_Function}
        \operatorname{median}(x_1,x_2,\cdots,x_n) 
    \eqdef 
        \begin{cases}
            x_{(n+1)/2} & \mbox{ if } n \mbox{ is odd} \\
            (x_{(n/2)}+x_{(m/2+1)})/2 & \mbox{ if } n \mbox{ is even} \\
        \end{cases}
    \end{equation}
    where $\{x_{(i)}\}_{i=1}^n=\{x_i\}_{i=1}^n$ and $x_{(1)}\le \dots \le x_{(n)}$.
    \item \textbf{Majority Vote:} $\operatorname{Maj}_n:\{0,1\}^n\to \{0,1\}$ maps each \textit{Boolean} $x\in \{0,1\}^d$ to
    \[
        \operatorname{Maj}_n(x)
        \eqdef 
        \begin{cases}
            1 & \mbox{if }\sum_{i=1}^n\,x_i\ge \frac{n}{2} 
        \\
            0 & \mbox{else}
    .
        \end{cases}
    \]
\end{enumerate}

\paragraph{Higher-Order Gates}
Note that several other logical gates, such as strict inequalities, can be expressed in terms of our elementary logical operations.  Several of these will be presented as examples and applications of our results below.

\subsection{Logical Gates}
\label{s:Operations__ss:Logic}

\subsubsection{Propositional (Zeroth-Order) Logic}
\label{s:Operations__ss:Logic___sss:0thOrderLogic}
We begin by showing that a simple ReLU MLP can implement the vectorized equality verification gate; i.e.\ the maps any pairs $a,b\in \{0,1\}^B$ to
\[
    \EQUAL(a,b) \eqdef \big(I_{a_i=b_i}\big)_{i=1}^B
.
\]
This will allow us to test truth values of statements.

\begin{lemma}[ReLU MLP Representation of $\EQUAL$]
\label{lem:EQUAL}
For any $B \in \mathbb{N}$, we have
$$
\EQUAL(a, b) = \phi_{\EQUAL}(a, b) = I_B(\ReLU(\mathbf{1}_B-\ReLU(\operatorname{ReLU}(A^-_{2, B}(a, b)^\top) + \operatorname{ReLU}(A^-_{2, B}\Pi_{2, B}(a, b)^\top))))
$$
Thus, there exists a bitwise ReLU neural network $\psi : \{0, 1\}^B \times \{0, 1\}^B \to \{0, 1\}^B$ of depth 3, width 2B, and 9B non-zero parameters, such that $\psi(a, b) = \EQUAL(a, b)$.
\end{lemma}
\begin{proof}
For $a, b \in \{0, 1\}^B$, we have that $a = b$ if and only if both $a \le b$ and $b \le a$. Equivalently, if $\operatorname{ReLU}(a-b) + \operatorname{ReLU}(b-a) = 0$. Define \begin{align*}
    \Phi_=(a, b) &= I_B(\ReLU(\mathbf{1}_B-\ReLU(\operatorname{ReLU}(a-b) + \operatorname{ReLU}(b-a))))\\
    &= I_B(\ReLU(\mathbf{1}_B-\ReLU(\operatorname{ReLU}(A^-_{2, B}(a, b)^\top) + \operatorname{ReLU}(A^-_{2, B}\Pi_{2, B}(a, b)^\top))))
\end{align*}
Thus, $a=b$ iff $\Phi_=(a,b) = \mathbf{1}_B$.
For each $a,b\in \{0,1\}^B$, we have that $a=b$ if and only if both: $a\le b$ and $b\le a$; or equivalently, $\min\{\operatorname{ReLU}(a-b),\operatorname{ReLU}(b-a)\}=0$.  Thus,\\
\begin{align}
\label{eq:verification_equality_detector}
        \min\{\operatorname{ReLU}(a-b),\operatorname{ReLU}(b-a)\} 
    & =
        \min\{\operatorname{ReLU}(A_{2,B}(b,a)^{\top}),\operatorname{ReLU}((A_{2,B}^{-}(a,b)^{\top})\} 
\end{align}
As shown, for example in~\cite[Lemma 5.11]{petersen2024mathematical}, the minimum function $\min\{\cdot,\cdot\}$ can be realized as the following depth $3$ MLP of width $B$ over all $x,y\in \mathbb{R}^B$
\begin{equation}
\label{eq:min}
        \min\{
            x,y
        \}
    =
        \operatorname{ReLU}(y)
        -
        \operatorname{ReLU}(-y)
        -
        \operatorname{ReLU}(y-x)
.
\end{equation}
Combining~\eqref{eq:verification_equality_detector} with~\eqref{eq:min} shows that the following equals to the zero vector
\begin{equation}
\label{eq:Not_Version_of_Construction}
\begin{aligned}
    &
            \mathbf{1}_B
            -
            \Big(
                \operatorname{ReLU}(
                    \operatorname{ReLU}(A_{2,B}(a,b)^{\top})
                )
                -
                \operatorname{ReLU}(
                    -
                    \operatorname{ReLU}(A_{2,B}(a,b)^{\top})
                )
    \\
    &
                -
                \operatorname{ReLU}(
                    \operatorname{ReLU}(A_{2,B}(a,b)^{\top})
                -
                    \operatorname{ReLU}((A_{2,B}^{-}\Pi_{2,B}(a,b)^{\top})
                )
            \Big)
\end{aligned}
\end{equation}
if and only if $a=b$.  Now, using Lemma~\ref{lem:not} we know that the $\NOT$ gate can be represented by $\NOT(a)=\operatorname{ReLU}(\mathbf{1}_B-a)$.  Post-composing the ReLU MLP in~\eqref{eq:Not_Version_of_Construction} by our $\NOT$ gate yields the conclusion.
\end{proof}

The $\NAND$ operation is functionally complete, meaning that every possible truth table can be expressed entirely in terms of $\NAND$s.  We therefore, represent $\NAND$ using a small ReLU MLP.
\begin{lemma}[ReLU MLP Representation of $\NAND$]
\label{lem:NAND}
For any $B\in \N$, we have
\[
        \NAND(a,b) 
    =
        \phi_{\NAND}(a,b)
    \eqdef 
        I_B
        \operatorname{ReLU}\big(
            \mathbf{1}_B
            - 
            I_B \operatorname{ReLU}(A_{2,B}(a,b)^{\top} - \mathbf{1}_B)
        \big)
\]
Thus, there exists a bitwise ReLU neural network $\psi : \{0, 1\}^{B} \times \{0, 1\}^B \to \{0, 1\}^B$ of depth $2$, width $B$, and $6B$ non-zero parameters, such that $\psi(a, b) = \NAND(a, b)$.
\end{lemma}
\begin{proof}
For any $a,b\in \{0,1\}^B$ 
by Lemma~\ref{lem:not}, $\NOT(a) = \ReLU(1-a)$ and by Lemma~\ref{lem:and} $\AND(a,b) = \ReLU(a+b-1)$. Together, we have that
\begin{align*}
            \NAND(a,b) 
    & =
        \operatorname{ReLU}\big(
            \mathbf{1}_B
            - 
            \operatorname{ReLU}(a+b - \mathbf{1}_B)
        \big)
\\
    & =
        \operatorname{ReLU}\big(
            \mathbf{1}_B
            - 
            \operatorname{ReLU}(A_{2,B}(a,b)^{\top} - \mathbf{1}_B)
        \big)
.
\end{align*}
\end{proof}

\paragraph{Examples of Other Propositional Logic Gates Implementable by ReLU MLPs}
Since $\NAND$ is functionally complete, every propositional logical gate can be canonically represented in terms of $\NAND$ gates.  Nevertheless, we provide examples of how many of the basic logical gates can be represented (possibly more efficiently) directly without relying on the $\NAND$ gate construction above; note that, using these will no longer yield a canonical construction but will yield smaller ReLU MLP encodings of algorithms.
\begin{lemma}[Neural network representation of \NOT]
\label{lem:not}
For any $B\in \N$, $a\in \{0,1\}^B$, we have
\[
\NOT(a) = I_B \ReLU(\mathbf{1}_B-I_Ba).
\]
Thus, there exists a bitwise ReLU neural network $\psi: \{0,1\}^B \to \{0,1\}^B$ of depth $1$, width width $B$, and $3B$ non-zero parameters, such that $\psi(a) = \NOT(a)$.
\end{lemma}
\begin{lemma}[Neural network representation of \AND]
\label{lem:and}
For any $B\in \N$, $a,b\in \{0,1\}^B$, we have
\[
\AND(a,b) = I_B \ReLU(A_{2,B}(a,b)-\mathbf{1}_B).
\]
Thus, there exists a bitwise ReLU neural network $\psi: \{0,1\}^B \times \{0,1\}^B \to \{0,1\}^B$ of depth $1$, width
$B$
, and $4B$ non-zero parameters, such that $\psi(a,b) = \AND(a,b)$.
\end{lemma}
\begin{lemma}[Neural network representation of \OR]
\label{lem:or}
For any $B\in \N$, $(a,b)\in \{0,1\}^{2B}$, we have
\[
\OR(a,b) 
=
I_B \big(\ReLU(
\mathbf{1}_B-\ReLU
(
    -A_{2,B}(a,b)^{\top} + \mathbf{1}_B
))
\big)
.
\]
Thus, there exists a bitwise ReLU neural network $\psi: \{0,1\}^B \times \{0,1\}^B \to \{0,1\}^B$ of depth $2$, width $B$, and $6B$ non-zero parameters, such that $\psi(a,b) = \OR(a,b)$.
\end{lemma}
\begin{proof}
For any $a,b\in \{0,1\}^B$, 
$
    \OR(a,b) 
= 
    1-\ReLU(1-a-b)
=
    \mathbf{1}_B-\ReLU
    \big(
        -A_{2,B}(a,b)^{\top} + \mathbf{1}_B
    \big)
.
$
\end{proof}
\begin{lemma}[Neural network representation of \XOR]
\label{lem:xor}
For any $B\in \N$, $(a,b)\in \{0,1\}^{2B}$, we have
\[
\XOR(a,b) 
= 
I_B\big( \ReLU(
\ReLU(A_{2,B}(a,b)^{\top}) - 2\ReLU(A_{2,B}(a,b)^{\top} - \mathbf{1}_B))
\big)
.
\]
Thus, there exists a bitwise ReLU neural network $\psi: \{0,1\}^B \times \{0,1\}^B \to \{0,1\}^B$ of 
depth $2$, 
width $2B$, 
and with $8B$ non-zero parameters
such that $\psi(a,b) = \XOR(a,b)$.
\end{lemma}
\begin{proof}
Note that, for each $a,b\in \{0,1\}^B$ we have
\begin{align*}
&\XOR(a,b)\\
&= \ReLU(\ReLU(a+b) - 2\ReLU(a + b - 1))\\
&= \ReLU(\ReLU(A_{2,B}(a,b)^{\top}) - 2\ReLU(A_{2,B}(a,b)^{\top} - \mathbf{1}_B)).  
\end{align*}
\end{proof}

\begin{lemma}[Neural network representation of \IMPLY]
\label{lem:imply}
For any $B\in \N$, $(a,b)\in \{0,1\}^{2B}$
\[
        \IMPLY(a,b) 
    =
        I_B \big(\operatorname{ReLU}(\mathbf{1}_B-A_{2,B}(a,b)) \big)
.
\]
Thus, there exists a bitwise ReLU neural network $\psi: \{0,1\}^B \times \{0,1\}^B \to \{0,1\}^B$ of depth $1$, width
{ $B$}
, and with $4B$ non-zero parameters such that $\psi(a,b) = \IMPLY(a,b)$.
\end{lemma}
\begin{proof}
The fact that $\IMPLY= \NOT\circ \OR$ and Lemma \ref{lem:not} and Lemma \ref{lem:or} imply that
\begin{align*}
        \IMPLY(a,b) 
    & 
    =
    \NOT\circ \OR(a,b)\\
    &= I_B\operatorname{ReLU}(\mathbf{1}_B-A_{2,B}(a,b)),
\end{align*}
since $\operatorname{ReLU}\circ \operatorname{ReLU}= \operatorname{ReLU}$.
\end{proof}

\subsection{Predicate (First-Order) Logic}
\label{s:Operations__ss:Logic___sss:1srtOrderLogic}

Since we know that our MLPs can implement the $\NOT$ gate, then the identity $\neg (\forall \neg P) = \exists P$ for any given proposition $P$ implies that we can obtain all our first-order logical operations simply by implementing $\forall$ quantifiers applied to computable propositions $P$.  By computable, we simply mean that the proposition has itself already been represented by a ReLU MLP.
The next proposition shows that this is indeed possible; however, as one may expect, the resulting ReLU MLP unit is rather large.

\begin{lemma}[Sentence Verification]
\label{lem:1rstOrderLogic_Verification}
Let $B_1,B_2\in \mathbb{N}_+$ and 
$\phi:\mathbb{R}^{B_1+B_2}\to \mathbb{R}$ a ReLU MLP of depth $D$, Width $W$, and with $S$ non-zero weights.  
Then there exists a ReLU MLP $\Phi_{\forall =}(\cdot|\phi):\mathbb{R}^{B_2}\to \mathbb{R}$ satisfying: for all $y\in \{0,1\}^{B_2}$
\[
        \big(
            \forall x\in \{0,1\}^{B_1} .\, \phi(x,y) = 1
        \big)
    \Leftrightarrow
        \Phi_{\forall =}(y|\phi) = 1
.
\]
Moreover, $\Phi_{\forall =}(\cdot|\phi)$ has  depth 
$B_1 + D + 1$
,
width
at-most $2\max\{2^{B_1}W, 2^{B_1}3\}$, and with at-most $2(2^{B_1}S + 2^{B_1+3} - 9)$
non-zero weights.
\end{lemma}
\begin{proof}
Let $x\in \{0,1\}^{B_1}$ and $y\in \{0,1\}^{B_2}$. We have that 
\begin{equation}
\label{eq:case_single_xy}
\forall x\in \{0,1\}^{B_1} . \, \phi(x,y) = 1
\Leftrightarrow
\min_{2^{B_1}}
    \underset{x \in \mSet{0, 1}^{B_1}}{\bigoplus} \ReLU(\phi(x, y)) = 1
\end{equation}

Note that the cardinality of $\{0,1\}^{B_1}$ is $2^{B_1}$. 
Applying~\cite[Lemma 5.11]{petersen2024mathematical}, we know that there is a ReLU MLP 
$\Phi_{2^{B_1}}^{\min}: \mathbb{R}^{2^{B_1}} \to \mathbb{R}$ with
\[
\text{size}(\Phi_{2^{B_1}}^{\min}) \leq 16\,2^{B_1}
, \quad 
\text{width}(\Phi_{2^{B_1}}^{\min}) \leq 2^{B_1} 3
,\mbox{ and }
\text{depth}(\Phi_{2^{B_1}}^{\min}) \leq \lceil \log_2(2^{B_1}) \rceil
=
B_1
\]
satisfying: for each $x_1,\dots,x_{2^{B_1}}\in \mathbb{R}$ we have
\begin{equation}
\label{eq:min_realizationRecall}
        \Phi_{2^{B_1}}^{\min}(x_1, \dots, x_{2^{B_1}})
    =
        \min_{1 \leq j \leq 2^{B_1}} \, x_j
.
\end{equation}
Combining~\eqref{eq:min_realizationRecall} with~\eqref{eq:case_single_xy} we see that the network: $\Phi:\mathbb{R}^{B_2}\to \mathbb{R}$ given by
\begin{align*}
    \Phi(y)
    \eqdef
    \Phi^{\min}_{2^{B_1}} \circ \Phi^{\text{id}}_1 \circ
    \left(
        \underset{x \in \mSet{0, 1}^{B_1}}{\bigoplus} \ReLU(\phi(x, y))
    \right),
\end{align*}
has a depth of
$B_1 + D + 1$,
width
at-most $2\max\{2^{B_1}W, 2^{B_1}3\}$, and with
at-most $2(2^{B_1}S + 2^{B_1+3} - 9)$
non-zero weights.  Further, for $y\in \{0,1\}^{B_2}$
\[
    \Phi(y)
    =
    \min_{x\in \{0,1\}^{B_1}}\,
    \operatorname{ReLU}(\phi(x,y)-1)
.
\]
Since, for any given $y\in \{0,1\}^{B_2}$, $\Phi(y|\phi)=1$ if and only if $\Phi(y)=1$ then we are done.
\end{proof}

We remark that higher order logic, where, for instance, statements can quantify over subsets of $\{0, 1\}^{B_1}$, in much the same way as Lemma~\ref{lem:1rstOrderLogic_Verification}.  All of this, of course, is contingent on the finiteness of the structures we are considering.

\subsection{Analytic Gates}
\label{s:Ann_Gates}

\begin{lemma}[Constant Functions]
\label{lem:constant}
Let $d\in \mathbb{N}_+$ and $c\in \mathbb{R}_q^d$. Define the ReLU MLP $\Phi_c:\mathbb{R}^d\to \{c\}$ for each $x \in \mathbb{R}^d$ by
\[
\Phi_c(x)\eqdef (-I_n,I_n)\operatorname{ReLU}(\mathbf{0}_{2n\times n} x + (-c,c)^{\top}),
\]
where $\Phi_c(x)=c$ for all $x\in \mathbb{R}^d$. Furthermore, 
$\Phi_c(x)$ has depth 1, width $2d$, and $4d$ non-zero parameters.
\end{lemma}
\begin{proof}
By construction.
\end{proof}

\begin{lemma}[ReLU MLP Implementation of Indicator of $[a,\infty)$]
\label{lem:Indictor__PositiveLine}
Fix $q\in \mathbb{N}_+$, $a\in \mathbb{R}_q$, and define the ReLU MLP $\Phi_{[a,\infty):q}:\mathbb{R}\to [0,1]$ for each $x\in \mathbb{R}$ by
\[
        \Phi_{[a,\infty):q}(x)
    = 
            \operatorname{ReLU}\big(1 -\operatorname{ReLU}\big(-2^{q+1}(x-a)\big)\big)
\]
of depth $2$, width $1$, with
$5$
non-zero parameters.  Then $\Phi_{[0,\infty):q}\sim_q I_{[a,\infty)}$.
\end{lemma}
\begin{proof}[{Proof of Lemma~\ref{lem:Indictor__PositiveLine}}]
If $x\in [a,\infty)$ then $\Phi_{[0,\infty):q}(x)=1$ and if $x\in (-\infty,a-2^{-(q+1)}]$ then, $\Phi_{[0,\infty):q}(x)=0$.
\end{proof}
By a similar token, we may construct the indicator function of the \textit{open} complementary interval.
\begin{lemma}[ReLU MLP Implementation of Indicator of $(-\infty, a)$]
\label{lem:Indictor__OpenHalfLine}
Fix $q\in \mathbb{N}_+$, $a\in \mathbb{R}_q$, and define the ReLU MLP $\Phi_{(-\infty,a):q}:\mathbb{R}\to [0,1]$ for each $x\in \mathbb{R}$ by
\[
        \Phi_{(-\infty,a):q}(x)
    = 
        \operatorname{ReLU}\big(1 -\operatorname{ReLU}\big(2^{q+1}(x - a) + 1\big)\big)
\]
of depth $2$, width $1$, with $5$ non-zero parameters.  Then $\Phi_{[0,\infty):q}\sim_q I_{(-\infty,a)}$.
\end{lemma}
\begin{proof}[{Proof of Lemma~\ref{lem:Indictor__OpenHalfLine}}]
If $x\in (-\infty,a-{2^{-(q+1)}}]$ then $\Phi_{(-\infty,a):q}(x)=1$ and if $x\in [a,\infty)$ then $\Phi_{(-\infty,a):q}(x)=0$.
\end{proof}

\begin{lemma}[{ReLU MLP Implementation of Indicator Functions of $[a,b]$}]
\label{lem:Closed_Interval}
Let $a,b\in \mathbb{R}_q$ with $a\le b$.  
For every $q\in \mathbb{N}_+$ consider the ReLU MLP $\Phi_{[a,b]:q}:\mathbb{R}\to \mathbb{R}$ with depth $2$, width $2$, and
$9$ 
non-zero parameters defined for each $x\in \mathbb{R}$ by
\[
        \Phi_{[a,b]:q}(x)
    \eqdef 
        (1,-1)
        \operatorname{ReLU}\big(
            -
                \operatorname{ReLU}\big(
                    (-2^{q+1},-2^{q+1})^{\top}x + (2^{q+1}a,2^{q+1}b + 2^{-q-1})^{\top}
                \big)
        +(1,1)^{\top}\big)
.
\]
Then, $\Phi_{[a,b]:q}(x)\sim_q I_{[a,b]}$.
\end{lemma}
\begin{proof}[{Proof of Lemma~\ref{lem:Closed_Interval}}]
By construction, we have $\Phi_{[a,b]:q}(x) = \ReLU\big(\Phi_{[a,\infty):q}(x) - \Phi_{[b+{2^{-(q+1)}}, \infty):q}(x)\big)$. Now, 
Lemmas~\ref{lem:Indictor__PositiveLine} and~\ref{lem:Indictor__OpenHalfLine} imply that $
I_{[a,\infty)}-I_{[b+{2^{-(q+1)}},\infty)}
\sim_q \Phi_{[a,\infty):q}-\Phi_{[b+{2^{-(q+1)}, \infty)}:q}$.  However, $I_{[b+{2^{-(q+1)}},\infty)}\sim_q I_{(b,\infty)}$, and so, 
\[
    I_{[a,b]}
\sim_q 
    I_{[a,\infty)}- I_{(b,\infty)}
\sim_q 
    I_{[a,\infty)}-I_{[b+{2^{-(q+1)}},\infty)} 
  =\Phi_{[a,b]:q}(x)
\]
concluding our proof.
\end{proof}
A direct consequence of Lemma~\ref{lem:Closed_Interval} and the $\sim_q$ relation is the following indicator function.
\begin{corollary}[Indicator of Half-Open Intervals]
\label{cor:HalfOpenInterval}
Let $a,b\in \mathbb{R}_q$ with {$a\le b$.}  
For every $q\in \mathbb{N}_+$ consider the ReLU MLP $\Phi_{[a,b):q}:\mathbb{R}\to \mathbb{R}$ with depth {$2$}, width {$2$}, and 
$9$ 
non-zero parameters defined for each $x\in \mathbb{R}$ by
\[
        \Phi_{[a,b):q}(x)
    \eqdef 
        (1,-1)
        \operatorname{ReLU}\big(
            -
                \operatorname{ReLU}\big(
                    (-2^{q+1},-2^{q+1
                    })^{\top}x + (2^{q+1}a,
                    2^{q+1}(b-2^{-(q+1)}) + 2^{-(q+1)})
                    ^{\top}
                \big)
        +(1,1)^{\top}\big)
.
\]
Then, $\Phi_{[a,b):q}\sim_q I_{[a,b)}$.
\end{corollary}
\begin{proof}[{Proof of Corollary~\ref{cor:HalfOpenInterval}}]
Since
$I_{[a,b)}\sim_q I_{[a,b-{2^{-(q+1)}}]}$
then the result follows from Lemma~\ref{lem:Closed_Interval} upon setting $
\Phi_{[a,b):q}\eqdef \Phi_{[a,b+\tfrac{1}{2^{q+1}}]:q}$.
\end{proof}
We also obtain the following useful construction directly.  
\begin{corollary}[Indicator of Complements to Half-Open Intervals]
\label{cor:HalfOpenInterval__complement}
Let $a,b\in \mathbb{R}_q$ with {$a\le b$.}  
For every $q\in \mathbb{N}_+$ consider the ReLU MLP $\Phi_{[a,b)^c:q}:\mathbb{R}\to \mathbb{R}$ with depth {$2$}, width {$2$}, and $10$ non-zero parameters defined for each $x\in \mathbb{R}$ by
\begin{align*}
    1 -
        (1,-1)
            \operatorname{ReLU}\big(
                -
                \operatorname{ReLU}\big(
                        (-2^{q+1},-2^{q+1
                        })^{\top}x + (2^{q+1}a,
                        2^{q+1}(b-2^{-(q+1)}) + 2^{-(q+1)})
                        ^{\top}
                    \big)
            +(1,1)^{\top}\big)
\end{align*}
Then, $\Phi_{[a,b)^c:q}\sim_q I_{\mathbb{R}\setminus [a,b)}$.
\end{corollary}
\begin{proof}[{Proof of Corollary~\ref{cor:HalfOpenInterval__complement}}]
By construction $\Phi_{[a,b)^c:q}= 1-\Phi_{[a,b):q}$; and so the result follows from Corollary~\ref{cor:HalfOpenInterval}.
\end{proof}

\begin{lemma}[Indicator of a Single Point in $\mathbb{R}^d$]
\label{lem:neural_spike}
For each $d,q\in \mathbb{N}_+$, there is a ReLU MLP 
$\Phi_{a,q}:\mathbb{R}^d\to \mathbb{R}^d$ with width at-most $4d$, depth $4$, and $18d+5$ non-zero parameters satisfying
\[
        \Phi_{a,q}(x) 
    = 
        I_{x=a}
\]
for each $x\in \mathbb{R}^d_q$.
\end{lemma}
\begin{proof}
For each $a=\frac{(a_1,\dots,a_d)}{2^q}\in \mathbb{R}_q^d$, and each $i=1,\dots,d$, let
Consider the real-valued ReLU MLP, mapping any $x\in \mathbb{R}$ to
\[
\Phi_{a_i,q}\left(x\right)
\eqdef 
1\,
\operatorname{ReLU}\Big(
(-2^{q+1},-2^{q+1})^{\top}\, \operatorname{ReLU}\big(
(1,-1)x +\big(-2^{q}a,2^{q}a\big)
\big)
+1
\Big)
.
\]
Then, for each $x\in \mathbb{R}_q$, we have
\[
    \Phi_{a_i,q}(x)
    =
    \begin{cases}
        1 & \mbox{ if } x=\frac{a}{2^q}
        \\
        0 & \mbox{ else}
.
    \end{cases}
\]
Furthermore, $\Phi_{a_i,q}$ has depth $2$, width $2$, and $8$ non-zero parameters.  
Now, applying the parallelization Lemma in~\cite[Lemma 5.3]{petersen2024mathematical}, there is a ReLU MLP 
$\tilde{\Phi}_{a,q}:\mathbb{R}^d_q\to \mathbb{R}^d$ with width at-most $4d$, depth $2$, and $16d+
2d=
18d
$ non-zero parameters satisfying
\[
    \tilde{\Phi}_{a,q}(x) = 
    \mathbf{1}_d^{\top}
    \big(
        \Phi_{a_1,q}(e_1^{\top}x)
        ,
        \dots,
        \Phi_{a_d,q}(e_d^{\top} x)
    \big)
    =
    \sum_{i=1}^d\, \Phi_{a_d,q}(x_i)
\]
for all $x\in \mathbb{R}^d$. 

Though one can proceed from this point by using the multiplication lemma, namely Lemma~\ref{lem:mult}, the following construction is significantly more efficient: Note that, for each $x\in \mathbb{R}^d_q$, $\Phi_{a,q}\in [0,d]$ with value $d$ being achieved if and only if $x_1=\frac{a_1}{2^q},\dots,x_d=\frac{a_d}{2^q}$.  
Next, observe that the ``$d$-threshold'' ReLU MLP: $\operatorname{\Phi}_{\operatorname{Th}:d,q}:\mathbb{R}\to \mathbb{R}$ given for each $x\in \mathbb{R}$ by
\[
\operatorname{\Phi}_{\operatorname{Th}:d,q}(x)
\eqdef 
1
\operatorname{ReLU}\left(-
2^{q+1}
\operatorname{ReLU}\left(-x+d\right)+1\right)
\]
satisfies the following for each $x\in \mathbb{R}_q$
\[
\operatorname{\Phi}_{\operatorname{Th}:d,q}(x)
=
\begin{cases}
    1  & \mbox{ if } x\ge d
    \\
    0 & \mbox{ else}.
\end{cases}
\]
Moreover, $\operatorname{\Phi}_{\operatorname{Th}:d,q}$ has depth $2$, width $1$, and $5$ non-zero parameters.  Thus, the desired map $\Phi_{a,q}:\mathbb{R}^d\to \mathbb{R}$ is given by $
\Phi_{a,q}
\eqdef 
\operatorname{\Phi}_{\operatorname{Th}:d,q}\circ \tilde{\Phi}_{a,q}
$.
\end{proof}

\paragraph{Bit Extractor}
\hfill\\
We now show how a ReLU MLP can extract a binary representation of any number in $\mathbb{R}_q^d$.
\begin{lemma}[{ReLU MLP Implementation of the Integer Floor Representation}]
\label{lem:ceiling}
Let 
$\lfloor \cdot \rfloor:\mathbb{R}\to \mathbb{Z}$
denote the integer ceiling function. 
For any $M\in \mathbb{N}_+$ there exists a ReLU MLP 
$\Phi_{\lfloor \cdot \rfloor:M}:\mathbb{R}\to \mathbb{Z}\cap [-M,M]$ 
with depth at-most $\lceil \log_2(2M+1)\rceil + 3$, width $16M + 8$, and $108M + 54$ non-zero parameters.
such that 
\[
    \Phi_{\lfloor \cdot \rfloor:M}
\eqdef
    {
        \Phi_{2M+1}^{\min}
        \circ 
        \Phi_q^{\star}
    }
\sim_q
    \lfloor \cdot \rfloor
\]
{
where $\Phi_{2M+1}^{\min}=\min_{i=-M,\dots,M}x_i$ is constructed in~\cite[Lemma 5.11]{petersen2024mathematical} and 
\begin{align*}
    \Phi^{\star}_q(x)
    & \eqdef
    \bigoplus_{n=-M}^M\,
        \big(
        M+\tfrac{1}{2^{q+1}}
        \big)
        \,
        \Phi_{[0,1)^c:q}(n-x)
        + 
            n 
            \Phi_{[0,1):q}(n-x)
,
\end{align*}}
and $\Phi_{[0, 1]:q}$ and $\Phi_{[0, 1]^{c}:q}$ are constructed in Corollaries \ref{cor:HalfOpenInterval} and \ref{cor:HalfOpenInterval__complement}, respectively.
\end{lemma}
\begin{proof}[{Proof of Lemma~\ref{lem:ceiling}}]
First, observe that: for each $x\in [-M,M]$ we have
\begin{equation}
\label{eq:search_representation_celing}
        \lfloor x \rfloor
    =
        \min\big\{
            n\,I_{n-x\in {[0,1)}}
            {
            +(M+2^{-(q+1)})I_{n-x \not\in {[0,1)}}
            }
            :
            n\in [-M,M]\cap \mathbb{Z}
        \big\}
.
\end{equation}
By {Corollary~\ref{cor:HalfOpenInterval}, there is a ReLU MLP $\Phi_{[0,1):q}:\mathbb{R}\to \mathbb{R}$} satisfying $\Phi_{[0,1):q}\sim_q I_{[0,1)}$ with {depth $2$, width $2$, and $9$ non-zero parameters.}
{Similarly, by Corollary~\ref{cor:HalfOpenInterval__complement} there is a ReLU MLP $\Phi_{[0,1)^c:q}:\mathbb{R}\to \mathbb{R}$ given by $\Phi_{[0,1)^c:q}\eqdef 1-\Phi_{[0,1):q}$ satisfying
$\Phi_{[0,1)^c:q}\sim_q I_{\mathbb{R}\setminus [0,1)}$ which has with depth $2$, width $2$, and $10$ non-zero parameters. Stacking those networks and rescaling, we find that the function
\begin{equation}
\label{eq:stacked_network}
    (M+\tfrac{1}{2^{q+1}})I_{(n-x)\in \mathbb{R}\setminus [0,1)}+nI_{(n-x)\in [0,1)}
\sim_q
    (M+\tfrac{1}{2^{q+1}})\Phi_{[0,1)^c:q}(n-x)
+ 
    n 
    \Phi_{[0,1):q}(n-x)
=
    \Phi_{M,n:q}
\end{equation}
Consequently, we may parallelize the above construction to obtain the $2$ layer ReLU MLP $\Phi^{\star}_q:\mathbb{R}\to \mathbb{R}$ sending any $x\in \mathbb{R}$ to
\begin{align*}
    \Phi^{\star}_q(x)
\eqdef 
    \bigoplus_{n=-M}^M\,
        \Phi_{M,n:q}(x)    
.
\end{align*}
Moreover, $\Phi_q^{\star}$ has depth
$3$,
width 
$16M + 8$,
and 
$76M + 38$
non-zero parameters.
}
Appealing to~\cite[Lemma 5.11]{petersen2024mathematical}, there exists a ReLU MLP $\Phi^{\operatorname{min}}_{2M+1}:\mathbb{R}^{2M+1}\to \mathbb{R}$ implementing 
\begin{equation}
\label{eq:minimum_2M1}
    \Phi^{\operatorname{min}}_{2M+1}(x_1,\dots,x_{2M+1}) = \min_{m=1,\dots,2M+1}\, x_m
\end{equation} 
for all $x\in \mathbb{R}^{2M+1}$ where $\Phi^{\operatorname{min}}_{2M+1}$ has depth at-most $\lceil \log_2(2M+1)\rceil$, width at-most $6M+3$, and at-most $32M + 16$ non-zero parameters.
Define 
\[
    \Phi_{\lfloor \cdot \rfloor:M}
    \eqdef
        \Phi_{2M+1}^{\min}
            \circ 
            \Phi_q^{\star}
        .
\]
Consequentially,~{\eqref{eq:stacked_network}} and~\eqref{eq:minimum_2M1} imply that~\eqref{eq:search_representation_celing} can be represented as
\[
    \Phi_{\lfloor \cdot \rfloor:M} \sim \lfloor \cdot \rfloor
.
\]
By construction
$\Phi_{\lfloor \cdot \rfloor:M}$ which has depth at-most $\lceil \log_2(2M+1)\rceil + 3$, width 
$16M + 8$, and 
$108M + 54$
non-zero parameters.
\end{proof}

\subsection{Bit Encoder and Decoder Networks}
\label{s:BitManipulators}
In this section, we explicitly construct a mapping between our dictionary of elementary operations $\mathbb{G}^0$ to a fixed set of elementary MLPs.  It will often be convenient to move to, and from, binary representations of numbers in $\mathbb{R}_q$ of the form
\begin{equation}
\label{eq:BinaryFloatingForm}
    x 
=
    \underbrace{x_0}_{\text{sign}} \quad \underbrace{x_1 \dots x_B}_{\text{significant digits}} \quad \underbrace{x_{B+1} \dots x_{B+e}}_{\text{exponent}}
.
\end{equation}
For example, for double precision (64bits), one has 53 bits for significand, 11 bits for exponent, and 1 sign bit.  We now formalize a ReLU MLP \textit{encoder} sending any $x\in \mathbb{R}_q$ to the binary floating-point type representation on the right-hand side of~\eqref{eq:BinaryFloatingForm} and a ReLU MLP \textit{decoder} which reverses this procedure.  
{
The binary representation in~\eqref{eq:BinaryFloatingForm} implies two maps.  The first encodes every real number in $\mathbb{R}_q$ into its binary expansion; that is, given any $q,M\in \mathbb{N}_+$, the \textit{bit encoder} map 
\begin{equation}
\label{eq:Bit_Encoder}
\begin{aligned}
    \operatorname{Bin}_{M:q}: \mathbb{R}_q & \rightarrow \{0,1\}^{2q+2}
\\
    \operatorname{Bin}_{M:q}(x) & \mapsto 
    \underbrace{
        (\beta_i)_{i=0}^q
    }_{\text{Integer Part}}
    \oplus 
    \underbrace{
        (\beta_i)_{i=-q}^{-1}
    }_{\text{Fractional Part}}
    \oplus 
    \underbrace{
        I_{x\ge 0}
    }_{\text{Sign}}
\end{aligned}
\end{equation}
where $x = \sum_{i=-q}^q \beta_i 2^i$ is the unique bit representation, and sign, of any $x\in \mathbb{R}_q$.  
Next, the simpler \textit{bit decoder} inverts the bit encoder $\operatorname{Bin}_{M:q}$ by sending any such binary representation $\{0,1\}^{2q+2}$ back to the real number in $\mathbb{R}_q$ it encodes.  Given any $M,q\in \mathbb{N}_+$, we define the \textit{bit decoder} map
\begin{equation}
\label{eq:Bit_Decoder}
\begin{aligned}
    \operatorname{Bit}_{M:q}: \{0,1\}^{2q+2} & \rightarrow \mathbb{R}_q
\\
    \operatorname{Bit}_{M:q}((\beta_i)_{i=0}^{2q+2}) & \mapsto 
        (-1)^{q2+2}
        \biggl(
                \sum_{i=0}^q \beta_i 2^i 
            +
                \sum_{i=-q}^{-1}\beta_{q+1+i} 2^{-i}
        \biggr)
.
\end{aligned}
\end{equation}
We now construct networks emulating both of these maps.
}

\subsubsection{Bit Decoder}
\label{s:BitManipulators__ss:Decoder}
\begin{proposition}
\label{prop:bit_decoder}
Fix $M,q\in \mathbb{N}_+$.  Then, there is a ReLU MLP $\Phi_{Bit,M:q}:\mathbb{R}^{2q+2}\to \mathbb{R}$ of depth $1$, width $2q+2$, and with at-most $2q+10$ non-zero parameters such that: for every $\beta\in \{0,1\}^{2q+2}$ we have $\operatorname{Bit}_{M,q}(\beta)=\Phi_{Bit,M:q}(\beta)$.
\end{proposition}
\begin{proof}[{Proof of Proposition~\ref{prop:bit_decoder}}]
Consider the linear map $g(u)\eqdef -2u+1$ on $\mathbb{R}$; note that $g(0)=1$ and $g(1)=-1$.
Now, consider the $2\times (2q+2)$ matrix
\[
        \operatorname{D}_1
    \eqdef
        \begin{pmatrix}
        (2^i)_{i=0}^q \oplus (2^{-i})_{i=-1}^{-q} \oplus \mathbf{0}_1
        \\
        \mathbf{0}_{2q+1} \oplus (-2)
        \end{pmatrix}
\]
and the bias vector $\operatorname{d}_1\eqdef \mathbf{0}_{2q+1}\oplus (1)$; the total number of non-zero entries in $\operatorname{D}_1$ and $\operatorname{d}_1$ is $2q+2$.  
Now, for any $\beta\eqdef (\beta_i)_{i=1}^{2q+2}$ with $x\eqdef \operatorname{Bit}_{M:q}(\beta)$ we have
\begin{equation}
\label{eq:bit_identity}
    \operatorname{D}_1\beta + \operatorname{d}_1
=
    (y,z)^{\top}
\eqdef
    \big(
        |x|
    ,
        \operatorname{sgn}(x)
    \big)^{\top}
\end{equation}
where, we recall that $\operatorname{sgn}(x)=1$ if $x>0$, $\operatorname{sgn}(x)=-1$ if $x<0$, and $\operatorname{sgn}(0)=0$.

Now consider the $2\times 2$ matrix $\operatorname{D}_2$ and the vector $\operatorname{d}_2\in \mathbb{R}^2$ given by
\[
        \operatorname{D}_2
    \eqdef 
        \begin{pmatrix}
            1 & M+\tfrac{1}{2^{q+1}}\\
            -1 & -M-\tfrac{1}{2^{q+1}}
        \end{pmatrix}
\mbox{ and }
        \operatorname{d}_2
    \eqdef 
        \big(
            -M+\tfrac{1}{2^{q+1}}
            ,
            M+\tfrac{1}{2^{q+1}}
        \big)^{\top}
.
\]
We readily verify that the ReLU MLP $\operatorname{{\Phi}}_{Bit,M:q}:\mathbb{R}^{2q+2}\to \mathbb{R}$ defined for any $x\in \mathbb{R}^{2q+2}$ by
\begin{equation}
\operatorname{{\Phi}}_{Bit,M:q}(\beta)
(1,1)\operatorname{ReLU}\Big(
            \big(
                \operatorname{D}_2\operatorname{D_1}x
                +
                    \operatorname{d_1}
            \big)
        +
            \operatorname{d}_2
    \Big)
\end{equation}
is such that: if $\beta \in \{0,1\}^{2q+2}$ and $x=\operatorname{Bit}_{M,q}(\beta)$ then
\begin{align*}
\Phi_{Bit,M:q}(\beta)
& =
(1,1)
\operatorname{ReLU}
\biggl(
    \operatorname{D}_2
        \big(
            |x|
        ,
            \operatorname{sgn}(x)
        \big)^{\top}
    +
    \operatorname{d}_2
\biggr)
\\
&=
\begin{cases}
\biggl(
    \sum_{i=0}^q \beta_{i+q+1}
    \,2^i 
+
    \sum_{i=-q}^{-1}\beta_{i+q+1} 2^{-i}
\biggr)
: \mbox{ if } x\ge 0\\
\biggl(
    \sum_{i=0}^q \beta_{i+q+1}
    \,
    2^i
+
    \sum_{i=-q}^{-1}\beta_{i+q+1} 2^{-i}
\biggr)
: \mbox{ if } x<0
\end{cases}
\\
&=
\operatorname{Bit}_{M,q}(\beta)
=
x.
\end{align*}    
Moreover, $\Phi_{Bit,M:q}$ has depth $1$, width $2q+2$, and at-most $2q+10$ non-zero parameters.
\end{proof}

\subsubsection{Bit Encoder}
\label{s:BitManipulators__ss:Encoder}
{
We begin by emulating the bit encoder map in~\eqref{eq:Bit_Encoder}, by a feedforward ReLU neural network, the construction of which is undertaken via the next few lemmata.   
We note that more complicated bit-extraction networks have previously been constructed in the learning theory literature (see, e.g., \cite{bartlett1998almost,bartlett2019nearly}). Here, we prefer a simple and direct construction that could, in principle, be easily implemented, and we extract the sign information as well.
}
We first construct a ReLU MLP, which (uniquely) decomposes any number in $\mathbb{R}_q$ into: its positive integer part, its remainder in $(-1,1)$, and identifies its sign.  Once we have computed these quantities, it only remains to apply a long division algorithm on the remainder and integer parts to obtain the binary representation of any such number.
\begin{lemma}[ReLU MLP Implementation: Sign-Integer-Remainder Decomposition]
\label{lem:encoder_positive}
Fix $q\in \mathbb{N}_+$.
There exists a ReLU MLP $\Phi_{\operatorname{Bit}:M:+}:\mathbb{R}\to \mathbb{R}^3$ of 
depth at-most 
$\lceil \log_2(2M+1)\rceil+4$, 
width at-most 
$32M + 22$, 
and at-most 
$216M + 134$ 
non-zero parameters
such that: for each $x\in \mathbb{R}_q$
\[
        \Phi_{\operatorname{Bit}:M:+}(x) 
    = 
        \Big(
                \underbrace{\lfloor |x|\rfloor}_{n: \text{integer part}}
            ,
                \underbrace{|x|-\lfloor |x|\rfloor}_{r: \text{remainder}}
            ,
                \underbrace{
                    I_{x\ge 0}
                }_{s: \text{sign}}
        \Big)^{\top}
\]
where $x = s \, (n+r)$.
\end{lemma}
\begin{proof}[{Proof of Lemma~\ref{lem:encoder_positive}}]
The absolute value function $\mathbb{R}\ni x\mapsto |x|\in [0,\infty)$ can be implemented by the following ReLU MLP $\Phi_{|\cdot|}:\mathbb{R}\to [0,\infty)$ of depth $1$, width $2$, with $4$ non-zero parameters; defined for each $x\in \mathbb{R}$ by
\[
        \Phi_{|\cdot|}(x)
    =
        \operatorname{ReLU}(x)+\operatorname{ReLU}(-x)
    =
        (1,1)
        \,
        \operatorname{ReLU}\Big(
            (1,-1)^{\top}
            x
        \Big)
.
\]
Now, suppose that $x\in (0,\infty)$. First note that $x$ can be uniquely written as
$x = n + r$
with $n\in \mathbb{N}_0$ and $r\in [0,1)$ a remainder.  
Moreover, 
$n = \lfloor x \rfloor$
and
$r = x - \lfloor x \rfloor$.
Let 
$\Phi_{\lfloor \cdot \rfloor,M}$ 
denote the ReLU MLP of Lemma~\ref{lem:ceiling} and let $\Phi_{[0,\infty):q}$ denote the ReLU MLP of Lemma~\ref{lem:Indictor__PositiveLine}.  
Consider the map $\Phi_{\operatorname{Bit}:M:+}:\mathbb{R}\to \mathbb{R}^2$ given for each $x\in \mathbb{R}_q$ with $x=n+r$, $n\in \mathbb{N}_0$, and $r\in [0,1)$, by
\[
        \Phi_{\operatorname{Bit}:M:+}(x)
    \eqdef 
        \begin{pmatrix}
            0 & 1 & 0\\
            1 & -1 & 0\\
            0 & 0 & 1
        \end{pmatrix}
        \,
        \begin{pmatrix}
            \Phi_{|\cdot|}  \\
            \Phi_{\lfloor \cdot \rfloor:M}\circ \Phi_{|\cdot|}(x) \\
            \Phi_{[0,\infty):q}(x)
        \end{pmatrix}
    =
        \begin{pmatrix}
            \lfloor |x| \rfloor  \\
            |x| - \lfloor |x|\rfloor \\
            I_{x<-2^{-q}}
        \end{pmatrix}
    =
        \begin{pmatrix}
            n \\
            r \\
            I_{x\in [-1/2^{q},\infty)}
        \end{pmatrix}
.
\]
The parallelization lemma in~\cite[Lemma 5.3]{petersen2024mathematical} implies that
$\Phi_{\operatorname{Bit}:M:+}$ can be represented as a ReLU MLP of 
depth at-most 
$\lceil \log_2(2M+1)\rceil+4$
width at-most 
$32M+22$
and at-most 
$216M + 134$
non-zero parameters.  
\end{proof}

Using Lemma~\ref{lem:encoder_positive}, we only need to express two bit-extraction subroutines: 1) which obtains the binary representation of any natural number and 2) which obtains the binary representation of any ``remainder'' in $[0,1)$. 
\begin{lemma}[Bit Encoder {- Remainder Only}]
\label{lem:fraction_encoder__remainder}
Fix $q\in \mathbb{N}_+$.  Consider ReLU MLP $\Phi_{\operatorname{Rem}:q}:\mathbb{R}\to \mathbb{R}^q$ of 
depth of $\lceil \log_2(2M+1)\rceil + 3$, a width at-most $2q(16M+8)$, and no more than 
$216 M q + 111 q + 1$ and {defined by
\[
        \Phi_{\operatorname{Rem}:q}(x)
    \eqdef 
        (I_q - 2L)
        \Phi_{\lfloor \cdot \rfloor:M}^{\| q}
        \big(
            \operatorname{diag}((2^i)_{i=1}^q) 
            x
        \big)
\]
where 
$\Phi_{\lceil\cdot \rceil:M}^{\| q}$ is the $q$-fold parallelization of the ceiling network $\Phi_{\lfloor \cdot\rfloor:M}$ in 
Lemma~\ref{lem:ceiling}.
Then, $\Phi_{\operatorname{Rem}:q}$ is }
such that: for each $r\in (0,1)\cap \mathbb{R}_q$ we have
\[
        (\beta_1,\dots,\beta_q)\eqdef \Phi_{\operatorname{Rem}:q}(r) 
    \mbox{ and }
      r=\sum_{m=1}^q\, \beta_m\,2^{-q}
.
\]
\end{lemma}
\begin{proof}
Let $r \in [0,1)$; then, $r= \sum_{i=1}^{\infty}\, \beta_i\, 2^{-i}$ for some $(\beta_i)_{i=1}^{\infty}$ in $[0,\infty)$.  Set $\gamma_0\eqdef 0$ and for each $i\in \mathbb{N}_+$ define 
$
    \gamma_i\eqdef \lfloor 2^i r \rfloor
$.
Then, for each $i\in \mathbb{N}_+$, we may express each $\beta_i$ by
\begin{equation}
\label{eq:fractional__bit_encoding}
    \beta_i = \gamma_i-2\gamma_{i-1}
\end{equation}
Moreover, $\beta_i\in \{0,1\}$ for each $i\in \mathbb{N}_+$.

Let $\mathbf{A}\eqdef (2^i\delta_{i,j})_{i,j=1}^q
\eqdef \operatorname{diag}((2^i)_{i=1}^q)
$ where $\delta_{i,j}$ is the Kronecker delta and let $\mathbf{1}_q\in \mathbb{R}^q$ denote the vector with all entries equal to $1$.  Consider the map $\Phi_{\operatorname{rem}:q}:\mathbb{R}\to \mathbb{R}^q$ given for each $x \in \mathbb{R}$ by
\begin{equation}
\label{eq:MLP_Rep}
    \Phi_{\operatorname{rem}:q}(x)\eqdef 
        (I_q-2L)\,
        \lfloor \mathbf{A} x \rfloor
\end{equation}
where $\lfloor \cdot \rfloor$ is applied componentwise and where $L$ is the $q\times q$ left-shift matrix defined in~\eqref{eq:leftShift}.  
Note that the matrix $(I_q-2L)$ has exactly $2q-1$ non-zero parameters.
Then, for each $r\in \mathbb{R}_q\cap [0,1)$, we have that
\begin{equation}
\label{eq:fractional__bit_encoding__representation}
    \Phi_{\operatorname{rem}:q}(r) = 
    \big(
        \beta_m
    \big)_{m=1}^q
\end{equation}
where $\beta_1,\dots,\beta_q\in\{0,1\}$ are given by~\eqref{eq:fractional__bit_encoding}.  
\hfill\\
{Now, by Lemma~\ref{lem:ceiling} the network 
$\Phi_{\lfloor \cdot\rfloor:M}\sim_q \lfloor \cdot\rfloor$ 
{has depth at-most 
$\lceil \log_2(2M+1)\rceil + 3$, width $16M+8$, and $108M+54$ non-zero parameters.}
By the parallelization Lemma in~\cite[Lemma 5.3]{petersen2024mathematical}, there is a ReLU MLP 
$\Phi_{\lfloor \cdot \rfloor:M}^{\parallel q}:\mathbb{R}\to \mathbb{R}_q$
of depth at-most $\lceil \log_2(2M+1)\rceil + 3$, width at-most 
$2q(16M+8)$,
and with at-most 
$2q(108M+54)$ 
non-zero parameters, satisfying: for each $x\in \mathbb{R}_q$,
$\Phi_{\lfloor \cdot \rfloor:M}^{\parallel q}(x)=\bigoplus_{i=1}^q\, \Phi_{\lfloor \cdot \rfloor:M}(x)$.
Therefore,~\eqref{eq:MLP_Rep} can be represented as $\Phi_{\operatorname{Rem}:q}\sim_q \Phi_{\operatorname{rem}:q}$ where, for every $x\in \mathbb{R}$ we define
\[
    \Phi_{\operatorname{Rem}:q}(x)
\eqdef 
    (I_q-2L)
    \Phi_{\lfloor \cdot \rfloor:M}^{\parallel q}
    (\mathbf{A}x)
.
\]
Consequently, $\Phi_{\operatorname{Rem}:q}$ has a depth of $\lceil \log_2(2M+1)\rceil + 3$, a width at-most $2q(16M+8)$, and no more than 
$216 M q + 111 q + 1$
non-zero parameters.
}
\end{proof}
{
Next, we construct the bit encoder for integers in bounded sub-intervals of the positive real numbers.
\begin{lemma}[Bit Encoder \- Integer Part Only]
\label{lem:fraction_encoder__integerbit}
Fix $M\in \mathbb{N}_+$ and consider the map $\operatorname{Bin}_{\operatorname{Int}:q}:\mathbb{N}\cap [0,M]\mapsto \{0,1\}^{1+q}$ sending any $n\in \mathbb{N}\cap [0,M]$ to the coefficients $(\beta_i)_{i=0}^q\in \{0,1\}^{1+q}$ of its 
unique binary expansion; i.e.\ $n=\sum_{i=0}^q\, \beta_i\, 2^i$. 
There exists a ReLU MLP $\Phi_{\operatorname{Int}:q}:\mathbb{R}\to \mathbb{R}^q$ satisfying $\Phi_{\operatorname{Int}:q}\sim_q \operatorname{Bin}_{\operatorname{Int}:q}$.  
\hfill\\
Moreover, $\Phi_{\operatorname{Int}:q}$ has 
depth of 
$\lceil \log_2(2M+1)\rceil + 3$, 
width at-most 
$2(q+1)(16M + 8)$,
no more than 
$216Mq + 108q + 216M + 108$
non-zero parameters, and is represented for each $x\in \mathbb{R}$ by
\[
        \Phi_{\operatorname{Int}:q}(x)
    \eqdef 
        (I_{1+q}-2\operatorname{R}_{1+q})
        \,
        \Phi_{\lfloor \cdot \rfloor:M}^{\| q}
        \big(
            \operatorname{diag}((2^i)_{i=1}^q) 
            x
        \big)
\]
where $\operatorname{R}_{1+q}$ is the square $1+q$ right-shift matrix and
$\Phi_{\lfloor \cdot \rfloor:M}^{\| 1+q}$
is the $(1+q)$-fold parallelization of the floor network $\Phi_{\lfloor \cdot\rfloor:M}$ in 
\end{lemma}
\begin{proof}[{Proof of Lemma~\ref{lem:fraction_encoder__integerbit}}]
The proof is nearly identical to that of Lemma~\ref{lem:fraction_encoder__remainder}.  First, notice that for every integer $n$ in $[0,M]$ we have 
\begin{align*}
    \Phi_{\operatorname{Int}:q}(n) 
= 
    (I_{1+q} - 2\operatorname{R}_{1+q})
    \lfloor
        \operatorname{diag}((2^i)_{i=0}^q x
    \rfloor
=
 \begin{pmatrix}
    \left\lfloor \dfrac{n}{2^m} \right\rfloor - 2\left\lfloor \dfrac{n}{2^{m+1}} \right\rfloor \\
    \left\lfloor \dfrac{n}{2^{m-1}} \right\rfloor - 2\left\lfloor \dfrac{n}{2^{m}} \right\rfloor \\
    \vdots \\
    \left\lfloor \dfrac{n}{2^0} \right\rfloor - 2\left\lfloor \dfrac{n}{2^{1}} \right\rfloor \\  
    \end{pmatrix}
= \operatorname{Bin}_{\operatorname{Int}:q}(n)
.
\end{align*}
Upon parameter tallying nearly identically to the tally used in $\Phi_{\operatorname{Rem}:q}$, with $1+q$ in place of $q$ and without the parameters in the extra bias term $\operatorname{1}_q$, we obtain the conclusion.
\end{proof}
Combining the above into a single ReLU \textit{feedforward neural network} performing bit extraction on $\mathbb{R}_q$ yields our result.
\begin{proposition}[Bit Encoder]
\label{prop:bit_encoder}
Let $M,q\in \mathbb{N}_+$.  There exists a ReLU feedforward neural network $\Phi_{\operatorname{Bin}:M,q}:\mathbb{R}\to \mathbb{R}^{2q+2}$ satisfying $\Phi_{\operatorname{Bin}:M,q}\sim_q \operatorname{Bin}_q$.  Moreover, $\Phi_{\operatorname{Bin}:M,q}$ has
depth 
$2 \lceil log_2 (2M + 1) \rceil + 7$
width 
$128 Mq + 64M + 64q + 36$
at-most 
$432 M q + 221 q + 432 M + 254$
non-zero parameters, and representation
\[
    \Phi_{\operatorname{Bin}:M,q}
    \eqdef 
    \big(
        \Phi_{\operatorname{Int}:q}(P_1\cdot)
        ,
        \Phi_{\operatorname{Rem}:q}(P_2\cdot)
        ,
        \Phi_{[0,M+\tfrac{1}{2^{q+1}}):q}(P_3\cdot )
    \big)
    \circ 
    \Phi_{\operatorname{Bit}:M:+}
\]
where $\Phi_{\operatorname{Bit}:M:+}$ is the ReLU MLP defined in Lemma~\ref{lem:encoder_positive}, 
$\Phi_{\operatorname{Int}:q}$ is the ReLU MLP of Lemma~\ref{lem:fraction_encoder__integerbit},
$\Phi_{\operatorname{Rem}:q}$ is the ReLU MLP of Lemma~\ref{lem:fraction_encoder__remainder}
, $\Phi_{[0,M+\tfrac{1}{2^{q+1}}):q}$ is the ReLU MLP of Corollary~\ref{cor:HalfOpenInterval__complement}, and $P_1(x)\eqdef (x_i)_{i=1}^q$, $P_2(x)\eqdef (x_i)_{i=q+1}^{2q+1}$, and $P_3(x)=x_{2q+2}$ are projection matrices on $\mathbb{R}^{2q+2}$.
\end{proposition}
\begin{proof}[{Proof of Proposition~\ref{prop:bit_encoder}}]
By construction, i.e.\ from the quoted lemmata, $\Phi_{\operatorname{Bin}:M,q}$ satisfies $
    \Phi_{\operatorname{Bin}:M,q}\sim \operatorname{Bin}_{M:q}
$.
Tallying the number of non-zero parameters in each network, as well as the $2q+2$ non-zero parameters totalled in the projection matrices, yields our parameter count.
\end{proof}
}
\subsubsection{Additional Operations}
\label{s:BitManipulators__ss:Additional}

We now compile a list of networks that we will often use to manipulate the bits of strings that our networks are processing.  These include networks helping with modular arithmetic base $2$ and networks which shift bits along either left or right.

To implement multiplication, we must treat $i+j \mod(2)$ for integers $-2\le i,j\le 2$.  We thus implement an MLP which can perform the $\pmod 2$ operation on this small set of integers.  
\begin{lemma}
\label{lem:mod2}
Consider the following ReLU MLP $\Phi_{\operatorname{mod}{2}}:\mathbb{R}\to [0,1]$ 
of width $6$, depth $2$ and $24$ non-zero parameters, 
defined for each $x\in \mathbb{R}$ by
\allowdisplaybreaks
\begin{align}
\label{eq:lem:mod2_def}
    \Phi_{\operatorname{mod}{2}}(x)
    \eqdef
    (1,1,1)
    \operatorname{ReLU}\left(
        \begin{pmatrix}
            -1 & -1  & 0 & 0 & 0 & 0 \\
            0 & 0 & -1 & -1  & 0 & 0 \\
             0 & 0 & 0 & 0 & -1 & -1 \\
        \end{pmatrix}
        \operatorname{ReLU}\left(
            \mathbf{1}_6
            x 
            +
            \begin{pmatrix}
                1\\
                -1 \\
                -1 \\
                -3 \\
                3 \\
                1
            \end{pmatrix}
        \right)
        +
        \begin{pmatrix}
            1\\
            1\\
            1
        \end{pmatrix}
    \right)
.
\end{align}
Then, for each $x\in \{-2,1,0,1,2\}$, we have $
    \Phi_{\operatorname{mod}{2}}(x) = x \pmod 2
$.
\end{lemma}
\begin{proof}[{Proof of Lemma~\ref{lem:mod2}}]
Consider the map $g:\mathbb{R}\to \mathbb{R}$ defined for each $x\in \mathbb{R}$ by
\[
    g(x)
\eqdef 
    \operatorname{ReLU}\left(-\operatorname{ReLU}\left(-x+1\right)+1-\operatorname{ReLU}\left(x-1\right)\right)
\]
and define the map $f:\mathbb{R}\to \mathbb{R}$ for each $x\in \mathbb{R}$
\[
    f\left(x\right)
\eqdef 
    g\left(x\right)+g\left(x-2\right)+g\left(x+2\right)
\]
and note that $f(-2)=f(0)=f(2)=0$ and $f(-1)=f(1)=1$.  Thus, for each $x\in \{-2,1,0,1,2\}$, we deduce that $f(x)=x\pmod 2$.

\noindent To realize $f$ as a ReLU MLP $\Phi_{\operatorname{mod}{2}}:\mathbb{R}\to \mathbb{R}$.  For this note that for each $x\in \mathbb{R}$ we see that $\Phi_{\operatorname{mod}{2}}$, as defined in~\eqref{eq:lem:mod2_def}, implements $\mod{2} on $\{-2,-1,0,1,2\}.
This completes our proof.
\end{proof}

Finally, we will rely on the following set of MLPs which implement left and right shifts of the binary representations of integers.
\paragraph{Neural network representation of \LEFT} 
A left-shift on $\{0,1\}^B$ is implemented by multiplication by the following matrix:
\begin{equation}
\label{eq:leftShift}
L \eqdef 
\left[
\begin{array}{cccccc}
0 & 1 & 0 & 0 & \cdots & 0 \\
0 & 0 & 1 & 0 & \cdots & 0 \\
0 & 0 & 0 & 1 & \cdots & 0 \\
0 & 0 & 0 & 0 & \ddots & \vdots \\
\vdots & \vdots & \vdots & \vdots & \ddots & 1 \\
0 & 0 & 0 & 0 & \cdots & 0
\end{array}
\right]
.
\end{equation}

\begin{lemma}[Left Shift]
\label{lem:left}
For any $B\in \N$, $a\in \{0,1\}^B$ we have ${I_B}\ReLU(L a) = \LEFT(a)$. Thus, there exists a bitwise ReLU neural network $\psi: \{0,1\}^B \to \{0,1\}^B$ of width $B$ and depth $1$, such that $\psi(a) = \LEFT(a,b)$.
\end{lemma}

\paragraph{Neural network representation of \RIGHT} A right-shift on $\{0,1\}^B$ is implemented by multiplication by the following matrix:
\[
R \eqdef 
\left[
\begin{array}{cccccc}
0 & 0 & 0 & 0 & \cdots & 0 \\
1 & 0 & 0 & 0 & \cdots & 0 \\
0 & 1 & 0 & 0 & \cdots & 0 \\
0 & 0 & 1 & 0 & \cdots & 0 \\
\vdots & \vdots & \vdots & \ddots & \ddots & 0 \\
0 & 0 & \cdots & 0 & 1 & 0
\end{array}
\right]
\]

\begin{lemma}[Right Shift]
\label{lem:right}
For any $B\in \N$, $a\in \{0,1\}^B$ we have ${I_B}\ReLU(R a) = \RIGHT(a)$. Thus, there exists a bitwise ReLU neural network $\psi: \{0,1\}^B \to \{0,1\}^B$ of width $B$ and depth $1$, such that $\psi(a) = \RIGHT(a)$.
\end{lemma}

\subsection{Tropical Gates}
\label{s:Operations__ss:Tropical}
For dynamic programming applications, we will rely on the following tropical gates.
\begin{lemma}[Minimum and Maximum Function - {\cite[Lemma 5.11]{petersen2024mathematical}}]
\label{lem:Minimum}
For every $n\in \mathbb{N}_+$ with $ n\geq 2$ there exists a neural network $\Phi_n^{\min}: \mathbb{R}^n \to \mathbb{R}$ with depth $\lceil \log_2(n)\rceil$, width at-most $3n$, and 
at-most $16n$ non-zero parameters: 
such that $\Phi_n^{\min}(x_1, \dots, x_n) = \min_{1 \leq j \leq n} x_j$. 
\hfill\\
Similarly, there exists a neural network $\Phi_n^{\max}: \mathbb{R}^n \to \mathbb{R}$ realizing the maximum and satisfying the same complexity bounds.
\end{lemma}
\begin{lemma}[Median Function - {\cite[Lemma 7.2]{hong2024bridging}}]
\label{lem:median_estimators_have_small_errors}
    Let $n\in\mathbb{N}_+$, the median function on $2n + 1$ non-negative inputs can be implemented by a ReLU MLP of depth $11n +3$, width $6n + 3$.
\end{lemma}

\begin{lemma}[Majority Function]
\label{lem:Maj}
Let $n\in \mathbb{N}_+$ be odd and define the \textit{majority function}, then $\operatorname{Maj}_n$  can be implemented by a ReLU MLP of depth $11\lfloor \frac{n}{2}\rfloor+4$ and width $6\lfloor \frac{n}{2}\rfloor+3$.
\end{lemma}
\begin{proof}[{Proof of Lemma~\ref{lem:Maj}}]
Consider the ReLU MLP $\Phi:\mathbb{R}\to \mathbb{R}$ given by
\[
    \Phi(x)
    \eqdef 
    -2\operatorname{ReLU}(x-1)+2\operatorname{ReLU}(x)
    =
    (-2,2)^{\top}\operatorname{ReLU}\big((1,1)x + (-1,1)\big).
\]
Then $\Phi(0)=0$, $\Phi(\frac{1}{2})=1$, and $\Phi(1)=0$.  

Next, observe that for any Boolean vector $x\in \{0,1\}^{2n+1}$, $\operatorname{median}(x)\in \{0,\frac{1}{2},1\}$; further, $\operatorname{median}(x)>0$ only if $\lfloor \frac{2n+1}{2}\rfloor$ components are non-zero.  Thus,
\[
\operatorname{Maj}_{2n+1}= \Phi\circ \operatorname{median}.
\]
Consequentially, $\operatorname{Maj}_{2n+1}$, Lemma~\ref{lem:median_estimators_have_small_errors} implies that $\operatorname{Maj}_n$ can be realized as a ReLU MLP of depth $11n+4$ and width $\max\{2,6n+3\}=6n+3$.
\end{proof}

\subsection{Binary Arithmetic Operations}
\label{s:ArithmeticBitWise_Gates}
Our first aim in the following is to show that floating-point arithmetic operations on bit-strings can be represented exactly by neural networks (NNs). These {NNs} have conventional feedforward architectures, constructed to map bit-string inputs, i.e. vectors of the form $(x_1,\dots, x_B)$ consisting of bits $x_1,\dots, x_B \in \{0,1\}$, to bit-string outputs. We construct these NNs to perform floating-point operations, which are at the core of most (numerical) algorithms.
We will rely on NN representations of elementary logical operations on bits; using them, we will build representations of binary integer arithmetic and then subsequently combine them to form NN representations of floating-point arithmetic operation.
Floating-point operations can be reduced to arithmetic operations on binary representations of integers, of the form
\[
\pm \sum_{j=0}^{B-1} a_j 2^{j}, \quad a_0,\dots, a_{B-1} \in \{0,1\}.
\]

\paragraph{Unsigned integers}

Unsigned integers with $B$ bits are represented by $(a_{B-1},\dots, a_0) \in \{0,1\}^{B}$, corresponding to the integers $\sum_{j=0}^{B-1} a_j 2^{j}$. These bit-strings can uniquely represent all integers $\{0, \dots, 2^B-1\}$. We define $\UINT_B  \eqdef  \{0,1\}^{B}$, and note the one-to-one correspondence,
\[
\UINT_B 
\simeq 
\left\{
\sum_{j=0}^{B-1} a_j 2^j
\, : \, 
a_0, \dots, a_{B-1} \in \{0,1\}
\right\}.
\]

\paragraph{Addition (mod $2^B$)} Our aim in this paragraph is to bound the parameter count of a neural network which exactly represents addition modulo $2^{B}$. The inputs to the network will be two elements $a,b\in \UINT_B$, i.e. bit-representations of elements in $\Z\cap [0,2^B)$. The output will be the bit-representation of $a+b \pmod{2^B}$. To this end, we observe that addition of two numbers $a,b$ encoded by $B$ bits can be written as $B$-fold iteration of the following bitwise operations:
\begin{align}
\label{eq:add-1}
\begin{pmatrix}
a \\
b
\end{pmatrix}
\mapsto 
\begin{pmatrix}
\XOR(a,b) \\
\LEFT(\AND(a,b))
\end{pmatrix}.
\end{align}
The $\XOR$ computes the addition without the carry, $\AND$ computes the carry for each bit, the $\LEFT$ moves the carry one bit up (which then, in the next iteration is going to be added). 

\begin{example}
For illustration, we consider $a = 0001$, $b=0111$. In this case, $a \simeq 1$ and $b \simeq 7$, so that $a+b \simeq 8$, i.e. $a+b$ has bit-representation $1000$. In this case, using bitwise operations, we have the following steps:\\[0.5em] 

 \textbf{Step 1:} 
 $\XOR(a,b) = 0110$, $\AND(a,b) = 0001$, $\LEFT(\AND(a,b)) = 0010$. 

 We now replace $a \leftarrow 0110$, $b \leftarrow 0010$.\\[0.5em]

 \textbf{Step 2:} 
 $\XOR(a,b) = 0100$, $\AND(a,b) = 0010$, $\LEFT(\AND(a,b)) = 0100$. 

 We now replace $a \leftarrow 0100$, $b \leftarrow 0100$.\\[0.5em]

 \textbf{Step 3:} 
 $\XOR(a,b) = 0000$, $\AND(a,b) = 0100$, $\LEFT(\AND(a,b)) = 1000$. 

 We now replace $a \leftarrow 0000$, $b \leftarrow 1000$.\\[0.5em]

\textbf{Step 4:} 
 $\XOR(a,b) = 1000$, $\AND(a,b) = 0000$, $\LEFT(\AND(a,b)) = 0000$.

 We now replace $a \leftarrow 1000$, $b \leftarrow 0000$, and we terminate the algorithm. \\[0.5em]
 
 The result is the value of $a$, i.e. $a=1000$ the bit-representation of $8$. Here we had $B = 4$ bits and it took 4 steps.
\end{example}

We now show that bitwise addition can be represented efficiently by bitwise ReLU neural networks. To this end, we consider $\UINT_B = \{0,1\}^B$ as a subset of $\R^B$, so that the application of a feedforward neural network to $a,b \in \UINT_B$ is well-defined.

\begin{lemma}
\label{lem:add-1}
For any $B\in \N$ and $a,b \in \UINT_B$, we can represent \eqref{eq:add-1} by a ReLU neural network $\psi: \R^B \times \R^B \to \R^B$ of width $2B$ and depth $1$.
\end{lemma}

\begin{proof}
Recall the left-shift matrix $L$ from Lemma \ref{lem:left}. We can
represent the following mapping by a bitwise ReLU neural network $\psi$ of width $2B$ and depth $1$.
{
\[
\begin{pmatrix}
a \\
b
\end{pmatrix}
\mapsto 
\begin{pmatrix}
\ReLU(a+b) \\
\ReLU(a+b-1)
\end{pmatrix}
\mapsto 
\begin{pmatrix}
\ReLU(a+b) - 2 \ReLU(a+b-1) \\
L \, \ReLU(a+b-1)
\end{pmatrix},
\]
}
By Lemma \ref{lem:xor}, the first component of the output is $\XOR(a,b)$. By Lemma \ref{lem:and} and the definition of $L$,
the second component of the output is $\LEFT(\AND(a,b))$.
\end{proof}

As a consequence of the last lemma, it follows that there exists a neural network acting on bit-string inputs, with width $2B$ and depth at most $B$, which implements addition modulo $2^B$, denoted $\ADD_B(a,b)$.

\begin{lemma}
\label{lem:add}
For any $B\in \N$, we can represent bitwise addition $\ADD_B: \UINT_B \times \UINT_B \to \UINT_B$ by a ReLU neural network $\psi: \R^B \times \R^B \to \R^B$ of width $2B$ and depth $B$.
\end{lemma}

\begin{proof}
$\ADD_B$ is obtained by a $B$-fold composition of the operation \eqref{eq:add-1}. By Lemma \ref{lem:add-1}, this operation can be represented exactly by a ReLU neural network $\tilde{\psi}$ of width $2B$ and depth $1$. We define $\psi  \eqdef  \pi_1\tilde{\psi}\circ \dots \circ \tilde{\psi}$ as the $B$-fold composition of $\tilde{\psi}$, where $\pi_1(a,b) = a$ is a projection onto the first component. The claim now follows easily by~\cite[Lemma 5.2]{petersen2024mathematical}\footnote{If the user would rather have a feedforward emulator, rather than an MLP emulator, then there is no need to apply \cite[Lemma 5.2]{petersen2024mathematical}.}.
\end{proof}

\paragraph{Signed integers}
We now consider bit-representations of signed integers in the range $\Z \cap (-2^B,2^B)$. Using a single bit $\rho\in \{0,1\}$ to identify the sign as $(-1)^\rho$, we consider
bit-strings $a = (\rho, a_{B-1}\dots a_0) \in \{0,1\}^{B+1}$.
These bit-strings are in correspondence with the integers in $\{-(2^B-1),\dots, 2^B-1\}$, by associating
\[
(-1)^\rho \sum_{j=0}^{B-1} a_j 2^j,
\]
with each bit-string. This association is one-to-one except at $0$, which possesses two representations, $(\rho,0,\dots,0)$ for both $\rho=0$ and $\rho=1$. We define the relevant set of bit-strings as $\INT_B  \eqdef  \{0,1\}^{B+1}$. 

\paragraph{Two's complement representation} 
The signed integer representation above is very intuitive, but not ideally suited to performing arithmetic operations. An alternative is based on the so-called ``two's complement''. Similar to signed integers, we employ a representation in terms of $B+1$ bits: Non-negative integers in $\{0,\dots, 2^B-1\}$ are still represented by bit-strings of the form $(0,a_{B-1},\dots, a_0)$. However, in contrast to the signed integer representation introduced before, negative integers are instead identified with additive inverses under $2^{B+1}$-modular arithmetic, and representing elements of $\{2^B,\dots, 2^{B+1}-1\}$. To briefly summarize this, we consider an unsigned integer $a$, represented by $(a_B,a_{B-1}, \dots, a_0)$ with $a_B = 0$ and we note that in $2^{B+1}$-modular arithmetic,
\begin{align*}
-a &\equiv -\sum_{j=0}^{B} a_j 2^j 
\equiv 2^{B+1} - \sum_{j=0}^{B} a_j 2^j 
\equiv 2^0 + \sum_{j=0}^{B} (1-a_j) 2^j \pmod{2^{B+1}}.
\end{align*}
The expression on the right is represented by the bit-string $\widebar{a} = \ADD_{B+1}(0\dots 01, \NOT(a))$, and it can be verified that $\ADD_{B+1}(a,\widebar{a}) = 0$. Thus, $\widebar{a}$ is indeed a bit-string representation of $-a$, under $2^{B+1}$-modular arithmetic. Let us now introduce
\[
\COMP_{B+1}(a)  \eqdef  \ADD_{B+1}(0\dots 01, \NOT(a)).
\]
We then have the following result,
\begin{lemma}
\label{lem:comp}
$\COMP_{B+1}$ can be represented by a neural network of width $2B+2$ and depth $B+2$.
\end{lemma}
\begin{proof}
By Lemma \ref{lem:add}, $\ADD_{B+1}$ can be represented by a ReLU neural network of width $2B+2$ and depth {$B+1$}. By Lemma \ref{lem:not}, $\NOT$ can be represented by a ReLU neural network of width
$B+1$ and depth $1$. By \cite[Lemma 5.2]{petersen2024mathematical}, the composition $\COMP_{B+1}(a) = \ADD_{B+1}(0\dots 01, \NOT(a))$ can be represented by a neural network of width $2B+2$ and depth $B+2$.
\end{proof}

The positive integers are represented by strings of the form $(0,a_{B-1},\dots, a_0)$. Applying $\COMP$, such a string gets mapped to a string of the form $(1,\widebar{a}_{B-1},\dots, \widebar{a}_0)$, where at least one of $\widebar{a}_0, \dots, \widebar{a}_{B-1}$ is non-zero. Since $\COMP_{B+1}$ is an implementation of $a \mapsto -a$ under $2^{B+1}$ arithmetic, applying this mapping twice gives the identity, i.e. $\COMP_{B+1}\circ \COMP_{B+1} = \IDENTITY$ (henceforth assumed to be the neural network defined in \cite[Lemma 5.1]{petersen2024mathematical}).

To summarize, the unsigned integer representation from before ($\INT_B$), and the two's complement representation here, both involve bit-strings of length $B+1$. However, the interpretation of these bit-strings is very different. To make this distinction apparent in our notation, we define $\TINT_B  \eqdef  \{0,1\}^{B+1}$. 

We can now define a bijection between the signed integer representation $\INT_B$ and the two's complement representation $\TINT_B$, by defining $\Phi_{\INT_B \to \TINT_B}: \INT_B \to \TINT_B$, as
\begin{align}
\label{eq:bij-comp}
\Phi_{\INT_B \to \TINT_B}(a)
 \eqdef  
\begin{cases}
a, &\text{if $a = (0,a_{B-1},\dots, a_0)$}, \\
\COMP_{B+1}(0, a_{B-1}\dots a_0), &\text{if $a = (1,a_{B-1},\dots, a_0)$}.
\end{cases}
\end{align}
The following is immediate from our previous discussion:
\begin{lemma}
\label{lem:int-conv}
Let $a = (\rho, a_{B-1}, \dots, a_0) \in \INT_B$ be a signed integer bit-string, where $\rho$ represents the sign $(-1)^\rho$. Denote $\hat{a}  \eqdef  (a_{B-1},\dots, a_0)$, i.e. $a$ with the sign-bit $\rho$ removed. Then 
\[
\Phi_{\INT_B \to \TINT_B}(a)
= 
\ADD_{B+1}( \, \AND(\NOT(\rho), 0\hat{a}), \, \AND(\rho, \COMP_{B+1}(0\hat{a})) \, ).
\]
{converts from signed integers to two's complement. A careful analysis will reveal that $\Phi_{\INT_B \to \TINT_B}$ is in fact an involution;
and so, the inverse is computed in the same way:}
{\[
\Phi_{\INT_B \to \TINT_B}^{-1}(a)
= 
\ADD_{B+1}( \, \AND(\NOT(\rho), 0\hat{a}), \, \AND(\rho, \COMP_{B+1}(0\hat{a})) \, ).
\]}
\end{lemma}
{Note that in the neural network, $\rho$ is either the $1^B$ or $0^B$ vector.}
Thus, conversion between signed integer $\INT_B$ and two's complement $\TINT_B$ representations can be represented by a combination of modular addition, logic gates and $\COMP_{B+1}$. We also note that 
\begin{lemma}
\label{lem:int-to-tint}
$\Phi_{\INT_B \to \TINT_B}$ and $\Phi_{\INT_B \to \TINT_B}^{-1}$ can be represented by a neural network of width 
{$7B+2$}
and depth $2B+4$.
\end{lemma}
\begin{proof}
Let $\pi_\rho, \hat{\pi}$ be projections, defined by $\pi_\rho(a) = \rho$, $\hat{\pi}(a) = 0\hat{a}$. We can write $\Phi_{\INT_B \to \TINT_B}$ as a composition involving the following sequence of maps:
\[
\begin{pmatrix}
\NOT\circ \pi_\rho \\
\IDENTITY\circ \hat{\pi} \\
\IDENTITY\circ \pi_\rho \\
\COMP_{B+1}\circ \hat{\pi}
\end{pmatrix}
\to 
\begin{pmatrix}
\AND \\
\AND
\end{pmatrix}
\to \ADD_{B+1}.
\]
{By Lemma \ref{lem:not}, $\NOT$, acting on bit-strings of length 
$B$ can be represented by a neural network of depth $1$, and of width 
$B$.
By~\cite[Lemma 5.1]{petersen2024mathematical}, $\IDENTITY$ can be represented by a neural network of depth $1$ and width $2B$. By Lemma \ref{lem:comp}, $\COMP_{B+1}$ is representable by a neural network of width $2B+2$ and depth $B+2$. Parallelizing these networks, it follows that the first mapping in the display above can be represented by a neural network of width 
$B + 2B + 2B + (2B+2) = 5B + 2$
and depth 
$B+2$.}

By Lemma \ref{lem:and}, $\AND$, acting on bit-strings of length $B+1$, can be represented by a neural network of width $B+1$ and depth $1$. By parallelization, $(\AND, \AND)$ can thus be represented by a network of width $2B+2$ and depth $1$.

By Lemma \ref{lem:add} $\ADD_{B+1}$ can be represented by a neural network of width $2B+2$ and depth $B+1$. 

Composing these network representations, we conclude that $\Phi_{\INT_B \to \TINT_B}$ can be represented by a neural network of width
$\max(7B+2, 2B+2, 2B+2) = 7B+2$, and depth $(B+2) + 1 + (B+1) = 2B+4$.

A similar argument applies to $\Phi^{-1}_{\INT_B \to \TINT_B}$, leading to the claim.

\end{proof}

\paragraph{Exact Integer Addition}
Exact integer addition on signed integer bit-strings
$a,b \in \TINT_B$
generally requires outputs with $B+2$ bits, i.e. elements in $\TINT_{B+1}$. Indeed, if $a,b$ are bit-strings representing elements in
$\Z \cap [-2^B, 2^B - 1]$,
then $a+b$ is in
{$\Z \cap [-2^{B+1},2^{B+1} - 1]$}
and hence can be represented by a bit-string in
{$\TINT_{B+1}$}.
{Define}
$\EMBED_B: \TINT_B \to \TINT_{B+1}$ by
\[
\EMBED_B(a) = \EMBED(a_B \dots a_0)
 \eqdef 
\begin{cases}
0a_{B}\dots a_0, & \text{if } a_B = 0,\\
1a_{B}\dots a_0, & \text{if } a_B = 1.
\end{cases}
\]
{In particular, note how this embedding preserves additive inverses, where} {$\EMBED_B(a) = a_{B+1}a_Ba_{B-1} \dots a_0$} is obtained by replicating the first bit. {So,
$$
    \EMBED_B =
    \begin{pmatrix}
        1 & 0 & 0 & \dots & 0\\
        & & I_B
    \end{pmatrix}.
$$}
{We can realize exact integer addition with a neural network:}
\begin{lemma}[Exact integer addition]
For any $B\in \N$, there exists a neural network $\psi: \R^{B+1} \times \R^{B+1} \to \R^{B+2}$ of width at-most {$2B+4$} and depth {$B+2$}, such that $\psi(a,b)$ is the (exact) signed {two's complement} integer bit-string of $a+b$, represented by a bit-string in $\TINT_{B+1}$.
\end{lemma}

\begin{proof}
We can write the mapping $\psi$ in terms of the following sequence of maps:
\[
\begin{pmatrix}
    a\\
    b
\end{pmatrix}
\mapsto
{
\begin{pmatrix}
\EMBED_B \\
\EMBED_B
\end{pmatrix}
\mapsto
}
{\ADD_{B+2}}
\mapsto a+b
.
\]
By Lemma \ref{lem:add}, $\ADD_{B+2}$ can be represented by a neural network of width $2B+4$ and depth $B+2$.

Composing these networks, we obtain a neural network representation of exact integer addition $\TINT_B \times \TINT_B \to \TINT_{B+1}$, with a neural network of width
{$2B + 4$}
and depth
{$B+2$}.
\end{proof}

\paragraph{Exact Integer Multiplication}

The multiplication of $a,b\in \UINT_B$, with representations $a = (a_{B-1} \dots a_0)$ and $b=(b_{B-1} \dots b_0)$, has the following properties: Firstly, observe that if $a=(0\dots0a_0)$, then we either multiply by $0$ (if $a_0=0$) or $1$ (if $a_0=1$), and hence
\[
\MULT(0\dots 0a_0, b) = (\AND(a_0,b_{B-1}), \dots, \AND(a_0,b_0)).
\]
On the other hand, if $a=(0\dots010)$ (representing the integer $2$), then 
\[
\MULT(0\dots 010, b) = \LEFT(b).
\]
In general, we then have
\[
\MULT(a_{B-1}\dots a_0,b) = \ADD(\; \MULT(a_{B-1}\dots a_10,b), \; \MULT(0\dots 0a_0, b)\; ),
\]
and
\[
\MULT(x_B\dots x_20,y)
= \MULT(0\dots010,\MULT(x_B\dots x_2,y))
= \LEFT(\MULT(x_B\dots x_2,y)).
\]
so we have the recursion,
\[
\MULT(x_B\dots x_1,y) = \ADD(\; \LEFT(\MULT(\RIGHT(x),y)), \; \MULT(0\dots 0x_1, y)\; )
\]
This should allow us to compute the required depth and width to represent $\MULT$; if $\MULT_1(x_1,y) \eqdef \MULT(0\dots 0x_1, y)$, $\MULT_2(x_2x_1,y) \eqdef \MULT(0\dots 0x_2x_1,y)$ etc., then we know that we can represent $\MULT_1(x_1,y)$ by a ReLU neural network of width
{$2B$}
and depth $1$. The recursion is then 
\[
\MULT_B(x_B\dots x_1,y) 
= 
\ADD(\; 
\LEFT(\MULT_{B-1}(
{\RIGHT(x_B \dots x_1),y}
)), 
\; 
\MULT_1(x_1, y)
\; )
\]
This implies that 
\begin{align*}
\mathrm{width}(\MULT_B) 
&\le 
\max\left\{
\mathrm{width}(\ADD(\LEFT(\cdot), \cdot)),
\mathrm{width}(\MULT_{B-1}) + \mathrm{width}(\MULT_1)
\right\}
\\
&= 
\max\left\{
{2B},
\mathrm{width}(\MULT_{B-1})
{+ 2B}
\right\}
\\
&= 
\mathrm{width}(\MULT_{B-1}) 
{+ 2B.}
\end{align*}
By recursion on $B$, it follows that {$\mathrm{width}(\MULT_B) \le B^2 + 1$.}

Similarly, for the depth, we have
\begin{align*}
\mathrm{depth}(\MULT_B) 
&\le 
\mathrm{depth}(\ADD(\LEFT(\cdot), \cdot))
+
\max\{\mathrm{depth}(\MULT_{B-1}), \mathrm{depth}(\MULT_1)\}
\\
&= 
B + \mathrm{depth}(\MULT_{B-1}).
\end{align*}
{By recursion on $B$, it follows that}
$
\mathrm{depth}(\MULT_B) \le B^2.
$
Let us summarize the above discussion in the following lemma:
\begin{lemma}[Bitwise Multiplication]
\label{lem:mult}
For any $B\in \N$, we can represent bitwise multiplication $\MULT: \UINT_B \times \UINT_B \to \UINT_B$ by a bitwise ReLU neural network $\psi_{\MULT}: \{0,1\}^B \times \{0,1\}^B \to \{0,1\}^B$ of width {$\le B^2+1$} 
and depth 
{$\le B^2$}.
\end{lemma}

\subsection{Modular Arithmetic in \texorpdfstring{$\mathbb{R}_q$}{Rq}}
\label{s:ModularArithmeticGrid}
We will construct neural networks that compute modular addition and multiplication.\footnote{We could have also defined addition and multiplication in $\mathbb{R}_q$ (and the equivalent neural networks) to be \emph{saturated}. We leave it as an exercise for the reader to verify that it is possible.}

Given $a, b \in \mathbb{R}_q$, we show how to compute $a +_{\mathbb{R}_q} b \pmod {2^{q+1}}$; it follows that it is sufficient to compute
$2^{-q}(2^qa +_{\mathbb{R}_q} 2^q b) \pmod {2^{q+1}}$. Towards this goal, we will embed $a, b$ into $\TINT_{2q+1}$, as $\ADD_{2q+2}$ as defined in Lemma \ref{lem:add} performs addition modulo $2^{q+1}$ on members of $\TINT_{2q+2}$. Using the fact that 
$
    b \pmod c = \frac{1}{a} (ab \pmod {ac}),
$ we
divide the result by $2^q$.

\begin{lemma}[Modular Addition in $\mathbb{R}_q$]
\label{lem:Modular_Add}
Fix $q \in \mathbb{N}_+$. Then, there exists a $\ReLU$ MLP $\Phi_{+,q} : \mathbb{R}_q \times \mathbb{R}_q \to \mathbb{R}_q$ of depth and width $\mathcal{O}(q)$ such that for each $a, b \in \mathbb{R}_q$,
$$
        \Phi_{\times, q}(a, b) =
        \Big(
                \underbrace{\lfloor |y|\rfloor}_{n: \text{integer part}}
            ,
                \underbrace{|y|-\lfloor |y|\rfloor}_{r: \text{remainder}}
            ,
                \underbrace{
                    I_{y\ge 0}
                }_{s: \text{sign}}
        \Big)^{\top}
    $$
    where $y = a +_{\mathbb{R}_q} b \pmod {2^{q+1}}$.
\end{lemma}
\begin{proof}
    Based on the discussion above, we can summarize our intended computation as a sequence of maps
    $$
        \begin{pmatrix}
            a\\
            b
        \end{pmatrix}
        \mapsto
        \begin{pmatrix}
            \Phi_{\INT_{2q+1} \to \TINT_{2q+1}}(R(a))\\
            \Phi_{\INT_{2q+1} \to \TINT_{2q+1}}(R(b))
        \end{pmatrix}
        \mapsto
        \begin{pmatrix}
            L \big( \ADD_{2q+2}(\cdot, \cdot) \big)
        \end{pmatrix},
    $$
    where $R$, $L$ are right and left shift operators that move the sign bit to the leftmost bit and rightmost bit, respectively. Now recall that for any $B \in \mathbb{N}_+$, we have that $depth(\ADD_B) \le B$,
    $depth(\Phi_{\INT_B \to \TINT_B}) \le 2B+4$,
    $width(\ADD_B) \le 2B$, and
    $width(\Phi_{\INT_B \to \TINT_B}) \le 7B+2$ by Lemmas \ref{lem:add} and \ref{lem:int-to-tint}. Lemmas 5.2 and 5.3 of \cite{petersen2024mathematical} imply the existence of a $\ReLU$ MLP 
    $$
        \Phi_{+, q}(a, b) = L(\ADD_{2q+2}(\Phi_{\INT_{2q+1} \to \TINT_{2q+1}}(R(a)), \Phi_{\INT_{2q+1} \to \TINT_{2q+1}}(R(b))))
    $$
    such that $depth(\Phi_{+, q}) \le 6q + 8$ and $width(\Phi_{+, q}) \le 56 + 36$.
Hence, the claim holds.
\end{proof}

Given $a, b \in \mathbb{R}_q$, we show how to compute $a \timesRq b \pmod {2^{q+1}}$. Notice that $2^{-q} \le \mAbs{a \timesRq b} < 2^{2q+2}$. Thus, we only need $4q + 2 = 2(2q+1)$ bits to capture multiplication.
Since $a, b \in \mathbb{R}_q$, we have
\begin{align*}
    a &= (-1)^{s_a} \sum^{q}_{i = -q} \alpha_i 2^i,\\
    b &= (-1)^{s_b} \sum^{q}_{i = -q} \beta_i 2^i
\end{align*}
where $\alpha_i, \beta_i \in \mSet{0, 1}$ for all $-q \le i  \le q$. Thus, we can view $a \timesRq b$ as
$$
    (-1)^{s_a + s_b} \ 2^{-q}\left( \sum^{2q}_{i = 0} \alpha_i 2^i \right) \timesRq 2^{-q} \left( \sum^{2q}_{i = 0} \beta_i 2^i \right).
$$
Now for any $q$, let
$$
    A_{4q+2} =
    \begin{pmatrix}
        0_{2q+1}, 0_{(2q+1, 1)}\\
        I_{2q+1}, 0_{(2q+1, 1)}
    \end{pmatrix}.
$$
Given $a = (\alpha_i)^q_{i = 0}  \ \oplus (\alpha_i)^{-1}_{i = -q} \oplus  \ a_s$ and $b = (\beta_i)^q_{i = 0}  \ \oplus (\beta_i)^{-1}_{i = -q} \oplus  \ a_s$, we can compute the sign bit with
$\ADD_1(s_a, s_b)$. For the integer and fractional components, first compute
$$
    x = \MULT_{4q+2}(A_{4q+2} (a), A_{4q+2}(b))
$$
and notice that the integer component is contained in the leftmost $2q+2$ bits and the fractional component is contained in the rightmost $2q$ bits. We can round $x$ to have $q$ bits in the fractional component by rounding up.
That is, compute
$$
    g = \ADD_{4q+2}(x, \MULT(1, P_{3q+3}(x))),
$$
where $P_{3q+3}$ is a $4q+2$ by $4q+2$  matrix with $1$ at $(3q+3, 3q+3)$ and $0$ elsewhere. We can represent $g$ as 
$(\gamma_i)^{2q+2}_{i = 0}  \ \oplus (\gamma_i)^{-1}_{i = -2q}$
Notice that
\begin{align*}
    \sum^{2q+2}_{i=0} \gamma_i 2^i &= \sum^{q}_{i=0} \gamma_i 2^i + \sum^{2q+2}_{i=q+1} \gamma_i 2^i
    \equiv \sum^{q}_{i=0} \gamma_i 2^i \pmod {2^{q+1}}.
\end{align*}
Now define
$$
    T =
    \begin{pmatrix}
        0_{(2q+1, q+1)}  & I_{q+1} & 0_{(q+1, q)}
    \end{pmatrix}
$$
to be the matrix that selects the bits $\gamma_i$ for $-q \le i \le q+1$. Then $T(g)$ is the computation of $a \timesRq b$.

\begin{lemma}[Modular Multiplication in $\mathbb{R}_q$]
\label{lem:Modular_Mult}
    Fix $q \in \mathbb{N}_+$. Then, there exists a $\ReLU$ MLP $\Phi_{\times, q} : \mathbb{R}_q \times \mathbb{R}_q \to \mathbb{R}_q$ of depth and width $\mathcal{O}(q^2)$ such that for each $a, b \in \mathbb{R}_q$,
    $$
        \Phi_{\times, q}(a, b) =
        \Big(
                \underbrace{\lfloor |y|\rfloor}_{n: \text{integer part}}
            ,
                \underbrace{|y|-\lfloor |y|\rfloor}_{r: \text{remainder}}
            ,
                \underbrace{
                    I_{y\ge 0}
                }_{s: \text{sign}}
        \Big)^{\top}
    $$
    where $y = a \times_{\mathbb{R}_q} b \pmod {2^{q+1}}$.
\end{lemma}
\begin{proof}
    Based on the discussion above, we can summarize our intended computation as a sequence of maps
$$
    \begin{pmatrix}
        a\\
        b
    \end{pmatrix}
    \mapsto
    \begin{pmatrix}
        \MULT_{4q+2}(A_{4q+2} (a), A_{4q+2}(b))\\
        \ADD_1(\pi_{2q+2}(a), \pi_{2q+2}(b))
    \end{pmatrix}
    \mapsto
    \begin{pmatrix}
        \ADD_{4q+2}(\cdot, P_{3q+3}(\cdot))\\
        \cdot
    \end{pmatrix}
    \mapsto
    \begin{pmatrix}
        T(\cdot)\\
        \cdot
    \end{pmatrix},
$$
where $\pi_{2q+2}$ is the $1$ by $2q+2$ matrix that projects the last (sign) bit of vectors in $\mathbb{R}_q$. Now recall that for any $B \in \mathbb{N}_+$, we have that
$\textit{depth}(\MULT_B) \le B^2$, $\textit{depth}(\ADD_B) \le B$,
$\textit{width}(\MULT_B) \le B^2 + 1$, and $\textit{width}(\ADD_B) \le 2B$
by Lemmas \ref{lem:add} and \ref{lem:mult}. Lemmas 5.2 and 5.3 of \cite{petersen2024mathematical} imply the existence of a $\ReLU$ MLP 
$$
    \Phi_{\times, q} = \begin{pmatrix}
        T(\ADD_{4q+2}(\MULT_{4q+2}(A_{4q+2} (a), A_{4q+2}(b)), P(\MULT_{4q+2}(A_{4q+2} (a), A_{4q+2}(b)))))\\
        \ADD_1(\pi_{2q+2}(a), \pi_{2q+2}(b))
    \end{pmatrix}
$$
such that $depth(\Phi_{\times, q}) \le {(4q + 2)(4q + 3)}$ and $width(\Phi_{\times, q}) \le {2(4q+2)^2 + 6}$.
Hence, the claim holds.
\end{proof}

\section{{Proof of Theorem~\ref{thrm:WorstCaseUniversalGate}}}
\begin{proof}
Enumerate $\mathbb{R}^d_q=\{x_n\}_{n=1}^{N_1}$
and 
with $N_1\eqdef \bigl(2^{2q+2}-1\bigr)^d \le 4^{d(q+1)}$
.  
For each $n\in [N_1]$, let $\Phi_n\eqdef \Phi_{x_n,q}$ be the corresponding ReLU MLP constructed in Lemma~\ref{lem:neural_spike}
(having width at-most $4d$, depth $4$, and $18d+5$ non-zero parameter satisfying) and satisfying
\begin{equation}
\label{eq:spike_x_n}
    \Phi_{n}(x) = I_{x=x_n}
    \qquad
    (\forall x\in\mathbb{R}_q^d)
.
\end{equation}
Applying the parallelization lemma, see  \cite[Lemma 5.3]{petersen2024mathematical}, we deduce that the exists a ReLU MLP 
$\Phi_{\operatorname{Enc}}:
\mathbb{R}^d\to \mathbb{R}^{N_1}
$ satisfying
\begin{equation}
\label{eq:LosslessEncoder}
        \Phi(x) 
    = 
        \big(
            \Phi_n(x)
        \big)_{n=1}^{N_1}
    =
        \big(
            I_{x=x_n}
        \big)_{n=1}^{2N_1}
\end{equation}
for each $x\in \mathbb{R}^d$.  

Moreover, $\Phi$ has depth $4$, width at-most $2N_1 \le 4^{d(q+1)}2
=
2^{2d(q+1)+1}
$, and at-most $2N_1 \le 
2^{2d(q+1)+1}
$ non-zero parameters.
Consider the $1\times N_1$-matrix $\beta_f$ with entries in $\mathbb{R}_q$ whose $n^{th}$ column $(\beta_f)_n$, for each $n\in [N_1]$, is given by
\[
    (\beta_f)_n\eqdef f(x_n).
\]
Thus, by definition $\beta_f$, has at-most $N_1 \le 4^{d(q+1)}$ non-zero parameters.
Thus, the right-hand side of~\eqref{eq:LosslessEncoder} implies that for each each $x_n\in \mathbb{R}_q^d$
\[
        \beta_f\Phi_{\operatorname{Enc}}(x_n) 
    = 
        \sum_{m=1}^{N_1}\, (\beta_f)_m I_{x_n=x_m}
    = 
        f(x_n)
.
\]
Viewing $\beta_f$ as a (column) vector in $\mathbb{R}^{N_1}$, and abusing the notational identification therefrom, we conclude that $\beta_f^{\top}\Phi_{\operatorname{Enc}}$ satisfies our claim.  
\hfill\\
Additionally, the number of parameters defining $\beta_f^{\top}\Phi_{\operatorname{Enc}}$ has the same depth and width as $\Phi_{\operatorname{Enc}}$ and is has not more than the sum of the non-zero parameters defining $\Phi_{\operatorname{Enc}}$ and $\beta_f$; i.e.\ $\beta_f^{\top}\Phi_{\operatorname{Enc}}$ has no more than $N_1 + 2N_1\le 3 \, 4^{d(q+1)}$ non-zero parameters.
\end{proof}

\section{Proof of Applications and Corollaries}
\label{s:ProofsCorollaries}

\begin{proof}[{Proof of Corollary~\ref{cor:RandCirc}}]
 Let $X_{\cdot}\eqdef (X_k)_{k=B+1}^{B+m}$ be i.i.d.\ Bernoulli random variables with success probability $p>\frac{1}{2}$ and let $\mathcal{A}$ be a $\mathbb{G}_{\operatorname{lgc}}$ circuit, for which the $p$-randomized $\mathbb{G}_{\operatorname{log}}$-circuit $(\mathcal{A},X_{\cdot})$ computes $f$ with probability $p$, as defined in~\eqref{eq:comput_random}.  Then, the proof of~\citep[Theorem 1]{adleman1978two} given in~\citep[page 15]{JunkaBooleanBook} consider the ``probabilistic circuit'' 
 $\mathcal{A}^K$ which computes
    \begin{equation}
    \label{eq:MajRep}
        \Rep(\mathcal{A}^K)(x)\eqdef 
    \operatorname{Maj}_{8BK+1}\big(
        \Rep(\mathcal{A}(x,X_{\cdot}),
        \dots,
        \Rep(\mathcal{A})(x,X_{\cdot})
    \big)
    \end{equation}
 (we allow for one extra copy, but the union bound in~\citep[page 15]{JunkaBooleanBook} was already vacuous at $8nK$ copies of $\mathcal{A}$)
 is deterministic due to~\citep[Lemma 1.5]{JunkaBooleanBook}.  Therefore, Lemma~\ref{lem:Maj} implies that there is a ReLU MLP of 
 $44BK+4$ and width $33BK+3$ implementing $\operatorname{Maj}_{8BK+1}$ on $\{0,1\}^{8BK+1}$. 

 Since $\mathcal{A}$ is a randomized circuit with $K$ computation nodes, each with fan-in $B$ and gate in $\mathbb{G}_{\operatorname{lgc}}$.  Then lemmata~\ref{lem:EQUAL},~\ref{lem:NAND},~\ref{lem:not},~\ref{lem:and},~\ref{lem:or}, \ref{lem:xor}, and \ref{lem:imply} imply that the largest ReLU MLP implementation of these basic logic gates in $\mathbb{G}_{\operatorname{lgc}}$ with fan-in is the one used to represent $\operatorname{EQUAL}$ in Lemma~\ref{lem:EQUAL} which has depth $4$, width at-most $\lceil\frac{B}{2}\rceil$ and at-most $13\lceil\frac{B}{2}\rceil$ non-zero parameters and the widest is the one used to represent $\operatorname{NOT}$ in Lemma~\ref{lem:not} which has width $B$; thus, no ReLU MLP representation of any of these logic gates has depth more than $4$, width more than $B$, and more than $13\lceil\frac{B}{2}\rceil$ non-zero parameters.  Thus, there exists a ReLU MLP with no more than depth $4\Delta$ width $wB$, and 
 $K13\lceil\frac{B}{2}\rceil$ non-zero parameters implementing $\Rep(\mathcal{A})$.  

 By the Parallelization Lemma in~\cite[Lemma 5.3]{petersen2024mathematical}, there is a ReLU MLP implementing 
 \[
 x\mapsto \operatorname{Maj}_{8BK+1}\big(
        \Rep(\mathcal{A}(x,X_{\cdot}),
        \dots,
        \Rep(\mathcal{A})(x,X_{\cdot})
    \big)
 \]
 with width at-most $2(8BK+1)$, depth $4\Delta$, and with no more than $26K\lceil\frac{B}{2}\rceil$ non-zero parameters.  

Consequentially, the ``derandomized representation'' in~\eqref{eq:MajRep} implies that $\Rep(\mathcal{A}^K)$ can be implemented by a ReLU MLP of 
width at-most $
44BK+4+
+
2(8BK+1)
=
60BK+6
$, depth $33BK+3+4\Delta$, and with no more than $
26K\lceil\frac{B}{2}\rceil$ non-zero parameters.  
\end{proof}

\begin{proof}[{Proof of Corollary~\ref{cor:Turing}}]
\hfill\\
\noindent\textbf{Step 1 - Computation at any \textit{single} intermediate time $t\in [T]$:}
Let $\mathcal{M}$ be a turning time with time-bound $T$.
For each $t\in [T]$,~\citep[Theorem 4]{pippenger1979relations} guarantees that there is a $\mathbb{G}_{\operatorname{lgc}}\eqdef \{\operatorname{AND}_2,\operatorname{OR}_2,\operatorname{NOT}\}$-circuit $\mathcal{A}_t$ simulating $\mathcal{M}$; i.e.\ 
\begin{equation}
\label{eq:BoolSim_Turing}
        \Rep(\mathcal{A}_t)(x) 
    = 
        \mathcal{M}(x)_t
\end{equation}
for all $x\in \{0,1\}^B$.  Moreover, $\mathcal{A}$ has depth $\mathcal{O}(t)$ and computation units $\mathcal{O}(t\log(t))$; where $\mathcal{O}$ ignores the dependence on $B$.

Now, similarly to the middle of the proof of Corollary~\ref{cor:RandCirc}, by lemmata~\ref{lem:not},~\ref{lem:and}, and~\ref{lem:or} imply that all the largest ReLU MLP implementation of these basic logic gates in $\mathbb{G}_{\operatorname{lgc}}$ can be represented as ReLU MLPs of depth $4$, width width at-most $B$, with at-most $13\lceil\frac{B}{2}\rceil$ non-zero parameters.  Thus, there exists a ReLU neural network, 
with depth $\mathcal{O}\big(t\big)$ and at-most $\mathcal{O}(B t \log(t))$ computation units, 
implementing $\mathcal{A}$ on $\{0,1\}^B$; i.e.
\begin{equation}
\label{eq:BoolCircruitSim}
        \Rep(\mathcal{A})(x) 
    = 
        \hat{f}_t(x)
\end{equation}
for all $x\in \{0,1\}^B$.  Combining~\eqref{eq:BoolSim_Turing} and~\eqref{eq:BoolCircruitSim} implies that 
\begin{equation}
\label{eq:BoolCircruitSim__singleTimeSlice}
        \hat{f}_t(x)
    =
        \mathcal{M}(x)_t
\end{equation}
for each $x\in \Sigma$.  

\hfill\\
\noindent\textbf{Step 2 - Identification of Time Step:}
By Lemma~\ref{lem:neural_spike}, for each $t\in [T]$ there exists a ReLU MLP $\Phi_{t,q}:\mathbb{R}\to \mathbb{R}$ with width at-most $4$, depth $4$, and $23$ non-zero parameter satisfying
\[
    \Phi_{t,q}(s) = I_{s=t}
\]
for each $s\in \mathbb{R}^1_q$.  

\hfill\\
\noindent\textbf{Step 3 - Aggregating Steps $1$ and $2$:} 
Let $P_x:\mathbb{R}^{B+1}\to \mathbb{R}^B$ and $P_t:\mathbb{R}^{B+1}\to \mathbb{R}$ be the projection matrices, respectively, mapping any $(x,t)\in \mathbb{R}^{B+1}$ to $x$ and any $(x,t)\in \mathbb{R}^{B+1}$ to $t$.
Applying the parallelization lemma in~\cite[Lemma 5.3]{petersen2024mathematical}, there exists a ReLU MLP $\tilde{\Phi}:\mathbb{R}^{B+1}\to \mathbb{R}^{2B}_q$ satisfying
\begin{equation}
\label{eq:parallelized_TM}
        \tilde{\Phi}(x)
    =
        \Big(
            \hat{f}_t(P_x x)
        ,
            \bigoplus_{t=0}^{T}\Phi_{t,q}(P_t x)
        \Big)
\end{equation}
for each $x\in \mathbb{R}^{B+1}_q$; moreover, $\tilde{\Phi}$ has at-most $\mathcal{O}(T\times T\log(T))$ non-zero parameters (still suppressing any dependence on $B$).  
Now, applying Lemma~\ref{lem:mult} we deduce that there exists a ReLU MLP $\psi_{\operatorname{MULT}}:\{0,1\}^B\times \{0,1\}^B\to \{0,1\}^B$ of width $B^2$ and depth $B^2$ implementing bitwise multiplication.  Thus, for each $(x,s)\in \{0,1\}^B\times \mathbb{R}^1_q$
\begin{equation}
\label{eq:TMCompleted}
        \Phi_{\mathcal{M}:T}(x,s)
    \eqdef
            \mathbf{1}_{T+1}^{\top}
            \psi_{\operatorname{MULT}}
        \circ
            \biggl(
                \hat{f}_t(P_x x)
            ,
                \bigoplus_{t=0}^{T}\Phi_{t,q}(P_t x)
            \biggr)
    =
        \sum_{t=0}^T\,
            \hat{f}_t(x) I_{s=t}
    =
        \mathcal{M}(x)_s
\end{equation}
as desired; were we recall that, $\mathbf{1}_{T+1}\in \mathbb{R}^{1+T}$ denotes the vector with all entries equal to $1$.
A brief verification shows that the number of non-zero parameters of $\Phi_{\mathcal{M}:T}$ is still $\mathcal{O}(T^2\log(T))$, since we surprise the dependence on $B$.
\end{proof}

\begin{proof}[{Proof of Lemma~\ref{lem:Basic_Decomposition}}]
Let $x\in [0,1]^d$ and $q\in \mathbb{N}_+$.  We have
\allowdisplaybreaks
\begin{align*}
    \big\|
        f(x) - \Rep(\mathcal{A})\circ \pi_q^d(x)
    \big\|
& \le 
        \|
            \pi^D_q\circ f\circ \pi_q^d(x) 
            - 
            f\circ \pi_q^d(x)
        \|_{\infty}
        +
        \|
            f(x)
            - 
            f\circ \pi_q^d(x)
        \|_{\infty}
\\
&
    +
        \|
            \pi^D_q\circ f\circ \pi_q^d(x) 
            - 
            \Rep(\mathcal{A})\circ \pi_q^d(x)
        \|_{\infty}
\\& \le 
        2^{-p+1}
        +
        \|
            f(x)
            - 
            f\circ \pi_q^d(x)
        \|_{\infty}
\\
&
    +
        \|
            \pi^D_q\circ f\circ \pi_q^d(x) 
            - 
            \Rep(\mathcal{A})\circ \pi_q^d(x)
        \|_{\infty}
.
\end{align*}
This completes our proof.
\end{proof}

\begin{proof}[{Proof of Proposition~\ref{cor:Continuous_Computable}}]
Let $\mathcal{A}_q$ be a ReLU MLP such that with $\Rep(\mathcal{A}_q):\mathbb{R}^d_{q}\to \mathbb{R}^D_{q}$, to be determined retroactively.
By the inequality in Lemma~\ref{lem:Basic_Decomposition}, we have
\begin{equation}
\label{eq:applicaiton_basic_inequality}
    \sup_{x\in [0,1]^d}\,
    \big\|
        f(x) - \Rep(\mathcal{A})\circ \pi_q^d(x)
    \big\|
\le 
    2^{-p+1}
    +
    \omega(2^{-p+1})
    +
    \underbrace{
        \|
            \bar{f}_q(x)
            - 
            \Rep(\mathcal{A}_q)\circ \pi_q^d(x)
        \|_{\infty}
    }_{\eqref{t:ComputeError}}
.
\end{equation}
Since $\bar{f}_q:\mathbb{R}^d_{q}\to \mathbb{R}^{D}_{q}$ then Theorem~\ref{thrm:WorstCaseUniversalGate} implies that there is a ReLU MLP $\mathcal{A}_q$ computing $\bar{f}_q$; whence~\eqref{t:ComputeError}$=0$.
\end{proof}

\section{Pseudocode Yielding The Universal Approximator of Theorem~\ref{thrm:WorstCaseUniversalGate}}
\label{a:UniversalCircuit}

This section contains a brief implementation of the universal approximator in Theorem~\ref{thrm:WorstCaseUniversalGate}.  
For any $a>0$ and $b\in \mathbb{R}^d$, we write
\[
\resizebox{0.99\hsize}{!}{$%
    g(x;a,b) 
\eqdef
    \prod_{j=1}^d \,
        \big(
            \operatorname{ReLU}(a(x_j - b_j) - \frac{3}{2})
            - \operatorname{ReLU}(a(x_j - b_j) - \frac{1}{2})
            - \operatorname{ReLU}(a(x_j - b_j) + \frac{1}{2})
            + \operatorname{ReLU}(a(x_j - b_j) + \frac{3}{2})
        \big)
.
$}
\]
Recall that the multiplication operation may be computed by a ReLU MLP (see Lemma~\ref{lem:mult}).  Having defined this basic helper function, Algorithm~\ref{alg:universal_circuit} effectively implements the universal approximator of Theorem~\ref{thrm:WorstCaseUniversalGate} when the inputs data is $X=\mathbb{R}_q^d$ and $Y=(f(q))_{q\in \mathbb{R}^d_q}$.  In this way, it is an algorithmic version of universal approximation theorems which rely on finitely many aptly computed sample values from the target function; e.g.~\cite{franco2024practicalexistencetheoremreduced,hong2024bridging,neuman2025reconstruction}.
\begin{figure}[H]
  \centering
  \begin{minipage}[t]{0.45\textwidth}
    \vspace{0pt}  
    \centering
    \captionof{algorithm}{\textsc{Universal Circuit}}
    \label{alg:universal_circuit}
    \begin{algorithmic}[1]
      \Require $X\in\mathbb R^{N\times d}$, $Y\in\mathbb R^{N\times D}$
      \Ensure Approximator $\hat f:\mathbb R^d\to\mathbb R^D$
      \State $a\gets\dfrac1{\max\{\min_{n\ne m}\|X_n-X_m\|_\infty/2,2^{q+1}\}}$
      \State $\hat f(x)\gets0$
      \For{$n=1,\dots,N$}
        \State $\hat f(x)\gets \hat f(x)+Y_n\cdot g(x;a,X_n)$
      \EndFor
      \State\Return~$\hat f(x)$
    \end{algorithmic}
  \end{minipage}
  \hfill
  \begin{minipage}[t]{0.4\textwidth}
    \vspace{-2pt}     
    \centering
    \includegraphics[width=\linewidth]{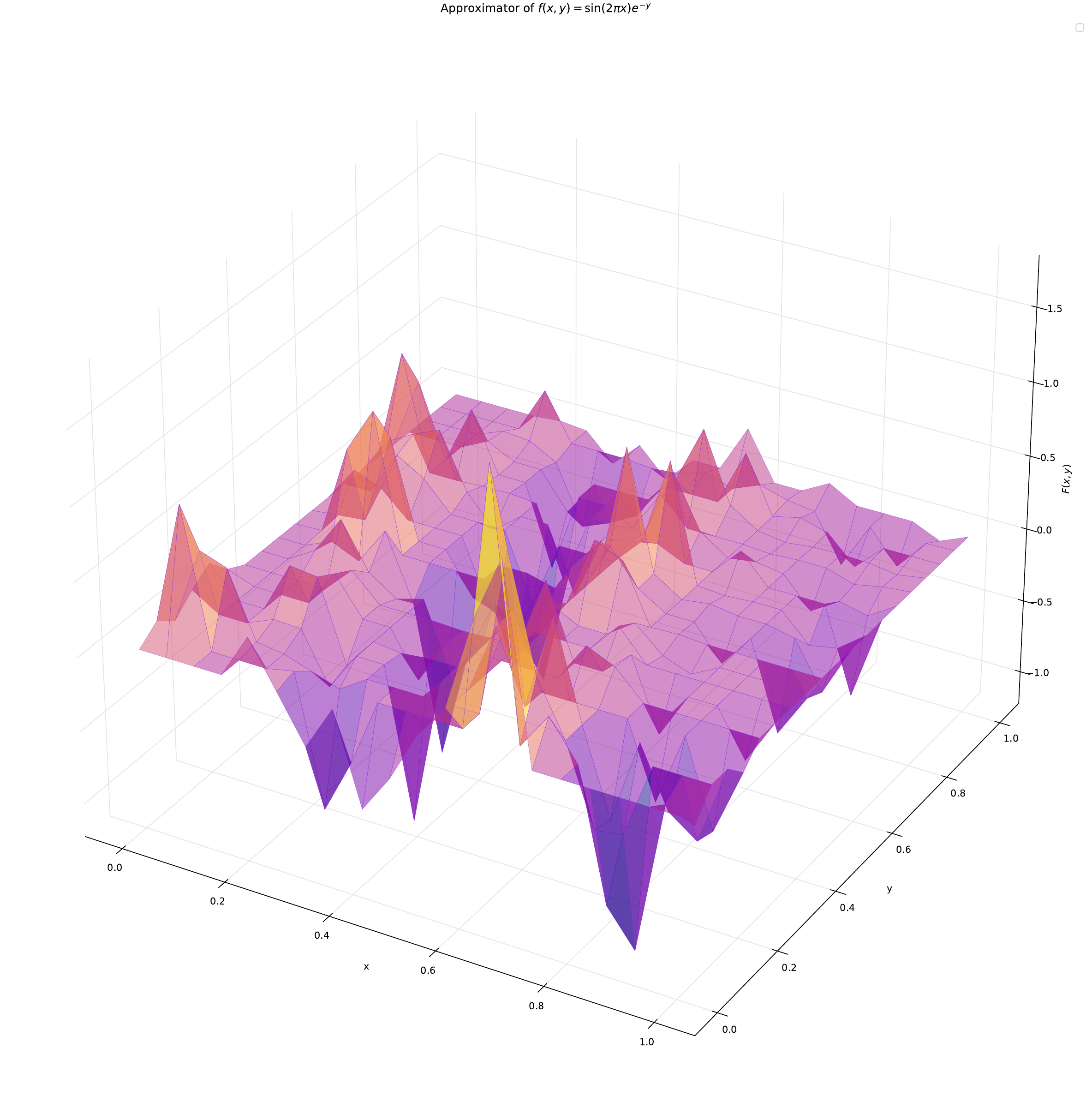}
  \end{minipage}
  \caption{Construction and visualization of the universal approximator. \textbf{Left:} pseudocode for computing the universal approximator in Theorem~\ref{thrm:Main}, and \textbf{Right:} typical graph of that function implemented by this construction.}
  \label{fig:UniversalCircuit_CodeNPic}
\end{figure}

\section{Quoted/Known ReLU MLP Constructions}
\label{a:MLP_Constructions}
This appendix collects useful $\operatorname{ReLU}$ MLP constructions that are available in the literature.

The following result is quoted from~\cite[Lemma 5.3]{petersen2024mathematical}.
\begin{lemma}[{Parallelization}]
Let $m \in \mathbb{N}$ and $(\Phi_i)_{i=1}^m$ be neural networks with architectures $(\sigma_{\text{ReLU}}; d_0^i, \dots, d_{L_i+1}^i)$, respectively. Then the neural network $(\Phi_1, \dots, \Phi_m)$ satisfies
\[
(\Phi_1, \dots, \Phi_m)(x) = (\Phi_1(x_1), \dots, \Phi_m(x_m)) \quad \text{for all } x \in \mathbb{R}^{\sum_{j=1}^m d_0^j}.
\]
Moreover, with $L_{\max} \eqdef \max_{j \leq m} L_j$, it holds that
\begin{align}
    \text{width}((\Phi_1, \dots, \Phi_m)) &\leq 2 \sum_{j=1}^m \text{width}(\Phi_j), \tag{5.1.5a} \\
    \text{depth}((\Phi_1, \dots, \Phi_m)) &= \max_{j \leq m} \text{depth}(\Phi_j), \tag{5.1.5b} \\
    \text{size}((\Phi_1, \dots, \Phi_m)) &\leq 2 \sum_{j=1}^m \text{size}(\Phi_j) + 2 \sum_{j=1}^m (L_{\max} - L_j)d_{L_j+1}^j. \tag{5.1.5c}
\end{align}
\end{lemma}

The following result is quoted from~\citep[Proposition 7.2]{hong2024bridging}.
\begin{proposition}[Optimal Memorization by Two-Hidden-Layer MLPs]
\label{prop:fit_a_network_with_two_hidden_layers}
Let $M,N\in\mathbb{N}_+$. For any set of $MN$ samples $\mathcal{S}\eqdef (x_{i},y_{i})_{i=1}^{MN}\subseteq\mathbb{R}^2$ (where $x_1<x_2<\cdots<x_{MN}$), there exists a ReLU MLP $\Phi$ with $\operatorname{widthvec}=[M,4N-2]$ that can memorize this sample set; i.e. 
\[
    \Phi_{\mathcal{S}}(x_i)=y_i \qquad \mbox{for } i=1,\dots,MN
.
\]
Furthermore, $\Phi$ is linear on the intervals $[x_i,x_{i+1}]$ for $i=1,2,\cdots,MN-1$, and it is constant on each of the segment $(-\infty,x_1]$ and $[x_{MN},\infty)$. The number of nonzero parameters in $\Phi$ is at-most $2MN+2M+8N-4$. 
\end{proposition}

\bibliography{References/0__Refs_Sorted}

\end{document}